\documentclass{article}

\usepackage{microtype}
\usepackage{graphicx}
\usepackage{subfigure}
\usepackage{booktabs} 

\usepackage{hyperref}

\usepackage[accepted]{icml2024}

\usepackage{amsmath}
\usepackage{amssymb}
\usepackage{mathtools}
\usepackage{amsthm}
\usepackage{graphicx}
\usepackage[utf8]{inputenc} 
\usepackage[T1]{fontenc}    
\usepackage{hyperref}       
\usepackage{url}            
\usepackage{booktabs}       
\usepackage{amsfonts}       
\usepackage{nicefrac}       
\usepackage{microtype}      
\usepackage{xcolor}         

\usepackage{xspace}
\usepackage{float}
\usepackage{amsmath}
\usepackage{amssymb}
\usepackage{amsthm}
\usepackage{bm}
\usepackage{accents}
\usepackage{multirow}
\usepackage{dsfont}

\usepackage[capitalize,noabbrev]{cleveref}

\definecolor{EFK}{RGB}{32, 178, 170}
\definecolor{EKF-B}{RGB}{255, 140, 0}
\definecolor{EKF-IW}{RGB}{220, 20, 60}
\definecolor{OGD}{RGB}{138, 43, 226}
\definecolor{WLF-IMQ}{RGB}{30, 144, 255}
\definecolor{WLF-MD}{RGB} {229, 193, 0} 
\definecolor{Hub-EnKF}{RGB}{0, 116, 0}
\definecolor{WS-EnKF}{RGB}{198, 122, 55}
\definecolor{tabblue}{RGB}{31, 119, 180}
\definecolor{taborange}{RGB}{255, 127, 14}

\newcommand{\mWlfMd}{{\color{WLF-MD} \texttt{\textbf{WoLF-TMD}}}\xspace} 
\newcommand{\mWlfImq}{{\color{WLF-IMQ} \texttt{\textbf{WoLF-IMQ}}}\xspace}
\newcommand{\mAgamenoni}{{\color{EKF-IW} \texttt{\textbf{KF-IW}}}\xspace}
\newcommand{\mWang}{{\color{EKF-B} \texttt{\textbf{KF-B}}}\xspace}

\newcommand{\mEnkf}{{\color{EFK} \texttt{\textbf{EnKF}}}\xspace}
\newcommand{\mWEnkfHard}{\texttt{\textbf{AP-EnKF}}\xspace}
\newcommand{\mWEnkfSoft}{{\color{WS-EnKF} \texttt{\textbf{PP-EnKF}}}\xspace}
\newcommand{\mHubEnkf}{{\color{Hub-EnKF} \texttt{\textbf{Hub-EnKF}}}\xspace}

\newcommand{\mkf}{{\color{EFK} \texttt{\textbf{KF}}}\xspace}
\newcommand{\mkfExtended}{{\color{EFK} \texttt{\textbf{EKF}}}\xspace}
\newcommand{\ogd}{{\color{OGD} \texttt{\textbf{OGD}}}\xspace}
\newcommand{\mAgamenoniExtended}{{\color{EKF-IW} \texttt{\textbf{EKF-IW}}}\xspace}
\newcommand{\mWangExtended}{{\color{EKF-B} \texttt{\textbf{EKF-B}}}\xspace}

\newcommand{\eat}[1]{} 

\newcommand{\removewhitespace}[1][-10mm]{\vspace*{#1}}

\theoremstyle{plain}
\newtheorem{theorem}{Theorem}[section]
\newtheorem{proposition}[theorem]{Proposition}
\newtheorem{lemma}[theorem]{Lemma}

\theoremstyle{definition}

\theoremstyle{remark}

\icmltitlerunning{Outlier-robust Kalman Filtering through Generalised Bayes}

\newcommand{\params}{{\boldsymbol{\theta}}}

\newcommand{\normdist}[3]{{\cal N}(#1\,\vert\, #2,\,#3)}
\newcommand{\expectation}[2]{\mathbb{E}_{#1}\left[#2\right]}

\newcommand{\argmax}{\operatornamewithlimits{argmax}}

\newcommand{\nparams}{m}
\newcommand{\nout}{d}

\newcommand{\Diag}{\mathrm{Diag}}

\newcommand{\real}{\sR}

\newcommand{\T}{\intercal}
\newcommand{\trans}{\T}

\newcommand{\myvec}[1]{\mathbf{#1}}
\newcommand{\myvecsym}[1]{\boldsymbol{#1}}

\makeatletter
\newcommand{\algorithmfootnote}[2][\footnotesize]{%
  \let\old@algocf@finish\@algocf@finish
  \def\@algocf@finish{\old@algocf@finish
    \leavevmode\rlap{\begin{minipage}{\linewidth}
    #1#2
    \end{minipage}}%
  }%
}
\makeatother

\def\1{\bm{1}}

\newcommand{\gauss}{\mathcal{N}}

\newcommand{\expect}[1]{\mathbb{E}\left[{#1}\right]} 
\newcommand{\expectQ}[2]{\mathbb{E}_{{#2}}\left[ {#1} \right]} 

\newcommand{\varQ}[2]{\mathbb{V}_{{#2}}\left[ {#1}\right]}

\newcommand{\vepsilon}{\myvecsym{\epsilon}}

\newcommand{\veta}{\myvecsym{\eta}}

\newcommand{\vmu}{\myvecsym{\mu}}

\newcommand{\vlambda}{\myvecsym{\lambda}}
\newcommand{\vLambda}{\myvecsym{\Lambda}}

\newcommand{\vphi}{\myvecsym{\phi}}
\newcommand{\vvarphi}{\myvecsym{\varphi}}
\newcommand{\vPhi}{\myvecsym{\Phi}}

\newcommand{\vPsi}{\myvecsym{\Psi}}

\newcommand{\vtheta}{\myvecsym{\theta}}

\newcommand{\vSigma}{\myvecsym{\Sigma}}

\newcommand{\vb}{\myvec{b}}

\newcommand{\ve}{\myvec{e}}

\newcommand{\vm}{\myvec{m}}

\newcommand{\vu}{\myvec{u}}
\newcommand{\vv}{\myvec{v}}
\newcommand{\vw}{\myvec{w}}

\newcommand{\vx}{\myvec{x}}

\newcommand{\vy}{\myvec{y}}

\newcommand{\vz}{\myvec{z}}

\def\vtheta{{\bm{\theta}}}

\def\vb{{\bm{b}}}

\def\ve{{\bm{e}}}

\def\vm{{\bm{m}}}

\def\vu{{\bm{u}}}
\def\vv{{\bm{v}}}
\def\vw{{\bm{w}}}
\def\vx{{\bm{x}}}
\def\vy{{\bm{y}}}
\def\vz{{\bm{z}}}

\newcommand{\vA}{\myvec{A}}
\newcommand{\vB}{\myvec{B}}

\newcommand{\vF}{\myvec{F}}

\newcommand{\vH}{\myvec{H}}
\newcommand{\vI}{\myvec{I}}

\newcommand{\vK}{\myvec{K}}

\newcommand{\vQ}{\myvec{Q}}
\newcommand{\vR}{\myvec{R}}
\newcommand{\vS}{\myvec{S}}

\def\sR{{\mathbb{R}}}

\DeclareMathOperator{\Tr}{Tr}
\begin{document}

\twocolumn[
\icmltitle{
Outlier-robust Kalman Filtering through Generalised Bayes
}



\icmlsetsymbol{equal}{*}

\begin{icmlauthorlist}
\icmlauthor{Gerardo Duran-Martin}{qmul,omi}
\icmlauthor{Matias Altamirano}{ucl}
\icmlauthor{Alexander Y. Shestopaloff}{qmul,canada}
\icmlauthor{Leandro S\'anchez-Betancourt}{omi,oxmath}
\icmlauthor{Jeremias Knoblauch}{ucl}
\icmlauthor{Matt Jones}{uboulder}
\icmlauthor{François-Xavier Briol}{ucl}
\icmlauthor{Kevin Murphy}{google}
\end{icmlauthorlist}

\icmlaffiliation{qmul}{School of Mathematical Sciences, Queen Mary University, London, UK}
\icmlaffiliation{ucl}{Department of Statistical Science, University College London, London, United Kingdom}
\icmlaffiliation{omi}{Oxford-Man Institute of Quantitative Finance, University of Oxford, UK}
\icmlaffiliation{oxmath}{Mathematical Institute, University of Oxford, UK}
\icmlaffiliation{uboulder}{Institute for Cognitive Science, University of Colorado Boulder, US}
\icmlaffiliation{google}{Google DeepMind}
\icmlaffiliation{canada}{Department of Mathematics and Statistics, Memorial University of Newfoundland, St. John’s, NL, Canada}
\icmlcorrespondingauthor{Gerardo Duran-Martin}{g.duranmartin@qmul.ac.uk}

\icmlkeywords{Machine Learning, ICML, filtering, online learning, robust}

\vskip 0.3in
]



\printAffiliationsAndNotice{}  

\begin{abstract}
We derive a novel, provably robust, and closed-form Bayesian update rule 
for online filtering in state-space models
in the presence of outliers and misspecified measurement models.
Our method combines generalised Bayesian inference
with filtering methods such as the extended and ensemble Kalman filter.
We use the former to show robustness and the latter to ensure 
computational efficiency in the case of nonlinear models.
Our method matches or outperforms
other robust filtering methods (such as those based on variational Bayes) at a much lower computational cost.
We show this empirically on a range of filtering problems
with outlier measurements,
such as object tracking, state estimation in high-dimensional chaotic systems, and online learning of neural networks.
\end{abstract}

\section{Introduction}

Probabilistic state-space models (SSMs) are widely used to address problems in
time-series forecasting, online learning, tracking problems, and signal processing.
A key challenge in SSMs is to perform online (sequential) posterior inference, 
also known as Bayesian filtering. 
If the model is linear and the and the dynamics are known, 
then the Kalman filter (KF) algorithm is the optimal filter in terms of
minimal mean squared error  \citep{morris1976}.
If the model is not linear, 
approximate inference methods must be used.
Although there are a number of techniques available,
in this paper we focus
on the extended Kalman filter (EKF) 
\citep[see e.g.,][]{Sarkka23}
and the ensemble Kalman filter (EnKF) \citep[see e.g.,][]{roth2017enekf},
because they both admit closed-form Bayesian updates to Gaussian posterior approximations
and scale to high dimensions.

\eat{
a common approach to obtain closed-form updates is
to linearise the measurement and state functions around the prior predictive mean;
this is known as the extended Kalman filter (EKF).
This algorithm appears often in the machine learning literature.
For instance, the first use of the EKF for training neural networks dates back to \citet{singhal1989}.
Recently, the EKF was linked to a number of
gradient descent algorithms such as
Adam \cite{aitchison2020bayesian},
natural gradient descent,
and online Newton \cite{RVGA, Ollivier2018}.
The EKF has also been adapted to train modern neural network architectures \cite{titsias2023kalman, lofi}.
Of course, their applicability goes beyond optimisation techniques and applies
to various filtering problems.
For further reference, see the following papers from:
the machine learning literature \citep{boustati2020generalised, Chang2022, Duran-Martin2022, bencomo2023implicit},
the finance literature \citep{cartea2023toxic-flow, cartea2023bandits},
the statistics literature \citep{LRVGA}, and
the engineering literature \citep{pearson2002, huang2020enkf}.
}

\begin{figure}
    \centering
    \includegraphics[width=0.9\columnwidth]{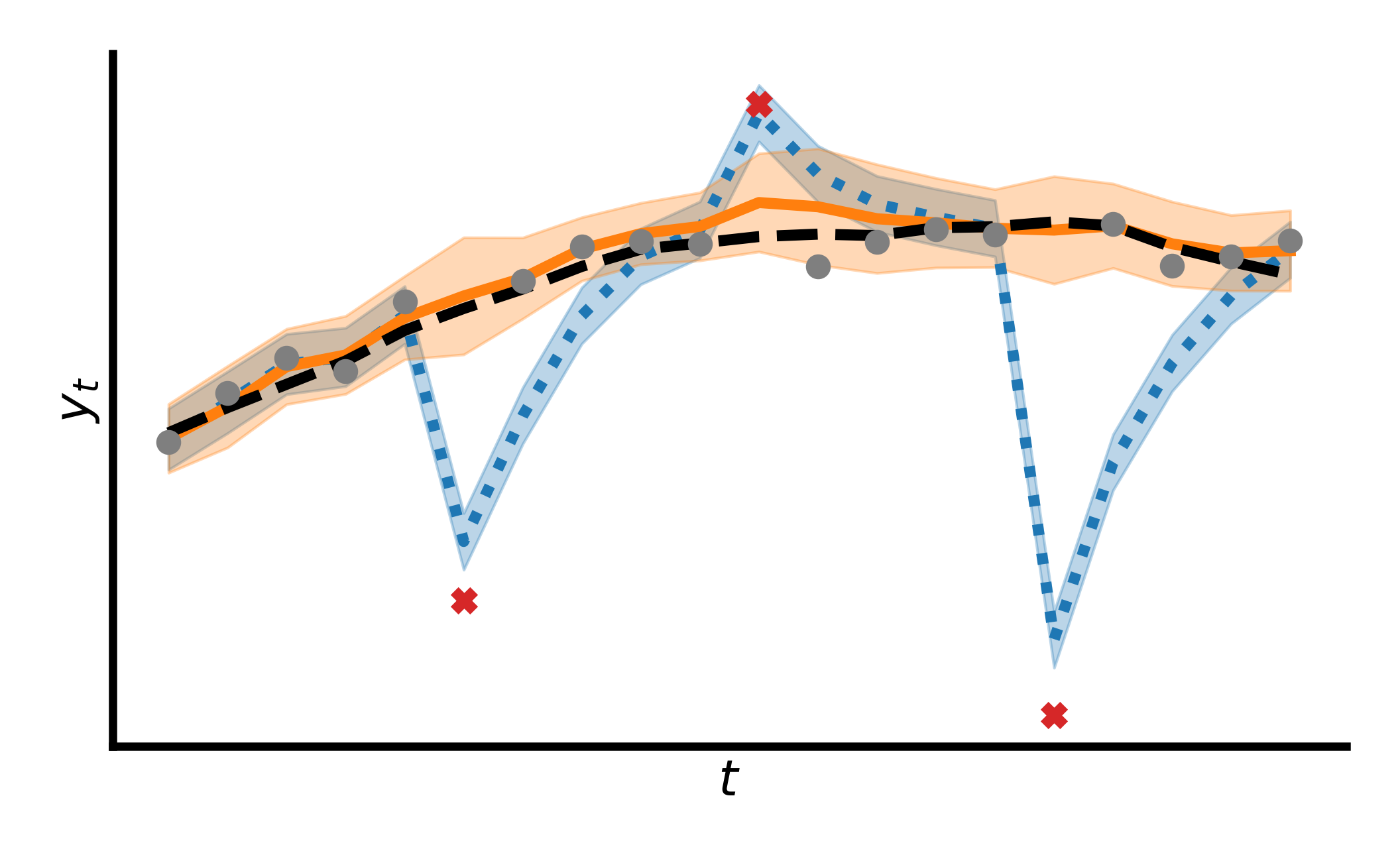}
    \removewhitespace[-6mm]
    \caption{
        First state component of the SSM \eqref{eq:noisy-2d-ssm}.
        The grey dots are measurements sampled from \eqref{eq:noisy-2d-ssm} and
        the red crosses are measurements sampled from an outlier measurement process.
        The \textcolor{tabblue}{dotted blue line} shows the KF posterior mean estimate and
        the \textcolor{taborange}{solid orange line} shows our proposed WoLF posterior mean estimate.
        The regions around the posterior mean cover two standard deviations.
        For comparison, the dashed black line shows the true sampled state process.
    }
    \label{fig:intro-image}
\end{figure}

In a Bayesian setting, the KF and KF-like methods typically use a Gaussian measurement model for computational convenience.
However, a weakness of this assumption is that extreme observations, such as outliers, are not well modelled by Gaussians, leading to model misspecification
\citep{grubbs1969, hampel2011robust}.
Moreover, in practice,
the true dynamics are often different from the one assumed by the filter.
This could be because the measurements are corrupted,
or because the user does not have access to the
measurement function.
For example, in tracking problems, sensor measurement errors record the wrong position of an object;
similarly, in economics and finance, there are instances of data providers sending
time series with erroneous values
\citep[see e.g.][for a survey]{liu2020data}.

There is a large literature that studies the filtering problem in the presence of
outliers and misspecified measurement models\footnote{Alternatively, a measurement model can be referred to as an observation model.}.
Some approaches extend the state space by introducing hierarchical
priors on the measurement model.
Since closed-form estimation of the state is not generally feasible,
a factorised variational Bayes (VB) approach can be employed to 
obtain a fixed-point solution
\citep[see e.g.,][]{ting2007, Agamennoni2012, Huang2016, nurminen2015, wang2018, piche2012}.
Alternatively, Huberised (E)KF-type methods are optimisation-based approaches
to filtering, which minimise the Huber loss;
see \citet{boncelet1983, karlgaard2015nonlinear, das2023robust} and the references therein.
Recently, \citet{boustati2020generalised} proposed a particle filter (PF)
approach that uses generalised Bayes (GB)
to handle model misspecification from outliers;
this builds on work showing that GB updates
are theoretically sound and robust to outliers and model 
misspecification
\citep[see, e.g.,][]{brissi2016, Jewson2018, knoblauch2019generalized,  fong2021martingale, Jewson2021,husain2022adversarial, matsubara2022generalised, matsubara2021robust}.
However, these approaches can be slow, due to the need
to perform iterative optimisation per step (for VB) \citep {knoblauch2018doubly},
or the need to use a large number of samples
(for PF) \citep{boustati2020generalised}.
See Appendix \ref{sec:related-overview} for an overview of a number of these methods.

In this paper, we propose a novel approach that
tackles the filtering problem
in the presence of outliers and measurement
model misspecification.
Our method is based on the GB approach 
where one replaces the log-likelihood of the measurement process
with a loss function.
We call our method the \emph{weighted observation likelihood filter (WoLF)}
because it uses a weighted log-likelihood as loss.
A key advantage of this choice is that we can have  closed-form (conjugate) update equations
that compute an approximate Gaussian posterior.
We derive WoLF variants for the KF, the EKF, and the ensemble Kalman Filter (EnKF).

Our approach has several key advantages
over prior work:
(i) it is fast and has similar computational cost to the KF thanks to closed-form updates,
(ii) it is flexible to the form of misspecification,
(iii) it is provably robust to outliers, and
(iv) it is easy to implement and straightforward to apply to other filtering methods.
See Figure \ref{fig:intro-image} for an illustrative example.

The remainder of the paper proceeds as follows:
in Section \ref{sec:filtering-intro} we introduce the filtering problem
and some common algorithms to approximate a Gaussian posterior.
In Section \ref{sec:method} we
present our method and
derive WoLF variants to the KF, the EKF, and the EnKF.
Next, in Section \ref{sec:theory} we
show that for certain choices of weighting functions,
our method is provably robust.
Finally, in Section \ref{sec:experiments}, we show empirically that our method
matches and in some cases outperforms previous robust filtering methods at a lower
running time.
We consider a variety of filtering problems, including
2d-tracking,
state estimation in high-dimensional chaotic systems,
and  online learning of neural networks.
Our code can be found at
\url{https://github.com/gerdm/weighted-likelihood-filter}.

\section{Background: Filtering in SSMs}
\label{sec:filtering-intro}

We briefly review Kalman filtering and the main extensions we consider in this work.
In what follows, $\nparams, \nout\in\mathbb{N}$ are the dimensions
of the state and measurement processes, and $p(\cdot)$
is used for densities.
Given an initial state $\vtheta_0 \in \real^\nparams$, a Markovian state-space model (SSM) is defined as
\begin{align}
    \vtheta_t &= f_t(\vtheta_{t-1}) + \vphi_t,\label{eq:ssm-latent}\\
    \vy_t &= h_t(\vtheta_t) + \vvarphi_t, \label{eq:ssm-measurement}
\end{align}
for $t \in \{1,\ldots,T\}$.
Here, $\vtheta_t \in \real^{\nparams}$
is the (latent) state vector,
$\vy_t \in \real^\nout$
is the (observed) measurement vector,
$f_t: \real^{\nparams}\to\real^\nparams$ is the dynamics function,
$h_t: \real^\nparams \to \real^\nout$ is the measurement function,\footnote{
In many applications, the function $h_t$ is modulated by an exogenous feature vector $\vx_t$. 
See e.g., Section \ref{subsec:uci-corrupted}.
}
$\vphi_t$ is a zero-mean Gaussian-distributed random vector with known covariance matrix $\vQ_t$,
and
$\vvarphi_t$ is any zero-mean random vector representing the measurement noise.
The meaning of the state depends on the application;
for example, it can be
the position of an object,
the state of the atmosphere,
or the weights of a neural network,
as we will see in the results section.

The filtering of an SSM consists of two steps: the predict step and the update step.
Given the result of a previous update step (i.e., the iterative posterior)
$p(\vtheta_{t-1} \vert \vy_{1:t-1})$, the predict step estimates the prior predictive distribution
\begin{equation}\label{eq:filtering-predict}
    p(\vtheta_t \vert \vy_{1:t-1}) = \int p(\vtheta_t \vert \vtheta_{t-1}) p(\vtheta_{t-1} \vert \vy_{1:t-1}) {\rm d}\vtheta_{t-1},
\end{equation}
and the update step estimates the new posterior distribution
\begin{equation}\label{eq:filtering-update}
    p(\vtheta_t \vert \vy_{1:t}) \propto p(\vy_t \vert \vtheta_t) p(\vtheta_t \vert \vy_{1:t-1}).
\end{equation}
where $\propto$ indicates equality up to a multiplicative normalisation constant.

\subsection{The Kalman filter}
Suppose the SSM is linear and Gaussian, that is,
\begin{equation}\label{eq:linear-ssm}
\begin{aligned}
    \vtheta_t &= \vF_t\, \vtheta_{t-1} + \vphi_t,\\
    \vy_t &=  \vH_t\, \vtheta_t + \vvarphi_t,
\end{aligned}
\end{equation}
with $\vvarphi_t$ a zero-mean Gaussian with known covariance matrix $\vR_t$,
$\vF_t \in \real^{\nparams\times\nparams}$, and $\vH_t \in\real^{\nout\times \nparams}$.
Then, given the initial condition $p(\vtheta_0) = \normdist{\vtheta_0}{\vmu_0}{\vSigma_0}$ ---
a Gaussian density with known mean $\vmu_0 \in \mathbb{R}^m$ and covariance $\vSigma_0 \in \mathbb{R}^{m \times m}$ ---
the exact Bayesian predict and update steps are given by
\begin{align}
    p(\vtheta_t \vert \vy_{1:t-1}) &= \normdist{\vtheta_t}{\vmu_{t|t-1}}{\vSigma_{t|t-1}},\\
    p(\vtheta_t \vert \vy_{1:t}) &= \normdist{\vtheta_t}{\vmu_{t}}{\vSigma_{t}},
\end{align}
with prior predictive mean and covariance given by
\begin{equation}\label{eq:kf-predict}
\begin{aligned}
    \vmu_{t\vert t-1} &= \vF_t\,\vmu_{t-1},\\
    \vSigma_{t | t-1} &= \vF_t\,\vSigma_{t-1}\,\vF_{t}^\intercal + \vQ_{t},\\
\end{aligned}
\end{equation}
and posterior mean and covariance given by
\begin{equation}\label{eq:kf-update}
\begin{aligned}
    \vSigma_t^{-1} &= \vSigma_{t \vert t-1}^{-1} + \vH_t^\intercal\,\vR_{t}^{-1}\,\vH_t,\\
    \vK_t &= \vSigma_t\,\vH_t^\intercal\,\vR_t^{-1},\\
    \vmu_t &= \vmu_{t | t -1} +  \vK_t\,(\vy_t - \hat\vy_t),\\
\end{aligned}
\end{equation}
where $\hat{\vy}_t =  \vH_t\,\vmu_{t|t-1}$ is the predicted observation and
$\vK_t$ is the Kalman gain matrix used to map
the error (residual) vector in observation space to an update in latent state space.

\subsection{The Extended Kalman filter}
\label{subsec:ekf}
When the state and measurement functions are non-linear,
a common approach is to introduce a Gaussian posterior density
$q(\vtheta_t \vert \vy_{1:t}) = \normdist{\vtheta_t}{\vmu_t}{\vSigma_t}$
that approximates
the posterior density $p(\vtheta_t \vert \vy_{1:t})$ through a Kullback-Leibler projection.
This is done using Gaussian approximations of the joint densities
$q(\vtheta_t, \vtheta_{t-1} \vert \vy_{1:t-1})$ and
$q(\vtheta_t, \vy_t \vert \vy_{1:t-1})$ 
and then performing a closed-form Bayesian update;
see Section 8.4 in \citet{Sarkka23} for details.

The EKF is a special case of the above in which one
linearises $f_t$ in \eqref{eq:ssm-latent} and $h_t$ in \eqref{eq:ssm-measurement},
and assumes a Gaussian measurement noise $\vvarphi_t$.
As a consequence, the predict and update equations resemble those of the standard KF.
More precisely, the EKF
replaces the mean of the state transition density $p(\vtheta_t \vert \vtheta_{t-1})$
in \eqref{eq:ssm-latent} with
\begin{equation}\label{eq:linearised-state-mean}
    \mathbb{E}[\vtheta_t \vert \vtheta_{t-1}] \approx
    \vF_t(\vtheta_{t-1} - \vmu_{t-1}) + f_t(\vmu_{t - 1}) =: 
    \bar\vmu_{t | t-1},
\end{equation}
and the measurement mean of  $p(\vy_t \vert \vtheta_t)$
in \eqref{eq:ssm-measurement} with
\begin{equation}\label{eq:linearised-measurement-mean}
    \mathbb{E}[\vy_t \vert \vtheta_t]
    \approx \vH_t( \vtheta_t - \vmu_{t | t-1}) + h_t(\vmu_{t | t-1}) =: \bar{\vy}_t,
\end{equation} 
where
$\vmu_{t|t-1} = \mathbb{E}[\bar{\vmu}_{t|t-1} | \vmu_{t-1}] = f_t(\vmu_{t-1})$,
$\vF_t$ is the Jacobian of $f_t$ evaluated at $\vmu_{t-1}$,
$\vH_t$ is the Jacobian of $h_t$ evaluated at $\vmu_{t|t-1}$, and
$\hat{\vy}_t = \mathbb{E}[\bar{\vy}_t] =  h_t(\vmu_{t|t-1})$.
The linearisation of the transition and measurement around the respective previous means allows the resulting
predict and update equations to closely resemble \eqref{eq:kf-predict} and \eqref{eq:kf-update}.
Therefore, the algorithm remains relatively scalable
and more efficient low-rank extensions can also be derived \citep[see e.g.][]{lofi, LRVGA, cartea2023toxic-flow}.

\subsection{The ensemble Kalman filter}
\label{subsec:enkf}
The ensemble Kalman filter (EnKF)
was developed as an alternative
to the extended Kalman filter for high-dimensional and highly non-linear state dynamics
\citep[see e.g.,][]{evensen1994enkf,burgers1998analysis, roth2017enekf}.
This method avoids
the need to compute Jacobians and the storage of an explicit $\nparams \times \nparams$ posterior covariance matrix,
by instead representing the belief state 
with  an ensemble of $N \in \mathbb{N}$ particles $\hat{\vtheta}_{t}^{\left(i\right)}\in\real^{\nparams}$ for $i = 1, \ldots, N$
evolved through time.
This enables the method to scale to high-dimensional
state spaces with complex nonlinear dynamics, such as those frequently arising in  data assimilation for
weather forecasting \citep{Evensen2009}.
In the EnKF,
the predict step  samples
\eqref{eq:ssm-latent} to obtain $\hat{\vtheta}_{t | t - 1}^{\left(i\right)}$.
Then, the update step samples predictions $\hat{\vy}_{t\vert t-1}^{\left(i\right)} \in \real^\nout$, for each particle, 
according to 
\begin{equation}\label{eq:enkf-sample-predict}
   \hat{\vy}_{t\vert t-1}^{\left(i\right)} \sim \gauss\left(h_t\left(\hat{\vtheta}_{t|t-1}^{(i)}\right),\vR_t\right).
\end{equation}
The particles are then updated according to
\begin{equation}\label{eq:enkf-update}
\hat{\vtheta}_{t}^{\left(i\right)}=
\hat{\vtheta}_{t\vert t-1}^{\left(i\right)}+\bar{\vK}_{t}\big(\vy_{t}-\hat{\vy}_{t\vert t-1}^{(i)}\big),
\end{equation}
with gain matrix $\bar{\vK}_{t}$ calculated from the ensemble.
See Appendix \ref{sec:weighted-ensemble-kalman-filter} for details.

\section{The weighted observation likelihood filter}
\label{sec:method}

Our method is based on the GB 
approach
where one modifies the update step in  \eqref{eq:filtering-update} to use a loss function
$\ell_t: \real^\nparams \to \real$ in place of the negative log-likelihood of the measurement process.
This gives the generalised posterior
\begin{equation}
    q(\vtheta_t \vert \vy_{1:t}) \propto \exp(-\ell_t(\vtheta_t)) q(\vtheta_t \vert \vy_{1:t-1}).
\end{equation}
We propose to gain robustness to outliers 
in observation space by taking the loss function to be the model's  negative log-likelihood scaled by a data-dependent
weighting term
\begin{equation}
\ell_t(\vtheta_t)
= -W^2(\vy_t, \hat{\vy}_t)\,\log q(\vy_t| \vtheta_t),
\label{eq:weighted-loglikelihood}
\end{equation}
with 
$W: \real^{\nout}\times\real^\nout\to\real$ the weighting function and 
$q(\vy_t \vert \vtheta_t)$ the modelled measurement process.
We call our method the \emph{weighted observation likelihood filter (WoLF)}.
To specify an instance of our method, ones needs to define
the  likelihood $q(\vy_t \vert \vtheta_t)$
and the weighting function $W$.
In the next subsections, we show the flexibility of  WoLF
and derive weighted-likelihood-based KF, EKF, and EnKF algorithms.
Setting $W(\vy_t, \bar{\vy}_t)=1$ trivially recovers existing methods, but we will instead use non-constant weighting functions inspired by the work of \cite{Barp2019,matsubara2021robust,Altamirano2023_bocpd,Altamirano2023_gp}.

\subsection{Linear weighted observation likelihood filter}

The following proposition gives a closed-form solution
for the update step of WoLF under a linear measurement function and a Gaussian likelihood (see \Cref{proof:weighted-kf} for the proof).
\begin{proposition}\label{prop:weighted-kf}
    Consider the linear-Gaussian SSM \eqref{eq:linear-ssm}
    with weighting function $W:\real^\nout\times\real^\nout \to \real$.
    Then, the update step of WoLF with loss function \eqref{eq:weighted-loglikelihood} is given by \eqref{eq:kf-update}
    with $\vR_t^{-1}$ replaced by $\bar{\vR}_t^{-1} = W^2(\vy_t, \hat{\vy}_t)\,\vR_t^{-1}$.
\end{proposition}
The resulting predict and update steps for WoLF under linear dynamics and
zero-mean Gaussians for the state and measurement process
are shown in Algorithm \ref{algo:wlf-step}.
\begin{algorithm}[ht]
\begin{algorithmic}
    \REQUIRE $\vF_t$, $\vQ_t$ // predict step
    \STATE $\vmu_{t|t-1} \gets \vF_t\,\vmu_{t-1}$
    \STATE $\vSigma_{t|t-1} \gets \vF_t\,\vSigma_{t-1}\, \vF_t^\trans + \vQ_t$
    \REQUIRE $\vy_t$, $\vH_t$, $\vR_t$  // update step
    \STATE $\hat{\vy}_t \gets \vH_t \,\vmu_{t|t-1}$
    \STATE $w_t \gets W(\vy_t, \hat{\vy}_t)$
    \STATE $\vSigma_{t}^{-1}  \gets \vSigma_{t\vert t-1}^{-1} +
    w_t^2\,\vH_{t}^{\trans}\, \vR_{t}^{-1}\, \vH_{t}$
    \STATE $\vK_t \gets w_t^2\,\vSigma_t\,\vH_t^\trans \,\vR_t^{-1}$
    \STATE $\vmu_{t}  \gets \vmu_{t\vert t-1}+ \vK_t \left(\vy_{t}-\hat{\vy}_{t}\right)$
\end{algorithmic}
\caption{
    WoLF predict and update step
}
\label{algo:wlf-step}
\end{algorithm}

The computational complexity of WoLF under linear 
dynamics matches that of the KF, i.e., 
$O(\nparams^3)$.
Alternative robust filtering algorithms require
multiple iterations per measurement to achieve robustness and stability, making them significantly slower;
see Table \ref{tab:complexity-linear-model} for the computational complexity for the methods we consider, and
Table \ref{tab:2d-ssm-running-time} and Figure \ref{fig:uci-per-step-time} for empirical comparisons.

\subsection{Nonlinear weighted observation likelihood filter}
\label{subsec:wlf-nonlinear-extensions}

Our method readily extends to other nonlinear filtering algorithms.
For example,
a WoLF version of the EKF is obtained by introducing a weighting function to $\eqref{eq:linearised-measurement-mean}$
yielding the approximate log-likelihood
\begin{equation}
    \log q(\vy_t \vert \vtheta_t) = W^2(\vy_t, \hat{\vy}_t)\log\normdist{\vy_t}{\bar{\vy}_t}{\vR_t}.
\end{equation}
Similarly, we can derive a weighted ensemble Kalman filter
by weighting error terms in the update step \eqref{eq:enkf-update};
see Appendix \ref{sec:weighted-ensemble-kalman-filter} for details.
Finally, our method can be easily extended to handle measurement processes
modelled as an exponential family, such as in classification;
see Appendix \ref{sec:wlf-expfam-extension} for details.

\subsection{The choice of weighting function}
\label{subsec:choice-weighting-function}

Weighted likelihoods have a well-established history in Bayesian inference
and have demonstrated their efficacy in improving robustness
\citep{grunwald2012safe,holmes2017assigning,grunwald2017inconsistency,miller2018robust,bhattacharya2019bayesian,alquier2020concentration,dewaskar2023robustifying}. 
In this context, the corresponding posteriors are often referred to as fractional, tempered, or power posteriors. 
In most existing work, the determination of weights relies on heuristics 
and the assigned weights remain constant across all data points
so that $W(\vy_t, \hat{\vy}_t) = w \in \real$ for all $t$. 
In contrast, we dynamically incorporate information from the most recent observations without incurring additional computational costs
by defining the weight as a function of the current observation
$\vy_t$
and its prediction
$\hat{\vy}_t =  h_t(\vmu_{t|t-1})$,
which is based on all of the past observations.

To define the weighting function, 
we take inspiration from previous work 
for dealing with outliers.
In particular,  \citet{wang2018}
proposed classifying robust filtering algorithms
into two main types:
\textit{compensation-based}
algorithms,
which incorporate information from tail events into the model
in a robust way
\citep[see, e.g.,][]{Huang2016, Agamennoni2012},
and
\textit{detect-and-reject} algorithms,
which 
assume that outlier observations bear no useful information 
and thus are ignored
\citep[see, e.g.,][]{wang2018, mu2015}.
Below we show how both of these strategies can be implemented
using our WoLF method by merely changing the weighting function.

\paragraph{Inverse multi-quadratic weighting function:}
As an example of a compensation-based method,
we follow \citet{Altamirano2023_gp}
and use the Inverse Multi-Quadratic (IMQ) weighting,
which in our SSM setting is
\begin{align} \label{eq:w_t}
    W(\vy_t, \hat{\vy}_t) & =\left(1+\frac{||\vy_{t}-\hat{\vy}_{t}||_2^2}{c^{2}}\right)^{-1/2},
\end{align}
where $c > 0$ is the soft threshold and $\|\cdot\|$ denotes the $l_2$ norm.
We call WoLF with IMQ weighting ``WoLF-IMQ''.

\paragraph{Mahalanobis-based weighting function:}
The $l_2$ norm in the IMQ can be modified to account for the covariance structure of the measurement
process by replacing it with the Mahalanobis distance between $\vy_t$ and $\hat{\vy}_t$:
\begin{align}\label{eq:mahalanobis-weight}
    W(\vy_t, \hat{\vy}_t) & =\left(1+\frac{\|\vR_t^{-1/2}(\vy_t - \hat{\vy}_t)\|_2^2}{c^{2}}\right)^{-1/2}.
\end{align}
We call  WoLF with this weighting function the WoLF-MD method.
This type of weighted IMQ function has been used extensively in the kernel literature
\citep[see e.g.][]{chen2019stein, detommaso2018stein, riabiz2022optimal}.

\paragraph{Threshold Mahalanobis-based weighting function:}
As an example of a detect-and-reject method,
 we consider
\begin{equation}\label{eq:thresholded-mahalanobis-weight}
    W(\vy_t, \hat{\vy}_t) =
    \begin{cases}
    1 & \text{if } \|\vR_t^{-1/2}(\vy_t - \hat{\vy}_t)\|_2^2 \leq c\\
    0 & \text{otherwise}
    \end{cases}
\end{equation}
with $c > 0$ the fixed threshold.
The weighting function \eqref{eq:thresholded-mahalanobis-weight} corresponds
to ignoring information from estimated measurements whose Mahalanobis distance
to the true measurement is larger than
some predefined threshold $c$.
In the  linear setting, this weighting function is related to the benchmark method employed in \citet{ting2007}.
We refer to WoLF with this weighting function as ``WoLF-TMD''.

For $\vy_t \in \real^\nout$ with $\nout = \nparams \gg 1$ and diagonal measurement covariance
$\vR_t = \text{diag}\left(r_{t,1}, \ldots, r_{t, \nout}\right)$,
the WoLF-TMD function can be modified to weight individual observations so that
$W: \real^{\nout}\times\real^\nout\to\real^\nout$.
See Section \ref{experiment:lorenz96} for an example and
Appendix \ref{sec:dimension-specific-weighting}
for a discussion. 

The proposed weighting functions --- the IMQ, the MD, and the TMD ---
are defined such that $W: \real^\nout\times\real^\nout \to [0, 1]$
and therefore can only down-weight observations.
This means that our updates are always conservative,
i.e., our posteriors will be wider in the presence of outliers
(see Figure \ref{fig:intro-image} for an example).
\subsection{Theoretical properties}
\label{sec:theory}

In this section, we prove
the outlier-robustness for WoLF-type methods.
We use the classical framework of \citet{huber2011robust}.
Consider measurements $\vy_{1:t}$.
We measure the influence of a contamination $\vy^{c}_t$ by examining the divergence
between the posterior with the original observation $\vy_t$
and the posterior with the contamination $\vy_t^{c}$, which is allowed to be arbitrarily large. 
As a function of $\vy_t^{c}$, this divergence is called the \emph{posterior influence function} (PIF) and was studied in \citet{Ghosh2016, matsubara2021robust, Altamirano2023_bocpd, Altamirano2023_gp}. Following \citet{Altamirano2023_gp}, we consider the Kullback-Leibler (KL) divergence,
which allows us to obtain closed-form expressions for Gaussians. 
The PIF is given by
\begin{align}
    \operatorname{PIF}(\vy_t^{c},\vy_{1:t}) = \operatorname{KL} \left(
        p(\vtheta_t|\vy^{c}_t,\vy_{1:t-1})
        \| p(\vtheta_t|\vy_t,\vy_{1:t-1})\right). \nonumber
\end{align}

If $\sup_{\vy_t^c\in\mathbb{R}^{d}}|\operatorname{PIF}(\vy_t^{c},\vy_{1:t}) |<\infty$,
then the posterior is called outlier-robust, which
indicates that as $\|\vy_t-\vy_{t}^{c}\|_2 \to \infty$,
the contamination's effect on the posterior is bounded.
This is the Bayesian equivalent to bias-robustness in frequentist statistics.

\begin{theorem}\label{prop:weighted-kf-bounded-pif}
    Consider the linear Gaussian SSM \eqref{eq:linear-ssm} or its linearised (EKF) approximation.
    The standard (E)KF posterior has an unbounded PIF and is not outlier robust.
    
    In contrast, the generalised posterior presented in \cref{prop:weighted-kf} 
    has bounded PIF and is, therefore, outlier robust for any weighting function $W$ such that
    $\sup_{\vy_t\in\real^d}W(\vy_t, \hat{\vy}_t)<\infty$ and $\sup_{\vy_t\in\real^d}W(\vy_t, \hat{\vy}_t)^2\,\|\vy_t\|_{2}<\infty$.
\end{theorem}

The proof is in \Cref{app:proof_robustness}. In particular, the conditions are satisfied when $W$ is
    \eqref{eq:w_t}, \eqref{eq:mahalanobis-weight}, or \eqref{eq:thresholded-mahalanobis-weight},
    which are the focus of our paper.
\cref{fig:pif-2d-tracking} shows an empirical validation of this proposition for the 2D tracking problem detailed in \cref{experiment:2d-tracking}.
Here, it is clear that the PIF for the standard KF is unbounded, whereas our methods exhibit a bounded influence. In Appendix \ref{proof:ensemble-kf} we show a version of the theorem above for the EnKF.

\begin{figure}[htb]
    \centering
    \includegraphics[width=\columnwidth]{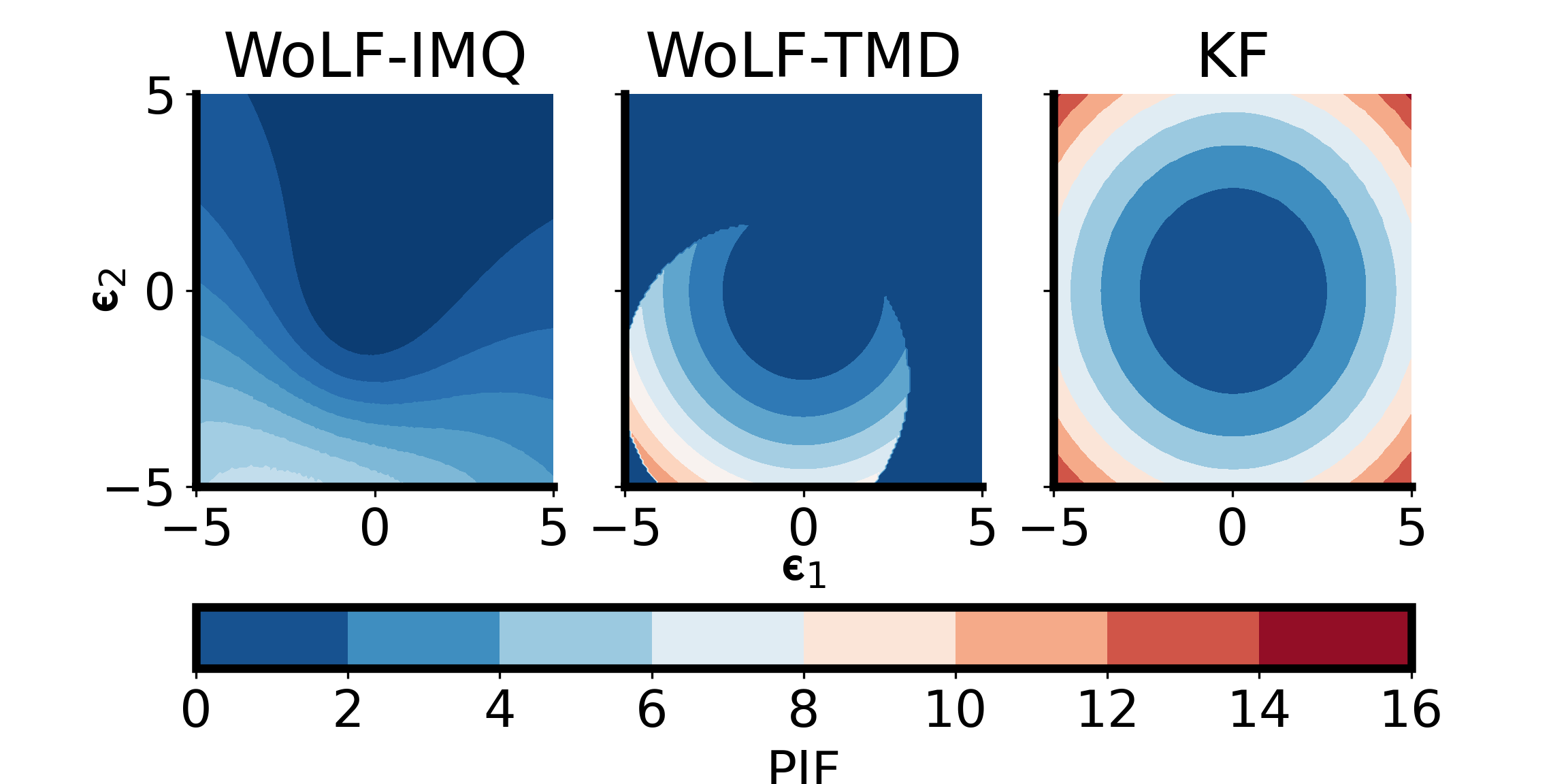}
    \removewhitespace[-6mm]
    \caption{
        PIF for the 2d tracking problem of Section \ref{experiment:2d-tracking}.
        The last measurement $\vy_t$ is replaced with $\vy^{c}_t = \vy_{t} + \vepsilon$,
        where $\vepsilon \in [-5, 5]\times[-5, 5]$.
        We observe that the PIF is asymetric for the weighted methods; this 
        is because the weighting term is a function of the prior predictive and the measurement at time $t$.
        See Appendix \ref{subsec:pif-plot-explain} for a more detailed explanation.
 }
    \label{fig:pif-2d-tracking}
\end{figure}

\section{Experiments}
\label{sec:experiments}

In this section, we study the performance of the WoLF methods in multiple filtering settings.
Each experiment employs
a dataset (or samples data from an SSM),
a collection of benchmark methods, and
a metric to compare the methods.

For our robust baselines,
we make use of three methods that are representative of recent state-of-the-art 
approaches to robust filtering:
the Bernoulli KF of \citet{wang2018} (\mWang),
which is an example of a detect-and-reject strategy;
the inverse-Wishart filter of \citet{Agamennoni2012} (\mAgamenoni),
which  is an example of a compensation-based strategy;
and
the Huberised EnKF of \citet{roh2013} (\mHubEnkf),
which  is an example of a Huberised algorithm.
The \mWang and \mAgamenoni are deterministic and optimise a VB objective to compute a Gaussian approximation to the state posterior
(see Appendix \ref{sec:related-overview} for details).
We do not compare against sophisticated hierarchical methods nor methods based on particle filtering, because
these do not scale well to high-dimensional state spaces.
For the neural network fitting problem,
we also consider  a variant of  online gradient descent (\ogd) based on Adam \cite{kingma2017adam}, which
uses multiple inner iterations per step (measurement).
This method does scale to high-dimensional state spaces, but sadly only gives a maximum a posteriori (MAP) estimate and
is not as sample efficient as a robust Bayesian filter.

For experiments where KF or EKF is used as the baseline,
 we consider the following WoLF variants:
(i) the WoLF version with inverse multi-quadratic weighting function (\mWlfImq),
(ii)  the thresholded WoLF with Mahalanobis-based weighting function (\mWlfMd).
When using the EKF variants, we linearise
the state mean \eqref{eq:linearised-state-mean} and measurement mean \eqref{eq:linearised-measurement-mean}.
For experiments where the ensemble KF (\mEnkf) is taken as the baseline algorithm,
we benchmark the performance of the weighted likelihood EnKF with
(i)  weighting with averaged-particles (\mWEnkfHard) and
(ii) the per-particle weightings (\mWEnkfSoft).
See Appendix \ref{sec:weighted-ensemble-kalman-filter} for a detailed description of the robust EnKF methods.

\begin{table}[ht]
    \small
    \centering
    \begin{tabular}{llll}
    \toprule
               Method & Cost & \#HP & Ref \\
    \midrule
         \mkf &  $O(\nparams^3)$ & 0 & \citeauthor{kalman1960}\\
         \mWang & $O(I\,\nparams^3)$ & 3 & \citeauthor{wang2018}\\
         \mAgamenoni & $O(I\,\nparams^3)$ & 2 & \citeauthor{Agamennoni2012} \\
         \ogd & $O(I\, \nparams^2)$ & 2 & \citeauthor{bencomo2023implicit}\\
         \mWlfImq  &  $O(\nparams^3)$ & 1 &(Ours)\\
         \mWlfMd  &  $O(\nparams^3)$ & 1 & (Ours)\\
    \bottomrule
    \end{tabular}
    \caption{
        Computational complexity  of the update step,
        assuming  $d \leq \nparams$ and assuming linear dynamics.
        Here, $I$ is the number of inner iterations,
        \#HP refers to the number of hyperparameters we tune, and
        ''Cost'' refers to the computational complexity.
    }
    \label{tab:complexity-linear-model}
\end{table}

In each experiment, and unless otherwise specified, we run $100$ trials to evaluate each method.
The hyperparameters of each method 
are chosen on the first trial using the Bayesian optimisation (BO) package of \citet{nogueira2014bayesopt}.
BO is a popular derivative-free approach to function maximisation \citep[see e.g.][]{frazier2018tutorial}.
Specifically, we optimise the hyperparameters that minimise the chosen metric on the first run of each experiment.
Where a multi-output metric is specified, the minimisation is taken over the maximum of the output.
The hyperparameters for KF/EKF-like methods are:
the noise scaling and number of inner iterations for the \mAgamenoni,
the two shape parameters and the number of inner iterations for the \mWang,
the learning rate and number of inner iterations for the \ogd,
the thresholding value for the \mWlfMd, and
the soft threshold hyperparameter for the \mWlfImq.
See  Table \ref{tab:complexity-linear-model} for a summary.

The results we obtain can be summarised as follows:
WoLF-based methods either
outperform or match the performance of their counterparts in the metrics we specify below, but typically at a much lower running time.

\subsection{Robust KF for tracking a 2D object}
\label{experiment:2d-tracking}

We consider the classical problem of estimating the position of an object moving in 2D
with constant velocity,
which is commonly used to benchmark tracking problems
(see e.g., Example 8.2.1.1 in \citet{pml2Book} or Example 4.5 in \citet{Sarkka23}).
The SSM takes the form
\begin{equation} \label{eq:noisy-2d-ssm}
\begin{aligned}
    p(\params_t \vert \params_{t-1}) &= \normdist{\params_t}{\vF_t\params_{t-1}}{\vQ_t},\\
    p(\vy_t \vert \params_t) &= \normdist{\vy_t}{\vH_t\params_t}{\vR_t},
\end{aligned}
\end{equation}
where $\vQ_t = q\,{\bf I}_4$, $\vR_t = r\,{\bf I}_2$,
$(\vtheta_{0,t}, \vtheta_{1,t})$ is the position, 
$(\vtheta_{2, t}, \vtheta_{3,t})$ is the velocity,
{\small
\begin{align*}
    \vF_t &=
    \begin{pmatrix}
    1 & 0 & \Delta & 0\\
    0 & 1 & 0 & \Delta \\
    0 & 0 & 1 & 0 \\
    0 & 0 & 0 & 1
    \end{pmatrix}, & 
    \vH_t &= \begin{pmatrix}
        1 & 0 & 0 & 0\\
        0 & 1 & 0 & 0
    \end{pmatrix},
\end{align*}}
$\!\!\Delta = 0.1$ is the sampling rate,
$q = 0.10$ is the system noise,
$r = 10$ is the measurement noise,
and ${\bf I}_K$ is a $K\times K$ identity matrix.
We simulate 500 trials, each with 1,000 steps.
For each method, we compute the scaled RMSE metric
$ J_{T,i} = \sqrt{\sum_{t=1}^T (\vtheta_{t,i}- \vmu_{t,i}) ^ 2}$
for $i\in\{0,1,2,3\}$ as well as the total running time (relative to the \mkf).

In our experiments,
the true data generating process is one of two variants of \eqref{eq:noisy-2d-ssm}.
The first variant
(which we call {\bf Student observations})
corresponds to a system whose measurement process
comes from the Student-t likelihood:
\begin{equation}\label{eq:noisy-2d-ssm-outlier-covariance}
    \begin{aligned}
    p(\vy_t \vert \params_t)
    &= \text{St}({\vy_t\,\vert\,\vH_t\params_t,\,\vR_t,\nu_t})\\
    &= \int_0^\infty {\cal N}\left(\vy_t\,\vert\,\vH_t\vtheta_t, \frac{\vR_t}{\tau}\right)\text{Gam}\left(\tau \vert \frac{\nu_t}{2}, \frac{\nu_t}{2}\right) d\tau, 
    \end{aligned}
\end{equation}
with $\text{Gam}(\cdot \vert a, b)$ the density of a Gamma distribution with shape $a$ and rate $b$, and
$\nu_t=2.01$.
The second variant 
(which we call {\bf mixture observations})
corresponds to a system where the mean of the observations
changes sporadically. Instances of this variant can occur as a form of
human error or a software bug in a data-entry program.
To emulate this scenario, we modify \eqref{eq:noisy-2d-ssm}
by using the following mixture model for the observation
process:
\begin{equation}\label{eq:noisy-2d-ssm-outlier-mean}
\begin{aligned}
p(\vy_t \vert \params_t) &= \normdist{\vy_t}{\vm_t}{\vR_t},\\
\vm_t &= 
    \begin{cases}
        \vH_t\,\params_t & \text{w.p.}\ 1 - p_\epsilon,\\
        2\,\vH_t\,\params_t & \text{w.p.}\ p_\epsilon,
    \end{cases}
\end{aligned}
\end{equation}
where $p_\epsilon = 0.05$. 

\paragraph{Results}

\begin{figure}[ht]
    \centering
    \includegraphics[width=0.45\linewidth]{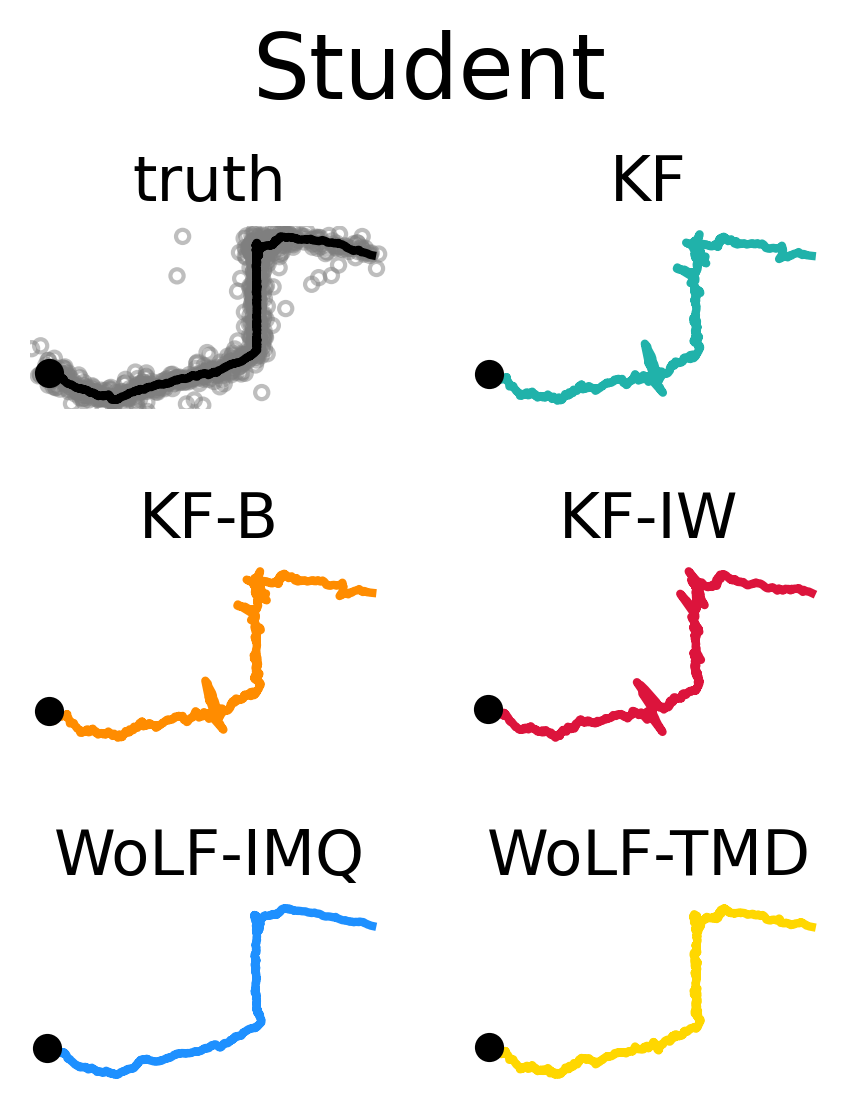}
    \hfill\vline\hfill
    \includegraphics[width=0.45\linewidth]{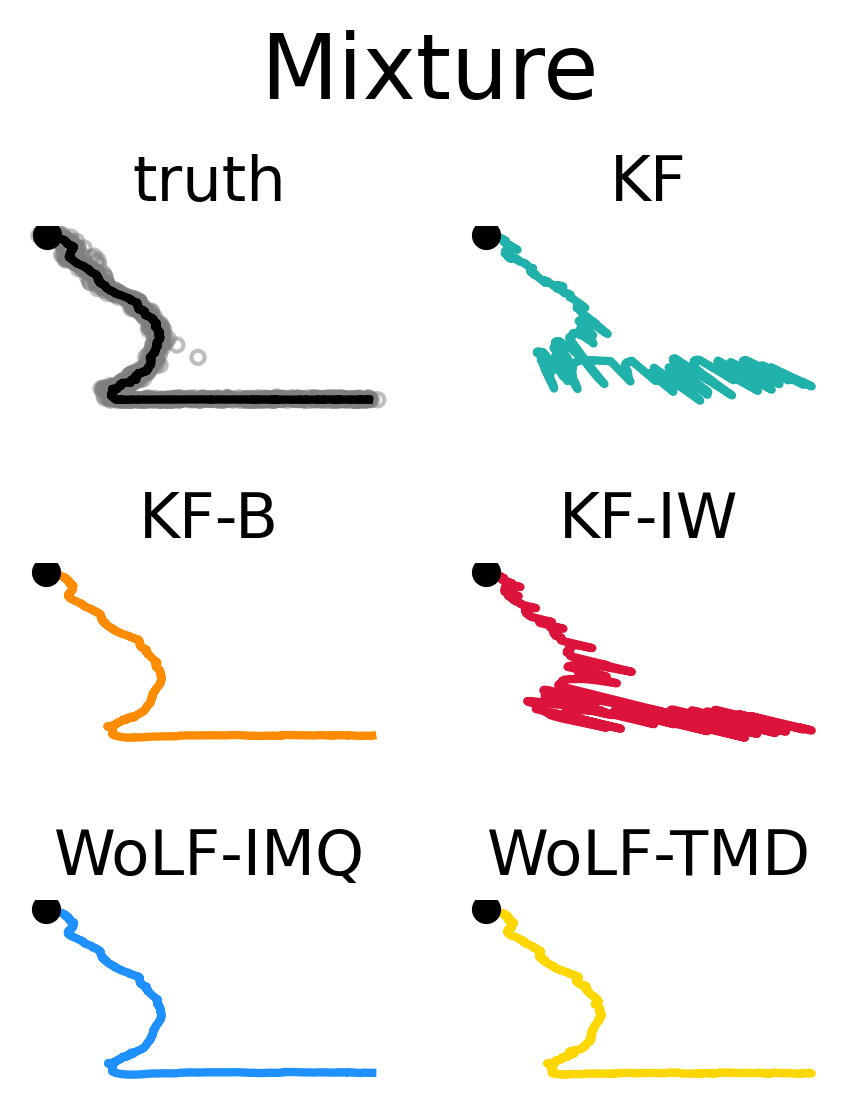}
    \caption{
    The left panel shows a sample path using the Student variant and
    the right panel shows a sample path using the mixture variant.
    The top left figure on each panel shows the true underlying state in black,
    and the measurements as grey dots.
    }
    \label{fig:linear-ssm-sample-runs}
\end{figure}

Figure \ref{fig:linear-ssm-sample-runs} shows a sample of each variant
along with the filtered state for each method.
For the Student variant (left panel),
the \mWlfImq and the \mWlfMd  estimate the true state
more closely than the competing methods.
Both the \mAgamenoni and the \mWang look comparable to the \mkf, which are not robust to outliers.
For the mixture variant (right panel),\footnote{
The top left figure in the right panel is cropped, see Figure \ref{fig:mixture-model-uncropped} for the uncropped version.
} we see that
the \mWlfImq, the \mWlfMd, and the \mWang filter the true state correctly.
In contrast, the \mAgamenoni and the \mkf are not robust to outliers.\footnote{
\mWang removes outliers that bear no information according to some criterion, but in the Student-t case, it fails.
\mAgamenoni, on the other hand, estimates a measurement covariance rather than the dispersion of such a measurement covariance.
In this sense, it is misspecified in both cases. 
}

\begin{figure}[ht]
    \centering
    \includegraphics[width=\linewidth]{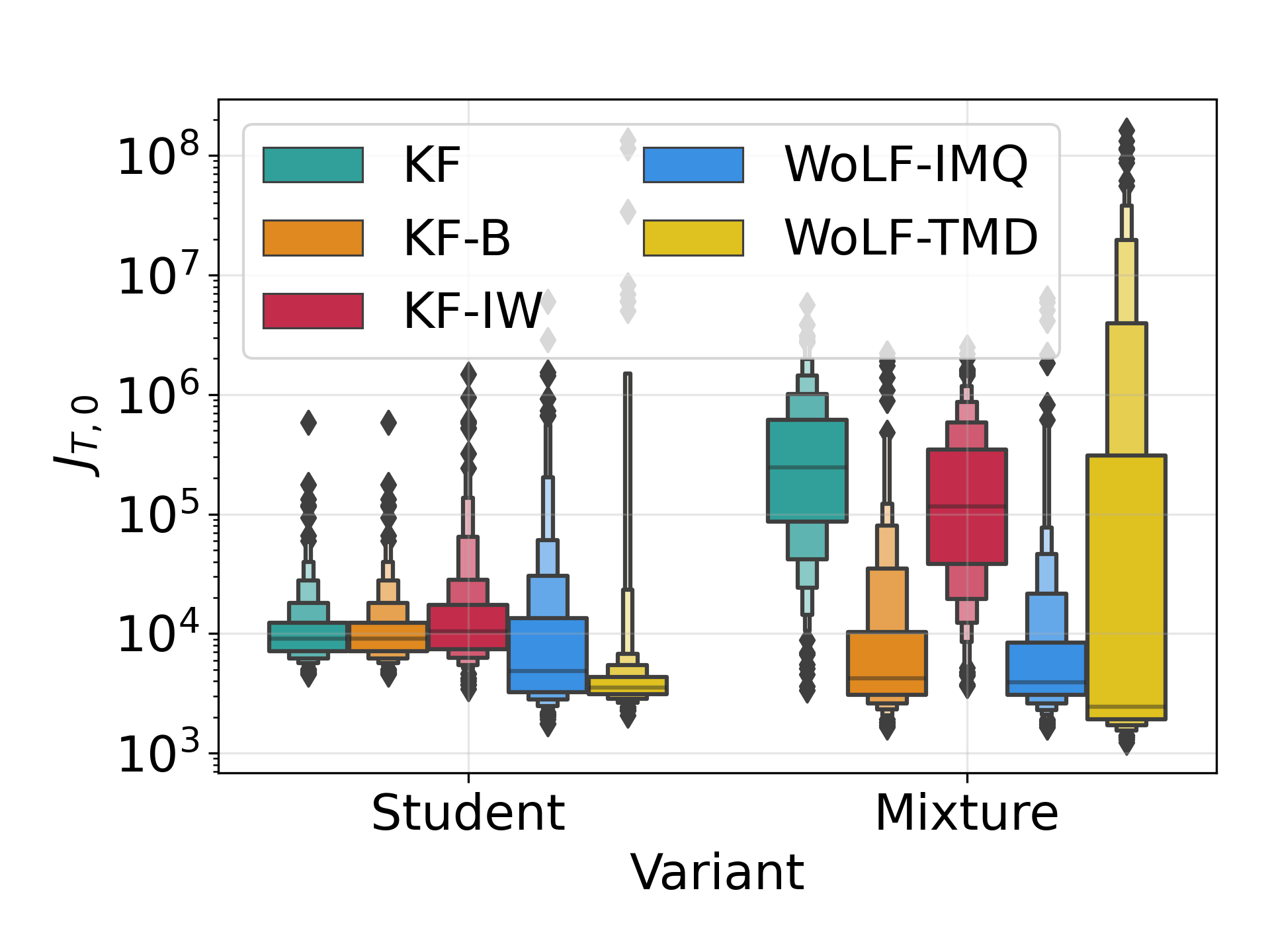}
    \removewhitespace[-8mm]
    \caption{
        Distribution  (across 500 2d tracking trials)
        of RMSE for first component of the state vector, $J_{T,0}$.
        Left panel: Student observation model.
        Right panel: Mixture observation model.
    }
    \label{fig:2d-ssm-sample}
\end{figure}

The results in Figure \ref{fig:linear-ssm-sample-runs}
hold for multiple trials as shown in Figure \ref{fig:2d-ssm-sample},
which plots the distribution
of the errors in the first component of the state vector.
The heavy tail of the \mWlfMd on the mixture observations variant,
and the distribution of $J_{T,k}$  for all $k$ are studied in
Appendix \ref{subsec:2d-tracking-further-results}.

\eat{
For the first variant, the \mWlfImq and the \mWlfMd have the lowest median $J_{T,0}$ among the methods.
The \mWlfMd has slightly lower median $J_{T,0}$ than the \mWlfImq, but the tails of the \mWlfMd are much longer
than any of the alternative methods.
For the second variant, and similar to the first variant,
we observe that \mWlfMd has the lowest median $J_{T,0}$, but the largest tail among the methods
(see Appendix \ref{subsec:2d-tracking-further-results} for an explanation of this behaviour).
Next, the median $J_{T,0}$ under \mAgamenoni is slightly lower than the \mkf.
Finally, the distribution of $J_{T,0}$ for the \mWang and the \mWlfImq are comparable,
but the \mWlfImq has slightly lower median $J_{T,0}$.
}

\begin{table}[ht]
    \centering
    \begin{tabular}{c|cc}
        \toprule
        Method & Student & Mixture \\
        \midrule
         \mWang & 2.0x & 3.7x \\
         \mAgamenoni & 1.2x & 5.3x \\
         \mWlfImq (ours) & 1.0x  & 1.0x \\
         \mWlfMd (ours) &  1.0x & 1.0x \\
         \bottomrule
    \end{tabular}
    \caption{Mean slowdown rate over \mkf.}
    \label{tab:2d-ssm-running-time}
\end{table}

Table \ref{tab:2d-ssm-running-time} shows
the median slowdown (in running time) to process the measurements relative to the \mkf.
The slowdown for method \texttt{X} is obtained 
diving the running time of method \texttt{X} over the running time of the \mkf.
Under the Student variant, the \mWlfImq, the \mWlfMd, and the \mAgamenoni have similar running time to the \mkf.
In contrast, the \mWang takes twice the amount of time.
Under the mixture variant, the \mWang and the \mAgamenoni are almost four times and five times slower than the \mkf respectively.
The changes in slowdown rate are due the number of inner iterations that were chosen during the first trial.

\subsection{Robust EKF for online MLP regression (UCI)}
\label{subsec:uci-corrupted}

In this section, we benchmark the methods using a corrupted version
of the tabular UCI regression datasets.\footnote{The dataset is available at \url{https://github.com/yaringal/DropoutUncertaintyExps}.}
Similar to other papers that deal with non-linear online learning
\citep[see, e.g.][]{lofi},
we consider a single-hidden-layer multi-layered perceptron (MLP)
with twenty hidden units and a real-valued output unit.
In this experiment, the state dimension (number of parameters in the MLP)
is $\nparams =(n_\text{in} \times 20 + 20) + (20 \times 1 + 1)$, where
$n_\text{in}$ is the dimension of the feature $\vx_t$.\footnote{
See Table \ref{tab:uci-description} for the values that $n_\text{in}$ takes for each dataset.
}
One of the main advantages of using a Bayesian filtering method for fitting neural networks
(compared to using \ogd)
is the ability to handle non-stationary distributions \citep[see e.g.][]{lofi}.
Below, we take a static state
 \citep[see e.g.][]{RVGA},
so that the prior predictive mean is $\vmu_{t|t-1} = \vmu_{t-1}$.
In Appendix \ref{experiment:training-neural-network} 
we study online learning with non-stationary environments.

Each trial is carried out as follows:
first, we randomly shuffle the rows in the dataset;
second, we divide the dataset into a warmup dataset (10\% of rows) and a corrupted dataset (remaining 90\% of rows);
third, we normalise the corrupted dataset using min-max normalisation from the warmup dataset;
fourth, with probability $p_\epsilon=0.1$,
we replace a measurement $\vy_t\in\real$ with a corrupted data point  $\vu_t \sim {\cal U}[-50, 50]$; and
fifth, we run each method on the corrupted dataset.

For each dataset and for each method, we evaluate the prior predictive
$\text{RMedSE} = \sqrt{\text{median}\{(\vy_t - h_t(\vmu_{t | t- 1})) ^ 2\}_{t=1}^T}$,
which is the squared root of the median squared error
between the measurement $\vy_t$ and the prior predictive $h_t(\vmu_{t|t-1}) = h(\vmu_{t | t-1}, \vx_t)$.\footnote{We use median instead of mean because we have outliers in measurement space.}
Here, $h$ is the MLP.
We also evaluate the average time step of each method, i.e.,
we run each method and divide the total running time by the number of samples in the corrupted dataset.

\paragraph{Results}

 Figure \ref{fig:uci-per-step-time}
shows the percentage change of the RMedSE and the percentage change of running time
with respect to those of the \ogd
for all corrupted UCI datsets.
\begin{figure}[ht]
    \centering
    \includegraphics[width=\columnwidth]{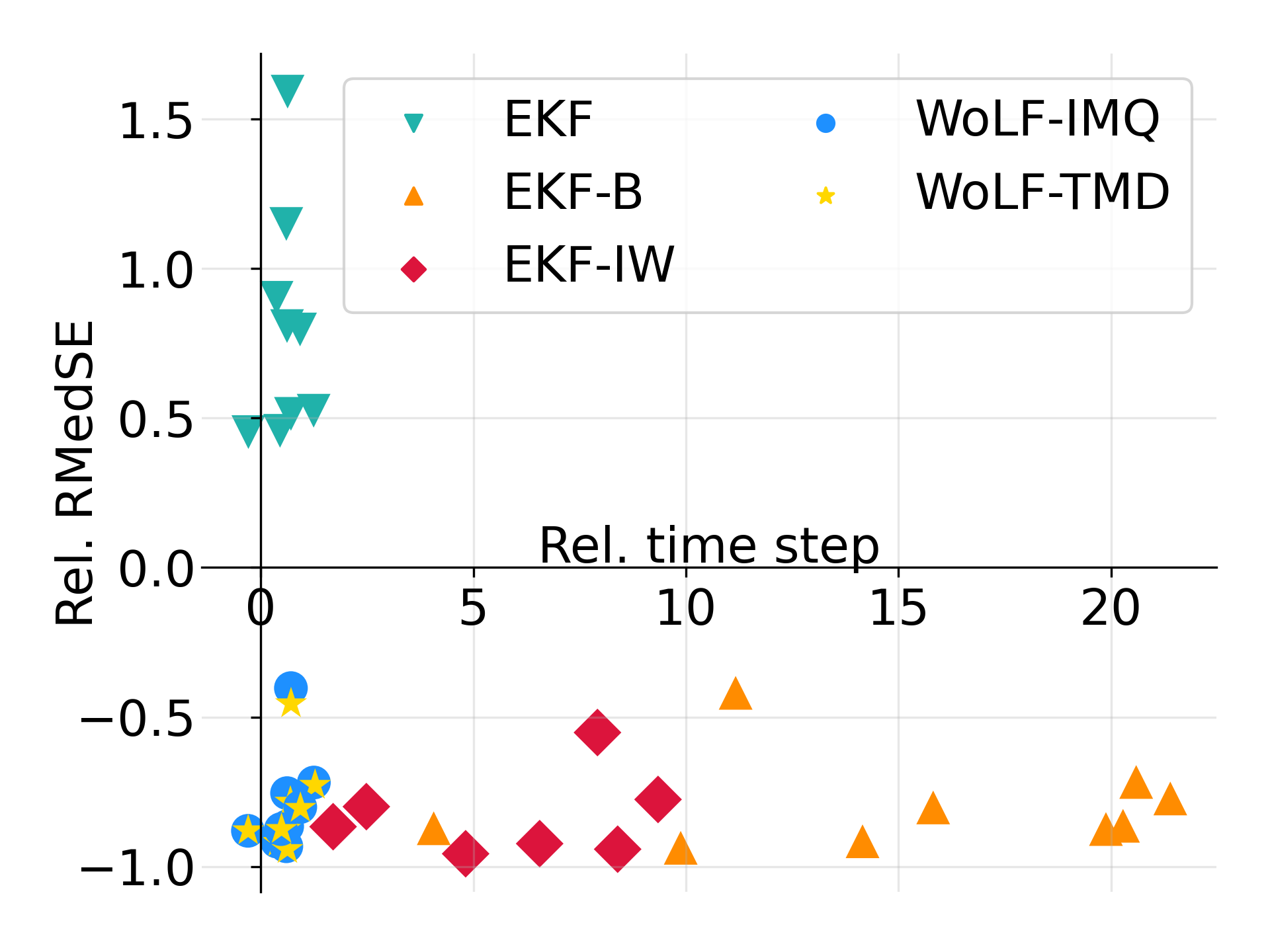}
    \removewhitespace 
    \caption{
    RMedSE
    versus time per step
    (relative to the \ogd minus $1$)
    across the corrupted UCI datasets.
    }
    \label{fig:uci-per-step-time}
\end{figure}
Given the computational complexity of the remaining methods, ideally, a robust Bayesian alternative to the \ogd
should be as much to the left as possible on the $x$-axis (rel. time step)
and as low as possible on the $y$-axis (rel. RMedSE).
We observe that the \mWlfImq and the \mWlfMd have both of these traits.
In particular, we observe that the only two points in the third quadrant are those of the \mWlfImq and the \mWlfMd.
Note that the \mAgamenoniExtended and the \mWangExtended have much higher relative running time and
the \mkfExtended has much higher relative RMedSE.

\eat{
We provide results of the RMedSE for all datasets,
perform a sensitivity analysis, and comparisons of the RMedSE versus
running time for the Kin8nm dataset in Appendix \ref{subsec:2d-tracking-further-results}.
In Appendix \ref{experiment:training-neural-network} we show that our method is robust
when used for online learning problems with non-stationary environments. 
}

\subsection{Robust EnKF for Lorenz96 model}
\label{experiment:lorenz96}

We consider a modified version of the Lorenz96 model
that is commonly used to simulate the atmosphere
\citep[see e.g.][]{lorenz2006, arnold2013stochastic} 

For a fixed $\Delta>0$, the SSM is given by 
\begin{equation}\label{eq:lorenz96-model}
\begin{aligned}
    \frac{\vtheta_{t+\Delta,i}-\vtheta_{t,i}}{\Delta}
    &= \Big(\vtheta_{t, i+1} - \vtheta_{t, i-2}\Big) \vtheta_{t, i-1} - \vtheta_{t,i} + \vphi_{t,i}, \\
    \vy_{t,i} &= 
    \begin{cases}
        \vtheta_{t,i} + \vvarphi_{t,i} & \text{w.p. } 1 - p_\epsilon,\\
        100 & \text{w.p. } p_\epsilon.
    \end{cases}
\end{aligned}
\end{equation}
Here, $\vtheta_{t,k}$ is the value of the state component $k$ at step $t$,
$\vphi_{t,i} \sim {\cal N}(8, 1)$, $\vvarphi_{t,i} \sim {\cal N}(0, 1)$, $p_\epsilon = 0.001$,
$i = 1, \ldots, d$, $t = 1, \ldots, T$, with $T \gg 1$ the number of steps,
and we use the convention $\vtheta_{t, d + k} = \vtheta_{t, k}$, $\vtheta_{t, -k} = \vtheta_{t, d - k}$.
Similar to \citet{roth2017enekf}, we integrate the state process in \eqref{eq:lorenz96-model}
to match the formulation in \eqref{eq:ssm-latent} using the Runge-Kutta-4 (RK4) procedure
with discretisation step $\Delta=0.05$, integrated over $T=10^3$ steps, $N=1,000$ number of particles,
and $\nout=\nparams=100$.
A run of the state process is shown in Figure \ref{fig:simulation-lorenz96}.
Note that the probability of an outlier happening on any state component at any timestep is
$p_\epsilon\times d= 0.1$.

\begin{figure}[htb]
    \centering
    \includegraphics[width=\columnwidth]{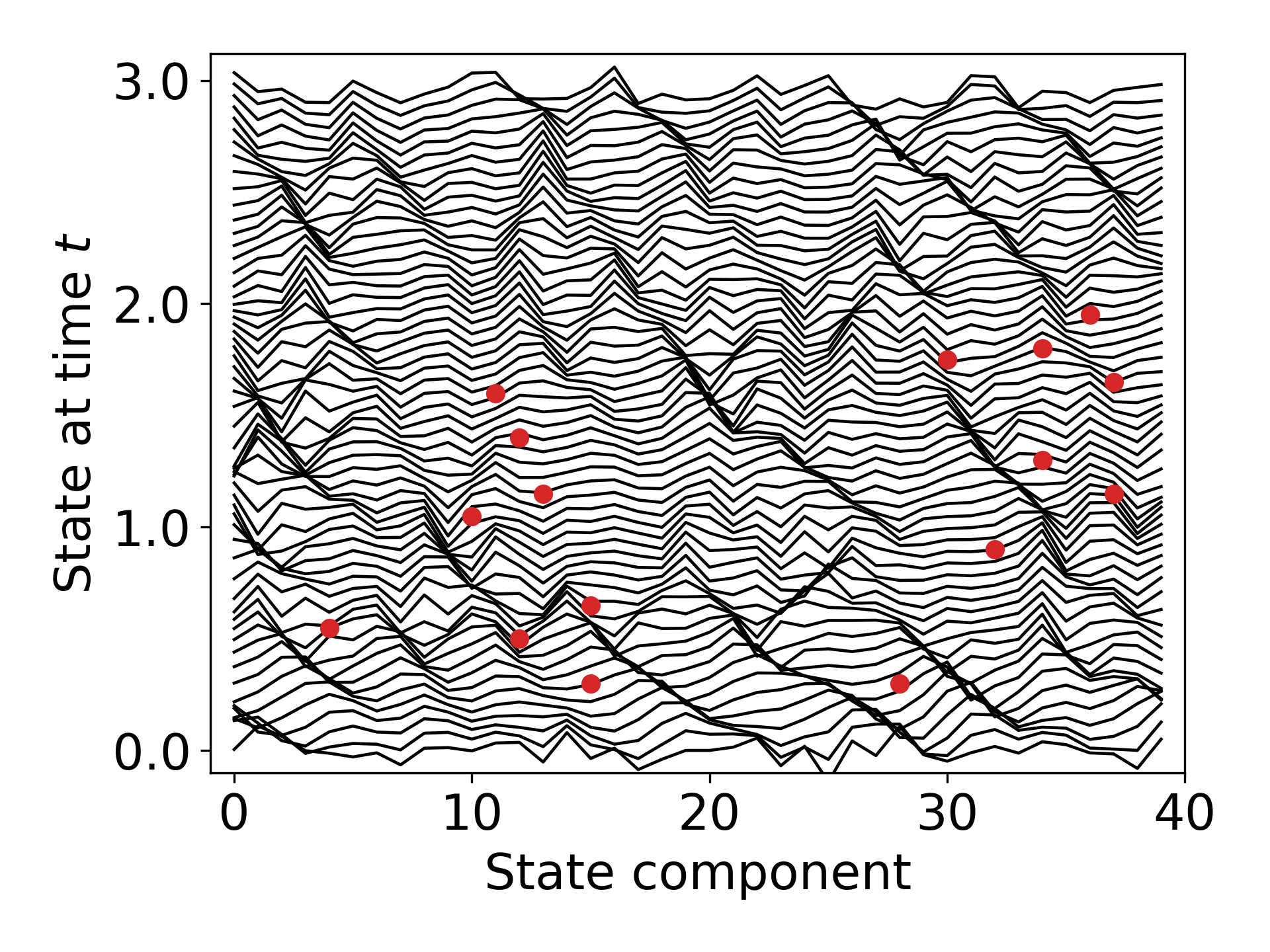}
    \removewhitespace
    \caption{
    Sample of the Lorenz process with $d=40$.
    Here, the ``waves'' move westward.
    The red dots represent where measurement outliers occur.
    We take $p_\epsilon = 0.01$ and $\Delta=0.05$ integrated over $T=60$ steps.
    }
    \label{fig:simulation-lorenz96}
\end{figure}

In this experiment, \mEnkf is the baseline.
As in \citet{roth2017ensemble},
we use the metric $L_t = \sqrt{\frac{1}{d}(\vtheta_t - \vmu_t)^\intercal (\vtheta_t - \vmu_t)}$
to measure the in-state RMSE.

\paragraph{Results}
An evaluation of $L_t$ for the \mEnkf, the \mWEnkfHard, the \mWEnkfSoft, and the \mHubEnkf is shown
in the top row of Figure \ref{fig:lorenz96-methods-comparison}.
The grey vertical lines denote timesteps where an outlier event happened, i.e., at least
one entry of $\vy_t$ is $100$.
The top plot shows that the \mHubEnkf and \mWEnkfHard have an almost-identical behaviour.
However, 
the bottom row of Figure \ref{fig:lorenz96-methods-comparison} shows
that the \mWEnkfHard and the \mWEnkfSoft are more robust to the choice of threshold $c$
compared to the \mHubEnkf.
This is because the \mHubEnkf makes updates with outlier observations clipped at $c$, whereas
WoLF-like methods disregard error measurements above $c$.
\begin{figure}[htb]
    \centering
    \includegraphics[width=\columnwidth]{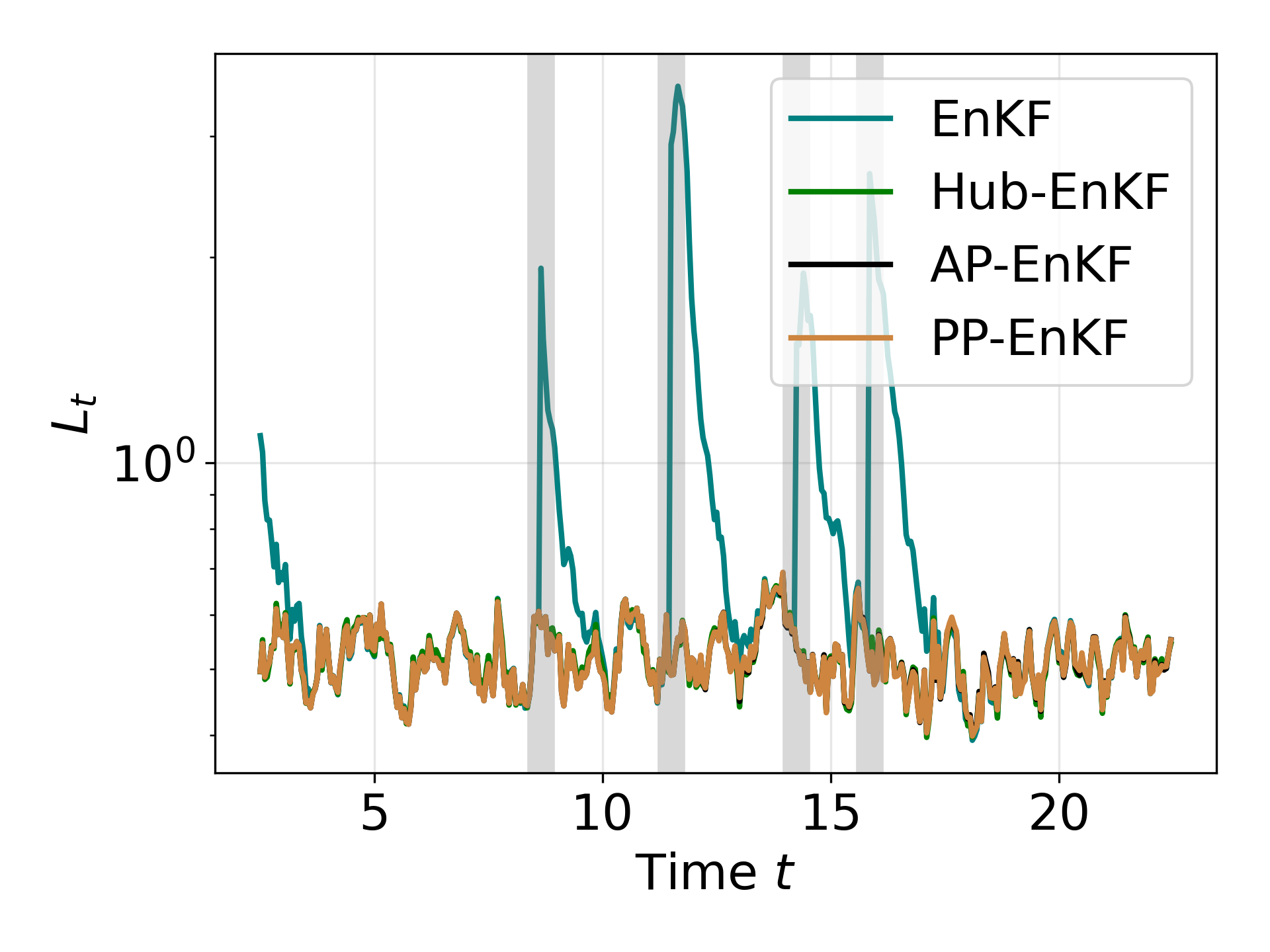}
    \includegraphics[width=\columnwidth]{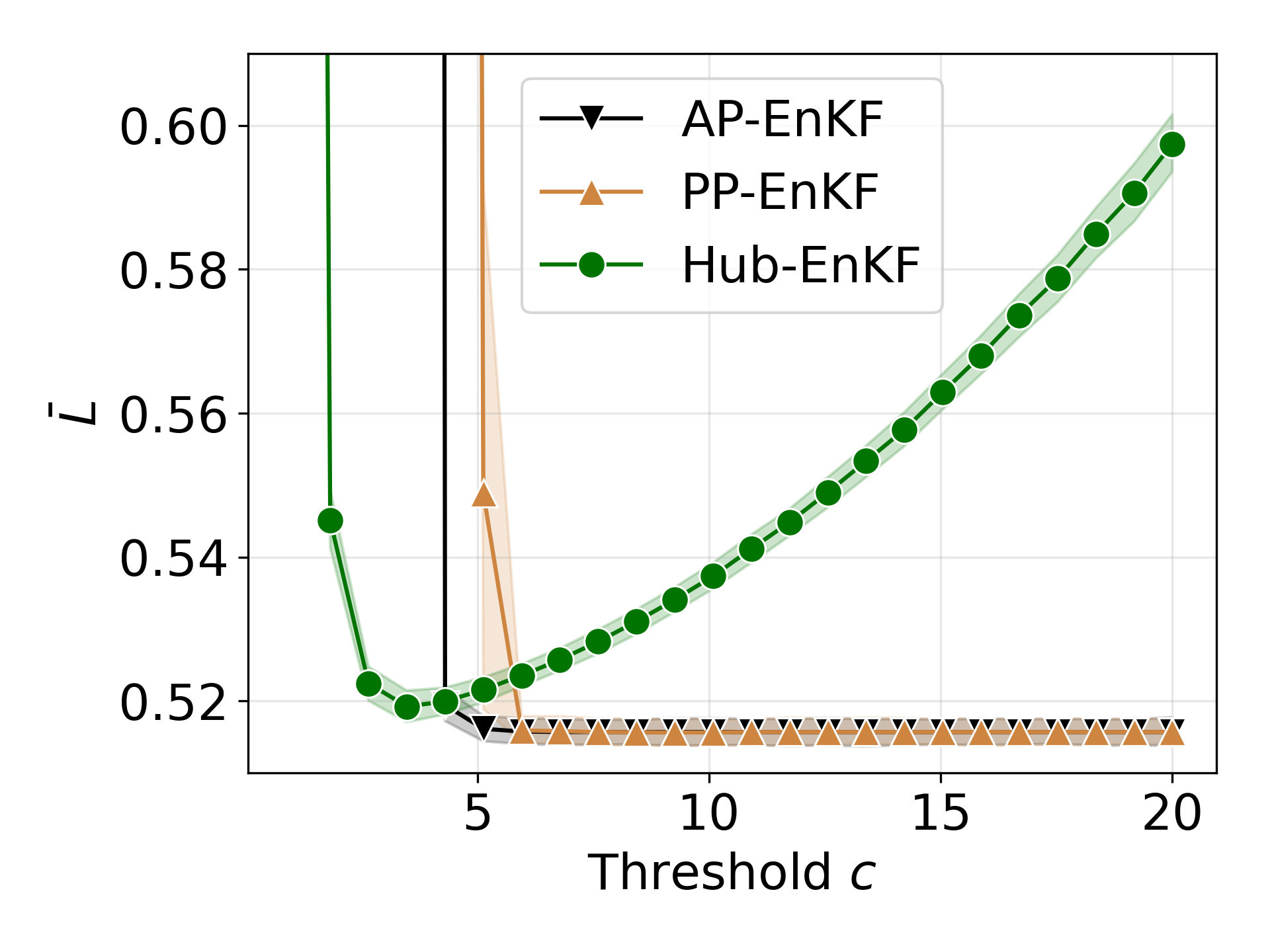}
    \removewhitespace
    \caption{
        The top row shows a sample run of the \mEnkf, the \mWEnkfHard, the \mWEnkfSoft, and the \mHubEnkf;
        outlier events are shown in grey vertical bars.
        The bottom row shows the bootstrap estimate of $L_T$ over 20 runs and 500 bootstrapped samples
        as a function of the $c$ hyperparameter.
    }
    \label{fig:lorenz96-methods-comparison}
\end{figure}
We present the results of the EnKF when the number of particles is less
than the number of state components in 
Appendix \ref{subsec:lorenz96-further-results}.
In the experiment, we modify the algorithms to make use of
the covariance inflation correction first proposed by \citet{anderson1999enkf-covariance-inflation}.
The conclusions from this section extend to the 
covariance inflation case.

\section{Conclusion}
We introduced a provably robust filtering algorithm based on generalised Bayes
which we call the weighted observation likelihood  filter or WoLF.
Our algorithm is as fast as the KF, has closed-form update equations, and is
straightforward to apply to various filtering methods.
The superior performance of the WoLF is shown on a wide range of filtering problems.
In contrast, alternative robust methods either have higher computational complexity than the WoLF, or
similar computational complexity but not higher performance.

Future work will investigate how to overcome the limitations of our approach. For example, 
(i) not being robust to outliers in the state-process,
(ii)  the assumption of a known covariance $\vR_t$,
(iii)  the assumption of known dynamics for the covariance $\vQ_t$, and
(iv) the assumption of a unimodal posterior.

\section*{Acknowledgements}
JK and FXB were supported by the EPSRC grant [EP/Y011805/1].

\section*{Impact statement}
This paper presents work whose goal is to advance the field of Machine Learning. There are many potential societal consequences of our work, none which we feel must be specifically highlighted here.

\bibliography{refs}
\bibliographystyle{icml2024}

\newpage
\appendix
\onecolumn
\vspace*{-18pt}
\section*{\LARGE \centering  Supplementary Materials
}
\vspace{14pt}

The Appendix is structured as follows.
Appendix \ref{sec:related-overview} provides an overview of  outlier-robust filtering methods.
Appendix \ref{sec:imq-as-map} derives the  WoLF-MD method from
a maximum-a-posteriori (MAP) estimate perspective.
Appendix \ref{sec:weighting-functions} collects proofs; more precisely, Appendix \ref{proof:weighted-kf} proves Proposition \ref{prop:weighted-kf} 
and Appendix \ref{app:proof_robustness} proves Theorem \ref{prop:weighted-kf-bounded-pif}.
We also show robustness for the IMQ, MD, and TMD weighting functions.
Next, in Appendix \ref{sec:wlf-extensions}, we discuss the
exponential-family, multi-output weighting function, and EnKF extensions to the WoLF.
Finally, in Appendix \ref{sec:additional-experiments},
we investigate the 2d tracking problem of Section \ref{experiment:2d-tracking} in more detail; we conduct robustness checks and
we introduce an additional experiment for corrupted non-linear non-stationary learning.

\section{
Background on existing robust filters
and related methods
}
\label{sec:related-overview}

Robust Bayesian filtering and the minimum variance estimator dates back to \citet{masreliez1975approximate} and \citet{masreliez1977robust}.
These methods propose a modified KF-recursion for the linear SSM written in terms of the score function of the measurement prior predictive.
In these early works, inference relies on Monte Carlo.

\citet{west1981robust} follows these earlier works
and proposes KF-like updates for non-normal measurement models.
That paper also studies whether several popular likelihood functions are robust in the sense of ``ignoring outliers'' --- the analysis includes the Student-t, power exponential, Huber, logistic, and stable-law likelihoods.
For the case of Student-t likelihood with one degree of freedom (i.e., Cauchy)
and linear dynamics, the update equation for the posterior mean derived in \citet{west1981robust} is equivalent to ours when using IMQ weights.
However, the approach taken in \citet{west1981robust} cannot recover our TMD scheme, due to being tied to  a given choice of measurement model.
Furthermore, \citet{west1981robust} does not provide a  theoretical foundation to use their approach in non-linear measurement models.

The work in \citet{meyr1984structure} proposes a scheme to eliminate observations that  a KF procedure labels as outliers.
Their methodology relies on a ``secondary decision system'' which checks for discrepancies between the predicted mean and the observation,
eliminating observations with high discrepancies.
This scheme is analogous to the TMD scheme, however, it is not shown to be provably robust. 

Another alternative to robustify measurement models against outliers is the work in \citet{agee1980robust}. Their paper introduces Gaussian-mixture models for robust filtering and smoothing. Their inference method is based on particle filtering.

To the best of our knowledge, the first work that proposes robust filtering in the context of robust statistics is \citet{calvet2015robust}.
Their robust filter follows \citet{masreliez1977robust} and is based on a ``Huberisation'' of the derivative of the log-measurement density (score function), which then they integrate.
See \citet{schick1994robust} for a comprehensive review of classical robust-KF methods for linear SSMs.
Recently proposed provably-robust methods include the work by \citet{boustati2020generalised} and \citet{cao2022robust}.

Finally, we discuss similarities and differences with the Bayesian learning rule (BLR) line of work by \citet{khanblr}.
BLR also allows the replacing of the likelihood with a loss function, but that loss depends only on the parameter and the observation, i.e., BLR uses $\ell_t(\vtheta) = L(\vy_t, f_t(\vtheta))$, with $L: \real^\nout\times\real^\nout$ a loss function.
By contrast, our loss depends on the weight $W_t$, which in turn depends on the belief state,
via $\hat{\vy}_t = f_t(\vmu_{t|t-1})$, which in turn is a function of all the past data, i.e.,
our loss has the form $\ell_t(\vtheta) = L(\vy_{1:t}, f_t(\vtheta))$.
In addition, WoLF is designed for the online inference setting, whereas BLR, at least in its standard form, is designed for offline inference.
Finally, note that the exponential family extension of BLR \citet{Lin2019}
has also been proposed as a way to get robustness by using scale mixture posteriors.
However, this induces robustness in parameter space, whereas we focus on robustness in observation space.
In particular, the weighting term for the BLR mixture model depends on the distance in parameter space,
$\vtheta_t - \vmu_t$, whereas ours depends on the difference in observation space, $\vy_t - \hat{\vy_t}$.

\subsection{Variational-based methods}
In this section, we provide an overview of variational-Bayes (VB) robust filtering methods.
As above, $\vtheta_t$ is the state vector of interest and 
$\vPsi_t$ are additional state parameters.
Given the SSM \eqref{eq:ssm-latent}
and measurement model $p(\vy_t \vert \vtheta_{t}, \vPsi_t)$,
VB-based methods seek an approximate posterior distribution over the extended state process $\vPhi_t = (\vtheta_t, \vPsi_t)$
that factorises as
\begin{equation}
    q(\vPhi_t) = \prod_{k=1}^K q(\vPhi_{t,k}),
\end{equation}
with $\vPhi_{t}= (\vPhi_{t,1}, \ldots, \vPhi_{t,K})$, and $K$ the number of collections.
It can be shown that the log-density $q^*$ that minimises the KL divergence
between the true posterior distribution and the variational distribution is given by
\begin{equation}\label{eq:vb-target}
    \log q^*(\vPhi_{t,k}) = \expectation{\neg k}{\log p(\vy_t, \vPhi_t)} + C,
\end{equation}
where $C$ is the normalising constant of $q^*$, and
the notation $\expectation{\neg k}{\cdot}$ denotes the conditional expectation given
all elements in $\vPhi_{t}$ except from $\vPhi_{t,k}$.
See Section 10.1.1 in \citet{bishop06} for details. Below, we discuss the robust VB-based filtering variants we use in the paper.

\subsection{KF-IW method of \citet{Agamennoni2012}}

\citet{Agamennoni2012} extend the state-space to be $\vPsi_t = (\vtheta_t, \vR_t)$,
where $\vR_t$ is the measurement covariance.
Note that the classical KF setting, $\vR_t$ is known.
The SSM is of the form
\begin{equation}
\begin{aligned}\label{eq:agamenoni-ssm}
    p(\vtheta_t \vert \vtheta_{t-1}) &= \normdist{\vtheta_t}{\vF_t\vtheta_{t-1}}{\vQ_t},\\
    p(\vR_t) &= {\cal W}^{-1}(\vR_t\,|\,\nu\vLambda, \nu),\\
    p(\vy_t \vert \vtheta_t, \vR_t) &= \normdist{\vy_t}{\vH_t\vtheta_t}{\vR_t},
\end{aligned}
\end{equation}
where ${\cal W}^{-1}(\cdot \vert {\bf P}, \eta)$ is the density of an inverse Wishart distribution
with positive-definite scale matrix ${\bf P}\in\real^{\nparams\times\nparams}$, $\eta > \nparams - 1$ degrees of freedom, and
$\nu  > 0$ is the noise-scaling hyperparameter.
They consider the class of variational distributions
\begin{equation}\label{eq:agamenoni-vb-factorisation}
    q(\bar{\vtheta}, \bar{\vR}) = q(\vtheta_0) q(\vR_0) \prod_{t=1}^T q(\vtheta_t \vert \vtheta_{t-1}) q(\vR_t),
\end{equation}
with $\bar{\vtheta} = (\vtheta_0, \ldots, \vtheta_T)$ and
$\bar{\vR} = (\vR_0, \ldots, \vR_T)$.
They show that the class of VB posteriors \eqref{eq:vb-target},
under the model in \eqref{eq:agamenoni-ssm} and \eqref{eq:agamenoni-vb-factorisation},
take the form
\begin{align}
    q(\vtheta_t \vert \vtheta_{t-1}) &= \normdist{\vtheta_t}{\vmu_t}{\vSigma_t},\\
    q(\vR_t) &= {\cal W}^{-1}(\vR_t \vert \nu_t \vLambda_t, \nu_t),
\end{align}
with $\vmu_t$, $\vSigma_t$, $\vLambda_t$, and $\nu_t$ specified in Algorithm \ref{algo:agamenoni-step}.
\begin{algorithm}[htb]
\begin{algorithmic}
    \REQUIRE $\vF_t$, $\vQ_t$, $\vmu_{t-1}$, $\vSigma_{t-1}$ // predict step
    \STATE $\vmu_{t|t-1} \gets \vF_t\,\vmu_{t-1}$
    \STATE $\vSigma_{t|t-1} \gets \vF_t\,\vSigma_{t-1} \,\vF_t^\trans + \vQ_t$
    \STATE $\vmu_t, \vSigma_t \gets \vmu_{t | t-1}, \vSigma_{t|t-1}$
    \REQUIRE $\vy_t$, $\vH_t$, $\vR_t$, $\ell\in\real_+$  // update step
    \FOR{$i = 1, \ldots, I$}
    \STATE $\vS_t \gets (\vy_t - \vH_t\,\vmu_{t})(\vy_t - \vH_t\,\vmu_{t})^\intercal + \vH_t^\intercal\,\vSigma_{t}\,\vH_t$
    \STATE $\vLambda_t \gets (\ell+1)^{-1}(\ell\,\vR_0 + \vS_t)$ 
    \STATE $\vK_t \gets (\vH_t\,\vSigma_{t|t-1} \vH_t + \vLambda_t)^{-1}\vH_t^\intercal\,\vSigma_{t|t-1}$
    \STATE $\vmu_t \gets \vmu_{t | t- 1} + \vK_t^\intercal(\vy_t - \vH_t^\intercal\,\vmu_{t|t-1})$
    \STATE $\vSigma_t \gets \vK_t^\intercal\,\vLambda_t\,\vK_t + (\vI - \vH_t\,\vK_t)^\intercal\,\vSigma_{t|t-1}(\vI - \vH_t\,\vK_t)$
    \ENDFOR
\end{algorithmic}
\caption{
    \citet{Agamennoni2012} predict and update  step for i.i.d. noise with $I\geq1$
    inner iterations.
}
\label{algo:agamenoni-step}
\end{algorithm}
In this method the hyperparameters are
the number of iterations $I$ and
the scaling term $\ell$.
The prior measurement covariance is $\vR_0$.

\subsection{KF-B method of \citet{wang2018}}
\citet{wang2018} extend the state-space to be $\vPsi_t = (\vtheta_t, \xi_t, w_t)$,
where $w_t$ is an outlier event and $\xi_t$ is its probability.
The SSM is of the form:
\begin{equation}
\begin{aligned}
    p(\vtheta_t \vert \vtheta_{t-1}) &= \normdist{\vtheta_t}{\vF_t\vtheta_{t-1}}{\vQ_t},\\
    p(\xi_t) &= \text{Beta}(\xi_t \vert \alpha_0, \beta_0),\\
    p(\rho_t \vert \xi_t) &= \text{Bern}(\rho_t \vert \xi_t),\\
    p(\vy_t \vert \vtheta_t, \vR_t \rho_t) &=
    \begin{cases}
        \normdist{\vy_t}{h_t(\vtheta_t)}{\vR_t} & \text{if } \rho_t = 1,\\
        1 & \text{if } \rho_t = 0 .
    \end{cases}
\end{aligned}
\end{equation}
In the above equations, $\text{Beta}(\cdot | a, b)$ is the density of a Beta distribution with \textit{shape} parameters $a$ and $b$,
and $\text{Bern}(\cdot \vert \pi)$ is the mass of a Bernoulli random variable with parameter $\pi\in[0,1]$.
They consider the class of variational distributions
\begin{equation}
    q(\bar{\vtheta}, \bar{\rho}, \bar{\xi}) = q(\vtheta_0) q(\rho_0) q(\xi_0) \prod_{t=1}^T q(\vtheta_t \vert \vtheta_{t-1})  q(\rho_t) q(\xi_t),
\end{equation}
with
$\bar{\vtheta} = (\vtheta_0, \ldots, \vtheta_T)$,
$\bar{\rho} = (\rho_0, \ldots, \rho_T)$,
$\bar{\xi} = (\xi_0, \ldots, \xi_T)$.
We provide the the predict and update equations in Algorithm \ref{algo:wang-step}.
\begin{algorithm}[htb]
\begin{algorithmic}
    \REQUIRE $\alpha_0, \beta_0$, $\vmu_{t-1}$, $\vSigma_{t-1}$ 
    \REQUIRE $\vF_t$, $\vQ_t$ // predict step
    \STATE $\vmu_{t|t-1} \gets \vF_t\,\vmu_{t-1}$
    \STATE $\vSigma_{t|t-1} \gets \vF_t\,\vSigma_{t-1}\,\vF_t^\trans + \vQ_t$
    \STATE $\vmu_t, \vSigma_t \gets \vmu_{t | t-1}, \vSigma_{t|t-1}$
    \REQUIRE $\vy_t$, $\vH_t$, $\vR_t$, $\text{tol.} \ll 1$  // update step
    \STATE $\rho_t, \alpha', \beta' \gets 1, \alpha_0, \beta_0$
    \FOR{$i = 1, \ldots, I$}
    \IF{$\rho_\epsilon < \text{tol.}$}
        \STATE $\vmu_t \gets \vmu_{t|t-1}$
        \STATE $\vSigma_t \gets \vSigma_{t|-1}$
    \ELSE
        \STATE $\bar{\vR}_t \gets \vR / \rho_t$
        \STATE $\hat{\vy}_t \gets \vH_t\,\vmu_{t|t-1}$
        \STATE $\vSigma_{t}^{-1}  \gets \vSigma_{t\vert t-1}^{-1} +
        \,\vH_{t}^{\top}\, \bar{\vR}_{t}^{-1} \,\vH_{t}$
        \STATE $\vK_t \gets \,\vSigma_t\,\vH_t^\trans\,\bar{\vR}_t^{-1}$
        \STATE $\vmu_{t}  \gets \vmu_{t\vert t-1}+ \vK_t \left(\vy_{t}-\hat{\vy}_{t}\right)$
    \ENDIF

    \STATE $\vB_t \gets \expectation{\vtheta \sim {\cal N}(\vmu_t, \vSigma_t)}{(\vy_t - h_t(\vtheta)(\vy_t - h_t(\vtheta))^\intercal }$
    \STATE $\log\bar\pi_t \gets \Psi(\alpha') - \Psi(\alpha' + \beta' + 1)$
    \STATE $\log(1 - \bar\pi_t) \gets \Psi(\beta' +1) - \Psi(\alpha' + \beta' + 1) $
    \STATE $\rho_t \gets \frac{\exp(\log\bar\pi_t - \text{Tr}(\vB_t\vR_t^{-1}) / 2)}{\exp(\log\bar\pi_t - \text{Tr}(\vB_t\vR_t^{-1}) / 2) + \exp(\log(1 - \bar\pi_t))}$
    \STATE $\alpha' \gets \alpha_0 + \rho_t$
    \STATE $\beta' \gets \beta_0 + 1 - \rho_t$
    \ENDFOR
\end{algorithmic}
\caption{
    \citet{wang2018} predict and update step with $I\geq1$
    inner iterations.
}
\label{algo:wang-step}
\end{algorithm}
Here, the hyperparameters are the prior rates $\alpha_0$ and $\beta_0$, and the number of inner iterations $I$,
$\Psi(\cdot)$ is the digamma function, and
$\vB_t$ is of closed form after linearising the measurement function.

\subsection{\citet{ting2007}}

\citet{ting2007} extend the state-space to be $\vPsi_t = (\vtheta_t, w_t)$,
where $w_t$ is a weighting term for the observation covariance $\vR$.
In their method, $\vR$ is known and fixed.
The SSM is of the form:
\begin{equation}\label{eq:ting-ssm}
\begin{aligned}
    p(\vtheta_t \vert \vtheta_{t-1}) &= \normdist{\vtheta_t}{\vF\vtheta_{t-1}}{\vQ},\\
    p(w_t) &= {\rm Gam}(w_t \vert a_{w}, b_{w}),\\
    p(\vy_t \vert \vtheta_t) &= \normdist{\vy_t}{\vH\vtheta_t}{\vR / w_t},
\end{aligned}
\end{equation}
for a diagonal dynamics covariance $\vQ$,
$a_w, b_w > 0$,
and diagonal observation covariance $\vR$.
They consider the class of variational distributions
\begin{equation}
    q(\bar w, \bar{\vtheta}) = q(\vtheta_0) \prod_{t=1}^T q(\vtheta_t \vert \vtheta_{t-1}) q(w_t),
\end{equation}
with
$\bar{\vtheta} = (\vtheta_0, \ldots, \vtheta_T)$ and
$\bar{w} = (w_0, \ldots, w_T)$.
They show, for known $\vF$, $\vH$, $\vQ$, and $\vR$, that the variational distributions are of the form
\begin{align}
    q(\vtheta_t \vert \vtheta_{t-1}) &= \gauss(\vtheta_t \vert \vmu_t, \vSigma_t),\\
    q(w_t) &= \text{Gam}(w_t \vert a_{w,t}, b_{w,t}),
\end{align}
where
\begin{equation}
\begin{aligned}
    a_{w,t} &= a_{w} + \frac{1}{2},\\
    b_{w,t} &= b_w + \mathbb{E}_{\vtheta \sim {\cal N}(\vmu_{t}, \vSigma_t)}[(\vy_t - \vH\vtheta)^\intercal \vR^{-1}(\vy_t - \vH\vtheta)], \\
    \vSigma_t^{-1} &= \vQ^{-1} + v_t\vH^\intercal\vR^{-1}\vH,\\
    \vK_t &= \vSigma_t\vH^\intercal\vR^{-1},\\
    \vmu_{t} &= \vF\vmu_{t-1}  v_t\,\vK_t(\vx_t - \vH\vF\vmu_{t-1}),\\
    v_t &= \frac{a_{w,t} + \frac{1}{2}}{b_{w,t} + \mathbb{E}_{\vtheta \sim {\cal N}(\vmu_{t}, \vSigma_t)}[(\vy_t - \vH\vtheta)^\intercal \vR^{-1}(\vy_t - \vH\vtheta)]}.
\end{aligned}
\end{equation}
Their method assumes no prior knowledge of either the measurement matrix $\vH$ or the projection matrix $\vF$.
These are estimated using the EM algorithm.
Assuming known $\vH$ and $\vF$ --- as we do in this paper ---
allows us to bypass the M-step.
However,
this is detrimental to their approach since no information about the posterior covariance is propagated forward.

\subsection{\citet{Huang2016}}
\citet{Huang2016} extend the state-space to be $\vPsi_t = (\vtheta_t, \vR_t, \nu_t, w_t)$,
where
$\vR_t$ is the measurement covariance, 
$w_t$ is a weighting term for the measurement covariance, and
$\nu_t$ are the degrees of freedom for the weighting term.
The SSM takes the form:
\begin{equation}
\begin{aligned}
    p(\vtheta_t \vert \vtheta_{t-1}) &= \normdist{\vtheta_t}{f_t(\vtheta_{t-1})}{\vQ_t},\\
    p(\nu_t) &= {\rm Gam}(\nu_t\,\vert\,a_t, b_t),\\
    p(w_t) &= {\rm Gam}(\lambda_t \vert \nu_t / 2, \nu_t / 2),\\
    p(\vR_t) &= {\cal W}^{-1}(\vR_t\,|\,\vLambda_t, u_t),\\
    p(\vy_t \vert \vtheta_t, \vR_t) &= \normdist{\vy_t}{h_t(\vtheta_t)}{\vR_t/w_t},
\end{aligned}
\end{equation}
with $\vQ_t = \text{diag}(q_{t,1}, \ldots, q_{t,D})$.
Note that their method combines the SSMs in \citet{ting2007} and \citet{Agamennoni2012}.

\section{WoLF-MD as a MAP estimator}\label{sec:imq-as-map}
In this section we show how to derive WoLF-MD as a MAP estimator. This is an alternative derivation that circumvents the use of generalised Bayes.

\subsection{Deriving the Mahalanobis IMQ term}

Consider the modified observation model of \eqref{eq:ting-ssm}:
\begin{align}
p(w_{t}) & ={\rm Gam}(w_{t}\vert\alpha,\beta),\\
p(\vy_{t}\vert\hat{\vy}_{t}) & =\gauss(\vy_{t}\vert\hat{\vy}_{t},w_{t}^{-1}\vR_{t}),
\end{align}
with $\hat{\vy}_t =  \vH_t \vmu_{t|t-1}$ known at time $t$, and $\alpha,\beta>0$.
The posterior on $w_{t}$ is
\begin{align}
p(w_{t}\vert\vy_{t}) & \propto w_{t}^{\alpha-1}e^{-\beta w_{t}}\left|w_{t}\vR_{t}^{-1}\right|^{1/2}\exp\left(-\frac{1}{2}\ve_{t}^{\top}w_{t}\vR_{t}^{-1}\ve_{t}\right)\\
 & \propto{\rm Gam}\left(w_{t}\vert\alpha+\frac{n_{y}}{2},\beta+\frac{1}{2}\left\Vert \ve_{t}\right\Vert _{\vR_{t}^{-1}}^{2}\right),
\end{align}
where $\ve_{t}=\vy_{t}-\hat{\vy}_{t}$ and $\left\Vert \vv\right\Vert _{\vA}=\left\Vert \vA^{1/2}\vv\right\Vert _{2}$
is the Mahalanobis distance.
The maximum-a-posteriori (MAP) estimate for $w_{t}$ is
\begin{equation}
W_{t}
= \argmax_{w_t\in\real^+} p(w_t \vert \vy_{t})
=\frac{\alpha+\frac{n_{y}}{2}-1}{\beta+\frac{1}{2}\left\Vert \ve_{t}\right\Vert _{\vR_{t}^{-1}}^{2}}\label{eq:w*}.
\end{equation}
For a given $c\in\real$, take the hyperparameters $\alpha$ and $\beta$ to be
\begin{align}
\alpha  =\frac{c^{2}-n_{y}+2}{2}\,, &\quad \beta =\frac{c^{2}}{2}\,,
\end{align}
where $n_{y}$ is the number of measurements.
We obtain 
\begin{align}
w_{t} & =\left(1+\frac{\left\Vert \ve_{t}\right\Vert _{\vR_{t}^{-1}}^{2}}{c^{2}}\right)^{-1/2}.
\end{align}
This is the Mahalanobis-based IMQ weighting function \eqref{eq:mahalanobis-weight}.
Substituting the MAP estimate back into the observation model yields the weighted
loglikelihood approximation
\begin{equation}
\log p\left(\vy_{t}\vert\hat{\vy}_{t}\right)\approx w_{t}^{2}\log\gauss\left(\vy_{t}\vert\hat{\vy}_{t},\vR_{t}\right).
\end{equation}

\subsection{Prior Uncertainty\label{subsec:Prior-Uncertainty}}
In this section, we take the measurement mean to be the output of a predictive model $\bar{\vy}_t$
with unknown parameter $\vtheta_{t}$. We let
\begin{align}
p\left(\vtheta_{t}\vert\vy_{1:t-1}\right)
& =\gauss\left(\vtheta_{t}\vert\vmu_{t\vert t-1},\vSigma_{t\vert t-1}\right),\\
p\left(\vy_{t}\vert \vtheta_{t},w_{t}\right)
& =\gauss\left(\vy_{t}\vert\bar{\vy}_t, w_{t}^{-1}\vR_{t}\right),
\end{align}
where $\bar{\vy}_t$ is given by \eqref{eq:linearised-measurement-mean}.
The joint posterior is 
\begin{equation}
\begin{aligned}
p\left(\vtheta_{t},w_{t}\vert\vy_{1:t}\right) & \propto w_{t}^{\alpha-1}e^{-\beta w_{t}}\exp\left(-\frac{1}{2}\left(\vtheta_{t}-\vmu_{t\vert t-1}\right)^{\top}\vSigma_{t\vert t-1}^{-1}\left(\vtheta_{t}-\vmu_{t\vert t-1}\right)\right)\nonumber \\
 & \qquad\times\left|w_{t}\vR_{t}^{-1}\right|^{1/2}\exp\left(-\frac{1}{2}\left(\vy_{t}-\bar{\vy}_{t}\right)^{\top}w_{t}\vR_{t}^{-1}\left(\vy_{t}-\bar{\vy}_t\right)\right)\\
 & \propto w_{t}^{\alpha-1+n_{y}/2}\exp\left(-\frac{1}{2}\left\Vert \vtheta_{t}-\vmu_{t\vert t-1}-\left(\vSigma_{t\vert t-1}^{-1}+w_{t}\vH_{t}^{\top}\vR_{t}^{-1}\vH_{t}\right)^{-1}w_{t}\vH_{t}^{\top}\vR_{t}^{-1}\ve_{t}\right\Vert _{\vSigma_{t\vert t-1}^{-1}+w_{t}\vH_{t}^{\top}\vR_{t}^{-1}\vH_{t}}^{2}\right)\nonumber \\
 & \qquad\times\exp\left(-\beta w_{t}-\frac{1}{2}\ve_{t}^{\top}\left(w_{t}^{-1}\vR_{t}+\vH_{t}\vSigma_{t\vert t-1}\vH_{t}^{\top}\right)^{-1}\ve_{t}\right).
\end{aligned}
\end{equation}
where the notation $||{\vx}||_\vA$ means $ \vx^\intercal\, \vA^{-1}\, \vx $.
Then, the marginal for $w_t$ can be written as
\begin{align}
p\left(w_{t}\vert\vy_{1:t}\right) & \propto w_{t}^{\alpha-1+n_{y}/2}\left|\vSigma_{t\vert t-1}^{-1}+w_{t}\vH_{t}^{\top}\vR_{t}^{-1}\vH_{t}\right|^{-1/2}\exp\left(-\beta w_{t}-\frac{1}{2}\ve_{t}^{\top}\left(w_{t}^{-1}\vR_{t}+\vH_{t}\vSigma_{t\vert t-1}\vH_{t}^{\top}\right)^{-1}\ve_{t}\right)\label{eq:wt-marginal}.
\end{align}
By taking the limit $\vSigma_{t\vert t-1}\to\bm{0}$, equation
\eqref{eq:wt-marginal} becomes
\begin{equation}
\lim_{\vSigma_{t\vert t-1}\to\bm{0}}p\left(w_{t}\vert\vy_{1:t}\right)\propto w_{t}^{\alpha-1+n_{y}/2}\exp\left(-\beta w_{t}-\frac{1}{2}\ve_{t}^{\top}\vR_{t}^{-1}\ve_{t}w_{t}\right),
\end{equation}
with maximum at
\begin{equation}
w_{t}^{*}=\argmax_{w_t}\lim_{\vSigma_{t\vert t-1}\to\bm{0}} = \frac{\alpha-1+\frac{n_{y}}{2}}{\beta+\frac{1}{2}\ve_{t}^{\top}\vR_{t}^{-1}\ve_{t}},
\end{equation}
that matches \eqref{eq:w*}.
Therefore in the SSM setting,
the Mahalanobis IMQ weighting is the MAP estimate for $w_{t}$
after ignoring the prior uncertainty $\vSigma_{t\vert t-1}$.

\section{Proofs of theoretical results}
\label{sec:weighting-functions}

\subsection{Proof of Proposition \ref{prop:weighted-kf}}
\label{proof:weighted-kf}
\begin{proof}
    Let $w_t^2 := W^2(\vy_t, \hat{\vy}_t)$.
    The loss function takes the form
    \begin{equation}\label{eq:weighted-gaussian-loglikelihood}
    \begin{aligned}
        \ell_t(\vtheta_t)
        &= -w_t^2 \log \normdist{\vy_t}{\vH_t \vtheta_t}{\vR_t}\\
        &= \frac{1}{2}\left(\vy_t - \vH_t\vtheta_t\right)^\intercal(\vR_t / w_t^2)^{-1}\left(\vy_t - \vH_t\vtheta_t\right)
        -\frac{w_t^2\,\nout}{2}\log\pi - \frac{w_t^2}{2}\log|\vR_t|\\
        &= \frac{1}{2}\left(\vy_t - \vH_t\vtheta_t\right)^\intercal\bar\vR_t^{-1}\left(\vy_t - \vH_t\vtheta_t\right)
        + C,
    \end{aligned}
    \end{equation}
    with $\bar\vR_t = \vR_t / w_t^2$, and
    where $C = -\frac{w_t^2\,\nout}{2}\log\pi - \frac{w_t^2}{2}\log|\vR_t|$ is a term that does not depend on $\vtheta_t$.
    The remaining follows from the standard KF derivation. 
    Note that the loss function does not correspond to the log-likelihood for a homoskedastic Gaussian model since $\bar{\vR}_t$
    may depend on all data, including $\vy_t$.
\end{proof}
Figure \ref{fig:weighted-log-likelihood-gaussian} shows the weighted log-likelihood \eqref{eq:weighted-gaussian-loglikelihood}
for a univariate ${\cal N}(0, 1)$ Gaussian density
as a function of the weighting term $w_t^2\in(0, 1]$. 
\begin{figure}[htb]
    \centering
    \includegraphics[width=0.5\linewidth]{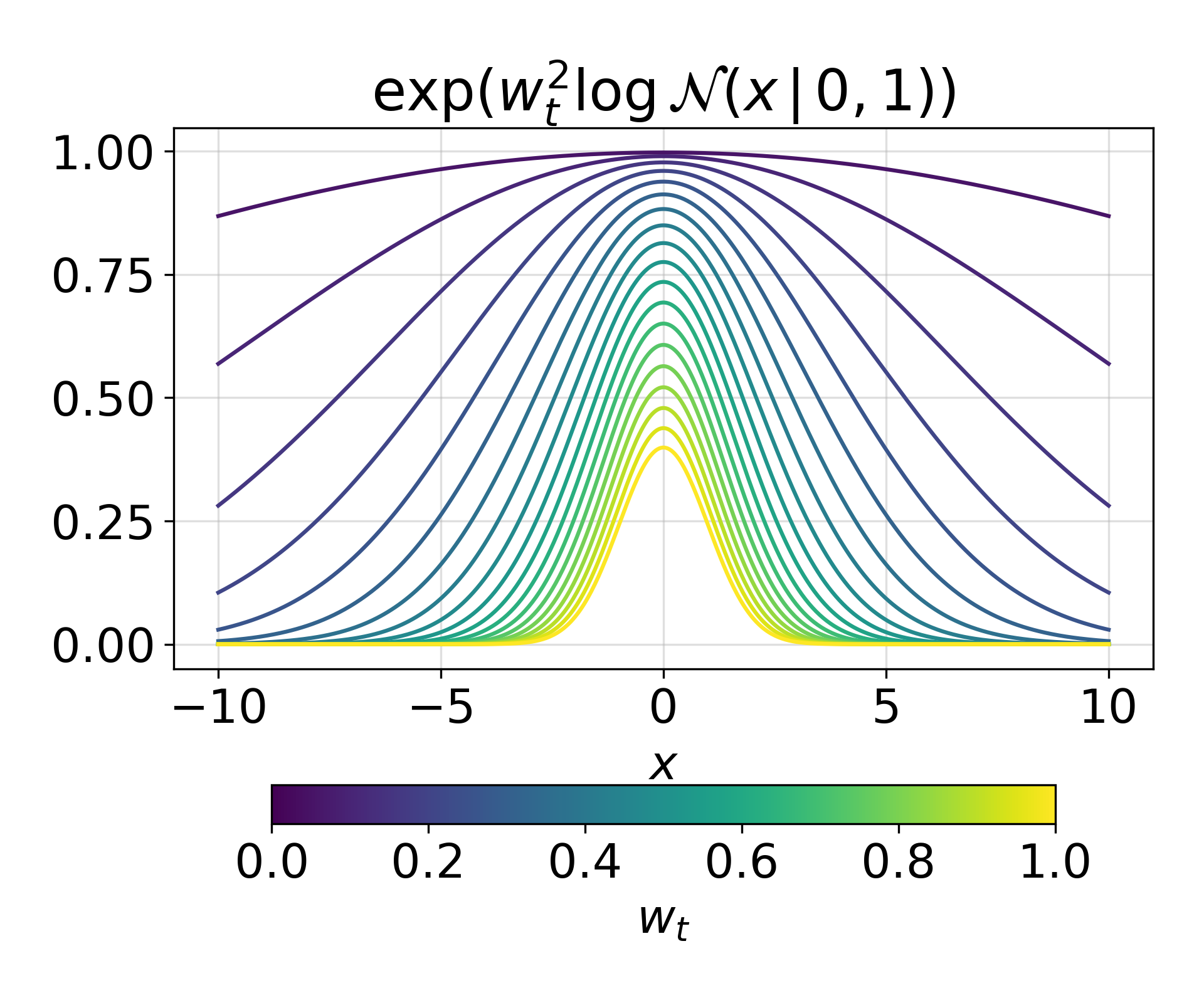}
    \removewhitespace[-5mm]
    \caption{Weighted likelihood (unnormalised) for a standard Gaussian.}
    \label{fig:weighted-log-likelihood-gaussian}
\end{figure}
We observe that a weighting log-likelihood resembles a heavy-tailed likelihood for $w_t < 1$.

\subsection{Proof of Theorem \ref{prop:weighted-kf-bounded-pif}}
\label{app:proof_robustness}
\Cref{prop:weighted-kf-bounded-pif} comprises several sub-claims, each of which we will prove as a separate lemma in this section. At the end of the section, we will integrate all the lemmas to prove the main result. The first results are for the linear Gaussian SSM case.

\begin{lemma}
\label{lemma:kf}
    Consider the linear Gaussian SSM. The standard KF posterior has an unbounded PIF and is not outlier robust.
\end{lemma}
\begin{proof}
Let $p(\vtheta_t|\vy_t,\vy_{1:t-1}) = \mathcal{N}(\vtheta_t | \vmu_{t}, \vSigma_{t})$, and
$p(\vtheta_t|\vy^{c}_{t},\vy_{1:t-1}) = \mathcal{N}(\vtheta_t| \vmu_{t}^c, \vSigma_{t}^c)$, the uncontaminated and contaminated standard Kalman filter posterior. Here
\begin{align*}
\vmu_{t} =\vmu_{t\vert t-1}+\vK_t \left(\vy_{t}-\hat{\vy}_{t}\right), &  \, &&\vmu_{t}^{c} =\vmu_{t\vert t-1}+\vK_t^{c} \left(\vy_{t}^{c}-\hat{\vy}_{t}\right),\\
\vSigma_{t}^{-1} =\vSigma_{t\vert t-1}^{-1} +
 \vH_{t}^{\intercal} \vR_{t}^{-1} \vH_{t}, & \,  && (\vSigma_{t}^{c})^{-1} =\vSigma_{t\vert t-1}^{-1} +\vH_{t}^{\intercal} \vR_{t}^{-1} \vH_{t}.
\end{align*}
Here $\vmu_{t|t-1}, \vK_t, \vH_t, \vSigma^{-1}_{t|t-1},\vR_t^{-1}$ do not depend on the contamination $\vy_t^c$.
It is clear that $\vSigma_{t}^{-1} = (\vSigma_{t}^{c})^{-1}$, hence $\vK_t^{c} = \vSigma_{t}^{c} \vH_t^{\intercal}\vR_t^{-1} = \vSigma_{t} \vH_t^{\intercal}\vR_t^{-1} = \vK_t$.
Using the fact that given two $d$-dimensional Gaussians $\mathcal{N}(\vmu_0,\vSigma_0)$ and $\mathcal{N}(\vmu_1,\vSigma_1)$,
the KL divergence is 
\begin{equation*}
    \text{KL}(\mathcal{N}(\vmu_0,\vSigma_0) \| \mathcal{N}(\vmu_1,\vSigma_1))= \frac{1}{2}\left(
    \operatorname{Tr}\left(\vSigma_1^{-1}\vSigma_0\right)  - \nparams +
    \left(\vmu_1 - \vmu_0\right)^\intercal \vSigma_1^{-1}\left(\vmu_1 - \vmu_0\right) +
    \ln\left(\frac{\det\vSigma_1}{\det\vSigma_0}\right)
  \right),
\end{equation*}
and given that $\operatorname{Tr}(\vI_p) = p$, we can derive the PIF as
\begin{align*}
    \operatorname{PIF}(\vy_t^{c},\vy_{1:t}) &=\operatorname{KL} \left(
        p(\vtheta_t|\vy^{c}_t,\vy_{1:t-1})
        \| p(\vtheta_t|\vy_t,\vy_{1:t-1})\right)\\
    &=\frac{1}{2}\Bigg(
    \Tr\left(\vSigma_{t}^{-1}\vSigma_{t}\right) - \nparams +
    \left(\vmu_{t} - \vmu_{t}^{c}\right)^\intercal \vSigma_{t}^{-1}\left(\vmu_{t} - \vmu_{t}^{c}\right)+ 
    \ln\left(\frac{\det\vSigma_{t}}{\det\vSigma_{t}}\right)\Bigg)\nonumber\\
    &=\Tr\left(\vI_\nparams\right) - \nparams + \frac{1}{2}\left(\vmu_{t} - \vmu_{t}^{c}\right)^\intercal \vSigma_{t}^{-1}\left(\vmu_{t} - \vmu_{t}^{c}\right)\nonumber\\
    &=\frac{1}{2}\left(\vy_{t} - \vy_{t}^{c}\right)^\intercal \vK_t^\intercal\vSigma_{t}^{-1}\vK_t\left(\vy_{t} - \vy_{t}^{c}\right).
    \end{align*}
Here, $\vK_t^{\intercal} \vSigma_{t}^{-1} \vK_t$ is positive definite.
To see this, first note that for every invertible matrix $B$ and conformable positive definite matrix $A$,  $B^{\intercal}AB$ is positive definite \citep[Chapter 7.1][]{horn2012matrix}. 
Since $\vK_t$  is a product of invertible matrices, it  therefore is itself invertible.
In fact, let $\vz\in \mathbb{R}^{\nparams}$ be a non-zero vector.
We can write $\vz^{\intercal} B^{\intercal}AB \vz = \tilde{\vz}^{\intercal}A\tilde{\vz}$, for $\tilde{\vz}= B\vz$.
We know that $\tilde{\vz}= B\vz \neq 0$, because $B$ is invertible.
Then, since $A$ is positive definite and $\tilde{\vz}\in\real^\nparams$,
it holds that $\tilde{\vz}^{\intercal}A\tilde{\vz}>0$, proving that $B^{\intercal}AB$ is positive definite.
Finally, since $\vSigma_{t}^{-1}$ is positive definite, it therefore holds that $\vK_t^\intercal \vSigma_{t}^{-1} \vK_t$ is positive definite, too.

This in turn allows us to lower-bound the PIF by
\begin{align*}
    \frac{1}{2}\lambda_{\min}(\vK_t^\intercal\vSigma_{t}^{-1}\vK_t)\|\vy_{t} - \vy_{t}^{c}\|^{2}_{2}\leq \frac{1}{2}\left(\vy_{t} - \vy_{t}^{c}\right)^\intercal \vK_t^\intercal\vSigma_{t}^{-1}\vK_t\left(\vy_{t} - \vy_{t}^{c}\right) = \operatorname{PIF}(\vy_t^{c},\vy_{1:t})
 , \nonumber
\end{align*}
where we use the fact that for every positive definite matrix $A$ and vector $\vz$,
we have $\lambda_{\min}(A)\|\vz\|^2\leq \vz^{\intercal}A\vz\leq\lambda_{\max}(A)\|\vz\|^2$,
where $\lambda_{\min}(A)$ and $\lambda_{\max}(A)$ are the minimum and maximum eigenvalues of $A$,
respectively \citep[Chapter 5.7][]{horn2012matrix}.
Moreover, we know that $\lambda_{\min}\left(\vK_t^\intercal \vSigma_{t}^{-1} \vK_t\right)>0$,
since $\vK_t^\intercal \vSigma_{t}^{-1} \vK_t$ is positive definite, and it does not depend on $\vy_{t}^{c}$.
Therefore, it indeed holds that $\operatorname{PIF}(\vy_t^{c},\vy_{1:t})\rightarrow + \infty$ as $\|\vy_{t} - \vy_{t}^{c}\|^{2}_{2} \to +\infty$.
\end{proof}

\begin{lemma}
    \label{lemma:robust-kf}
    Consider the linear Gaussian SSM. The generalised posterior presented in \cref{prop:weighted-kf} has bounded PIF and is, therefore, outlier robust for any weighting function $W$ such that
    $\sup_{\vy_t\in\mathbf{R}^d}W(\vy_{1:t})<\infty$ and $\sup_{\vy_t\in\mathbf{R}^d}W(\vy_{1:t})^2\,\|\vy_t\|_{2}<\infty$.
\end{lemma}
\begin{proof}
Let $W$ be a weighting function such that
    $\sup_{\vy_t\in\real^d}W(\vy_{1:t})<\infty$ and $\sup_{\vy_t\in\real^d}W(\vy_{1:t})^2\|\vy_t\|_{2}<\infty$. Lets define $w_t := W^2(\vy_{1:t})$, and $w_t^c := W^2(\vy_{t}^{c},\vy_{1:1-t})$.
Let $q(\vtheta_t|\vy_t,\vy_{1:t-1}) = \mathcal{N}(\vtheta_t | \vmu_{t}, \vSigma_{t})$, and
$q(\vtheta_t|\vy^{c}_{t},\vy_{1:t-1}) = \mathcal{N}(\vtheta_t| \vmu_{t}^c, \vSigma_{t}^c)$, where
\begin{align*}
\vmu_{t} =\vmu_{t\vert t-1}+w_t\vK_t \left(\vy_{t}-\hat{\vy}_{t}\right), &  \, &&\vmu_{t}^{c} =\vmu_{t\vert t-1}+w_t^c\vK_t^{c} \left(\vy_{t}^{c}-\hat{\vy}_{t}\right),\\
\vSigma_{t}^{-1} =\vSigma_{t\vert t-1}^{-1} +
 w_t\vH_{t}^{\intercal} \vR_{t}^{-1} \vH_{t}, & \,  && (\vSigma_{t}^{c})^{-1} =\vSigma_{t\vert t-1}^{-1} +w_t^c\vH_{t}^{\intercal} \vR_{t}^{-1} \vH_{t},
\end{align*}
as in \cref{algo:wlf-step} 
Then, the PIF has the form 
\begin{align*}
    \operatorname{PIF}(\vy_t,\vy_{1:t-1}) =
  \frac{1}{2}\Bigg(
    &\underbrace{\Tr\left((\vSigma_{t}^{c})^{-1}\vSigma_{t}\right) - p}_{(1)} 
    +
    \underbrace{\left(\vmu_{t} - \vmu_{t}^{c}\right)^\intercal (\vSigma_{t}^{c})^{-1}\left(\vmu_{t} - \vmu_{t}^{c}\right)}_{(2)}
    +
    \underbrace{\ln\left(\frac{\det\vSigma_{t}^{c}}{\det\vSigma_{t}}\right)}_{(3)}\Bigg)
 . \nonumber
\end{align*}
Now, we will get a bound for each term in the PIF. The first term can be bounded as
\begin{align*}
    (1) &= \Tr\left((\vSigma_{t}^{c})^{-1}\vSigma_{t}\right) - \nparams
    \leq \Tr\left((\vSigma_{t}^{c})^{-1}\right)\Tr\left(\vSigma_{t}\right)- \nparams,
\end{align*}
where we use the fact that for two positive semidefinite matrices $A$, $B$, it holds that $\Tr(AB)\leq\Tr(A)\Tr(B)$.
Observing that $\vSigma_{t}$ does not depend on $y_t^c$, we can now write
$C_1 = \vSigma_t$, so that by using the arithmetic rules of traces, we obtain
\begin{align*}
    (1) &\leq C_1\Tr\left(\left(\vSigma_{t}^{c}\right)^{-1}\right) - \nparams\\
    &= C_1\Tr\left(\vSigma_{t\vert t-1}^{-1} +w_t^c\vH_{t}^{\intercal} \vR_{t}^{-1} \vH_{t}\right) - \nparams\\
    &= C_1\Tr\left(\vSigma_{t\vert t-1}^{-1}\right) +
    C_1 w_t^c \Tr\left(\vH_{t}^{\intercal} \vR_{t}^{-1} \vH_{t}\right) - \nparams.
\end{align*}
Here, we use the fact that the trace of a sum is a sum of traces, and the trace of the constant times matrix is equal to the constant times trace of the matrix.
Finally, since  $\sup_{\vy_t^c\in\mathbb{R}^d} w_t^c\leq C_2<\infty$, the entire expression can be bounded by a constant
$C_3< \infty$ that does not depend on the contamination $\vy_t^c$ as
\begin{align*}
    (1) & \leq C_1\Tr\left(\vSigma_{t\vert t-1}^{-1}\right) +
    C_1 C_2 \Tr\left(\vH_{t}^{\intercal} \vR_{t}^{-1} \vH_{t}\right) - \nparams = C_3,
\end{align*}
where we use the fact that both traces are finite since both matrices are real-valued.
Next, we bound the second term by noting that it is the squared Mahalanobis norm of $\vmu_{t} - \vmu_{t}^{c}$ with respect to $\vSigma_{t}^{c}$, and we write it as $\|\vmu_{t} - \vmu_{t}^{c}\|_{\vSigma_{t}^{c}}$ (in particular, it then satisfies all the properties of a norm). Therefore, we apply the triangle inequality to obtain
\begin{align*}
(2) &= (\|\vmu_{t} - \vmu_{t}^{c}\|_{\vSigma_{t}^{c}})^2 \\
&= (\|w_t\vK_{t}(\vy_{t}-\hat{\vy}_{t}) - w_t^c\vK_{t}^c(\vy_{t}^c-\hat{\vy}_{t})\|_{\vSigma_{t}^{c}})^2 \\
&\leq (\|w_t\vK_{t}(\vy_{t}-\hat{\vy}_{t})\|_{\vSigma_{t}^{c}} + \| w_t^c\vK_{t}^c(\vy_{t}^c-\hat{\vy}_{t})\|_{\vSigma_{t}^{c}})^2\\
&\leq 2(\|w_t\vK_{t}(\vy_{t}-\hat{\vy}_{t})\|_{\vSigma_{t}^{c}}^2 + \| w_t^c\vK_{t}^c(\vy_{t}^c-\hat{\vy}_{t})\|_{\vSigma_{t}^{c}}^2).
\end{align*}
In the last inequality, we use the fact that, for two real numbers $a$ and $b$,
it holds that $(a+b)^2\leq 2(a^2+b^2)$.
Then, we bound each term separately,
\begin{align*}
 2\|w_t\vK_{t}(\vy_{t}-\hat{\vy}_{t})\|_{\vSigma_{t}^{c}}^2  &= 2w_t^{2}(\vy_{t}-\hat{\vy}_{t})^\intercal\vK_{t}^\intercal(\vSigma_{t}^{c})^{-1}\vK_{t}(\vy_{t}-\hat{\vy}_{t}) \\
 &\leq2w_t^{2}\lambda_{\max}\left(\vK_t^\intercal(\vSigma_{t}^{c})^{-1} \vK_t\right)\| \vy_{t}-\hat{\vy}_{t}\|_{2}^{2},
\end{align*}
where $2\lambda_{\max}\left(\vK_t^\intercal (\vSigma_{t}^{c})^{-1} \vK_t\right)\leq C_4<\infty$,
since $\vK_t^\intercal (\vSigma_{t}^{c})^{-1} \vK_t$ is a positive definite matrix, and thus, all eigenvalues are real-valued.
This property arises from the fact that for every invertible matrix $B$ and a positive definite matrix $A$,
$B^{\intercal}AB$ is positive definite.
Finally, we must verify that $(\vSigma_{t}^{c})^{-1}$ is positive definite.
Let $\vz\in\real^\nparams$ be a non-zero vector. Then,
\begin{align*}
\vz^{\intercal}\big(\vSigma_{t}^{c}\big)^{-1}\vz &= \vz^{\intercal}\big(\vSigma_{t\vert t-1}^{-1} +w_t^c\vH_{t}^{\intercal} \vR_{t}^{-1} \vH_{t}\big)\vz = \vz^{\intercal}\big(\vSigma_{t\vert t-1}^{-1}\big)\vz +w_t^c \big(\vz^{\intercal}\vH_{t}^{\intercal} \vR_{t}^{-1} \vH_{t}\vz\big).
\end{align*}
We know that $\vSigma_{t\vert t-1}^{-1}$ is positive definite, hence $\vz^{\intercal}\big(\vSigma_{t\vert t-1}^{-1}\big)\vz>0$. Moreover, $\vz^{\intercal}\vH_{t}^{\intercal} \vR_{t}^{-1} \vH_{t}\vz>0$ because $\vH_{t}^{\intercal} \vR_{t}^{-1} \vH_{t}$ is positive definite. Since $w_t^c=W^2(\vy_{t}^{c},\vy_{1:1-t})\geq0$ for all $\vy_t^c\in\mathbb{R}^d$, then $w_t^c \big(\vz^{\intercal}\vH_{t}^{\intercal} \vR_{t}^{-1} \vH_{t}\vz\big)\geq0$. Finally, combining these inequalities, $\vz^{\intercal}\big(\vSigma_{t}^{c}\big)^{-1}\vz>0$. Therefore, $(\vSigma_{t}^{c})^{-1}$ is positive definite.
Therefore,
\begin{align*}
 2\|w_t\vK_{t}(\vy_{t}-\hat{\vy}_{t})\|_{\vSigma_{t}^{c}}^2 \leq C_4 (w_t\| \vy_{t}-\hat{\vy}_{t}\|_{2})^{2} = C_5
\end{align*}
since $(w_t\| \vy_{t}-\hat{\vy}_{t}\|_{2})^{2}$ does not depend on the contamination. Similarly, 
\begin{align*}
 2\|w_t^c\vK_{t}^c(\vy_{t}^c-\hat{\vy}_{t})\|_{\vSigma_{t}^{c}}^2 \leq C_6 (w_t^c\| \vy_{t}^c-\hat{\vy}_{t}\|_{2})^{2},
\end{align*}
where, using the same argument as before, $2\lambda_{\max}\left((\vK_t^c)^\intercal (\vSigma_{t}^{c})^{-1} \vK_t^c\right)\leq C_6<\infty$.
Now, since $\sup_{\vy_t^c\in\mathbf{R}^d}w_t^c\|\vy_t^c\|_{2}\leq C_7<\infty$, we have:
\begin{align*}
  C_6 (w_t^c\| \vy_{t}^c-\hat{\vy}_{t}\|_{2})^{2} &\leq C_6 (w_t^c\| \vy_{t}^c\|_{2}+w_t^c\|\hat{\vy}_{t}\|_{2})^{2}\leq  C_6 (C_7+C_2\|\hat{\vy}_{t}\|_{2})^{2} = C_8.
\end{align*}
Putting it all together, we find that
\begin{align*}
(2) \leq C_5 + C_8.
\end{align*}
Lastly, the third and final term can be rewritten using  properties of determinants as 
\begin{align*}
    (3) &= \ln\left(\frac{\det\vSigma_{t}^{c}}{\det\vSigma_{t}}\right) = \ln\left(\frac{1}{\det\vSigma_{t}}\right) + \ln\left(\det \vSigma_{t}^{c}\right)
    = \ln\left(\frac{1}{\det\vSigma_{t}}\right) + \ln\left(\frac{1}{\det (\vSigma_{t}^{c})^{-1}}\right).
\end{align*}
We define $C_9 = \ln\left(\frac{1}{\det\vSigma_{t}}\right)$ since it does not depend on the contamination, and write
\begin{align*}
    (3) &= C_9 + \ln\left(\frac{1}{\det (\vSigma_{t}^{c})^{-1}}\right)\\
    &= C_9 + \ln\left(\frac{1}{\det\left(\vSigma_{t\vert t-1}^{-1} +
    w_t^c \vH_{t}^{\intercal} \vR_{t}^{-1} \vH_{t}\right)}\right)\\
    &\leq C_9 + \ln\left(\frac{1}{\det\left(\vSigma_{t\vert t-1}^{-1}\right) +\det\left(
    w_t^c \vH_{t}^{\intercal} \vR_{t}^{-1} \vH_{t}\right)}\right),\\
\end{align*}
where in the last inequality, we use the fact that for two positive semidefinite matrices $A$, $B$, it also holds that $\det(A+B)\geq\det(A) + \det(B)$. Finally,\begin{align*}
    (3) 
    &\leq C_9 + \ln\left(\frac{1}{\det\left(\vSigma_{t\vert t-1}^{-1}\right) +\det\left(
    w_t^c \vH_{t}^{\intercal} \vR_{t}^{-1} \vH_{t}\right)}\right)\\
    &\leq C_9 + \ln\left(\frac{1}{\det\left(\vSigma_{t\vert t-1}^{-1}\right)}\right) = C_{10}.\\
\end{align*}
Here, we use the fact that $w_t^c \vH_{t}^{\intercal} \vR_{t}^{-1} \vH_{t}$ is positive semidefinite, as we showed previously. Therefore $\det\left(
    w_t^c \vH_{t}^{\intercal} \vR_{t}^{-1} \vH_{t}\right) \geq 0$. By putting the bounds for (1), (2), and (3) together, we obtain
\begin{align*}
    \operatorname{PIF}(\vy_t,\vy_{1:t-1})& =
  \frac{1}{2}\Bigg(
    \underbrace{\Tr\left((\vSigma_{t}^{c})^{-1}\vSigma_{t}\right) - d}_{(1)} 
    +
    \underbrace{\left(\vmu_{t} - \vmu_{t}^{c}\right)^\intercal (\vSigma_{t}^{c})^{-1}\left(\vmu_{t} - \vmu_{t}^{c}\right)}_{(2)}
    +
    \underbrace{\ln\left(\frac{\det\vSigma_{t}^{c}}{\det\vSigma_{t}}\right)}_{(3)}\Bigg)
 . \nonumber\\
 &\leq C_3  + C_5 + C_8 + C_{10} < \infty.
\end{align*}
\end{proof}
We now extend the result to the linearised approximation of the SSM case.
\begin{lemma}
\label{lemma:ekf}
    Consider the linearised approximation of the SSM. The standard EKF posterior has an unbounded PIF and is not outlier robust.
\end{lemma}
\begin{proof}
We can easily replicate the procedure in \cref{lemma:kf} to the EKF since it can be straightforwardly applied to the approximate posterior presented in \cref{subsec:ekf}. Specifically, for the standard EKF, we compute:
\begin{align}
    \operatorname{PIF}(\vy_t^{c},\vy_{1:t}) = \operatorname{KL} \left(
        p(\vtheta_t|\vy^{c}_t,\vy_{1:t-1})
        \| p(\vtheta_t|\vy_t,\vy_{1:t-1})\right) 
\end{align}
for $p(\vtheta_t|\vy_t,\vy_{1:t-1}) = \mathcal{N}(\vtheta_t | \vmu_{t}, \vSigma_{t})$, and
$p(\vtheta_t|\vy^{c}_{t},\vy_{1:t-1}) = \mathcal{N}(\vtheta_t| \vmu_{t}^c, \vSigma_{t}^c)$, representing the uncontaminated and contaminated standard EKF posterior. In particular,
\begin{align*}
\vmu_{t} &= \vmu_{t\vert t-1} + \vK_t (\vy_{t} - \hat{\vy}_{t}),  & \vmu_{t}^{c} &= \vmu_{t\vert t-1} + \vK_t^{c} (\vy_{t}^{c} - \hat{\vy}_{t}),\\
\vSigma_{t}^{-1} &= \vSigma_{t\vert t-1}^{-1} + \vH_{t}^{\intercal} \vR_{t}^{-1} \vH_{t}, & (\vSigma_{t}^{c})^{-1} &= \vSigma_{t\vert t-1}^{-1} + \vH_{t}^{\intercal} \vR_{t}^{-1} \vH_{t}.
\end{align*}
Here, $\vmu_{t|t-1}, \vK_t, \vH_t, \vSigma^{-1}_{t|t-1}, \vR_t^{-1}$ do not depend on the contamination $\vy_t^c$, and
$\vmu_{t|t-1} = \mathbb{E}[\bar{\vmu}_{t|t-1} | \vmu_{t-1}] = f_t(\vmu_{t-1})$,
$\hat{\vy}_t = \mathbb{E}[\bar{\vy}_t] = h_t(\vmu_{t|t-1})$, where $\bar{\vmu}_{t|t-1}$ and $\bar{\vy}_t$ are defined as \cref{eq:linearised-state-mean} and \cref{eq:linearised-measurement-mean} respectively,  and
$\vH_t$ is the Jacobian of $h_t$ evaluated at $\vmu_{t|t-1}$.
Since it follows the same structure as the standard Kalman filter,
replicating the procedure in \cref{lemma:kf}, we obtain that the standard EKF is not robust.
\end{proof}
\begin{lemma}
\label{lemma:robust-ekf}
    Consider the linearised approximation of the SSM. The generalised posterior presented in \cref{prop:weighted-kf} has bounded PIF and is, therefore, outlier robust for any weighting function $W$ such that
    $\sup_{\vy_t\in\mathbf{R}^d}W(\vy_t, \hat{\vy}_t)<\infty$ and
    $\sup_{\vy_t\in\mathbf{R}^d}W(\vy_t, \hat{\vy}_t)^2\,\|\vy_t\|_{2}<\infty$.
\end{lemma}
\begin{proof}
We can easily replicate the procedure in \cref{lemma:robust-kf} since it can be straightforwardly applied to the approximate posterior presented in \cref{subsec:wlf-nonlinear-extensions}.
Now consider the weighted EKF. We compute
\begin{align}
    \operatorname{PIF}(\vy_t^{c},\vy_{1:t}) = \operatorname{KL} \left(
        q(\vtheta_t|\vy^{c}_t,\vy_{1:t-1})
        \| q(\vtheta_t|\vy_t,\vy_{1:t-1})\right) .
\end{align}

Let $W: \real^{\nout\times\nout}\to\real$ be a weighting function such that
$\sup_{\vy_t\in\mathbf{R}^d}W(\vy_t, \hat{\vy}_t)<\infty$ and
$\sup_{\vy_t\in\mathbf{R}^d}W(\vy_t, \hat{\vy}_t)^2\|\vy_t\|_{2}<\infty$.
Define $w_t := W^2(\vy_t, \hat{\vy}_t)$, and $w_t^c := W^2(\vy_{t}^{c},\hat{\vy}_t)$.

Let $q(\vtheta_t|\vy_t,\vy_{1:t-1}) = \mathcal{N}(\vtheta_t | \vmu_{t}, \vSigma_{t})$, and
$q(\vtheta_t|\vy^{c}_{t},\vy_{1:t-1}) = \mathcal{N}(\vtheta_t| \vmu_{t}^c, \vSigma_{t}^c)$, where
\begin{align*}
\vmu_{t} &= \vmu_{t\vert t-1} + w_t \vK_t (\vy_{t} - \hat{\vy}_{t}),  & \vmu_{t}^{c} &= \vmu_{t\vert t-1} + w_t^c \vK_t^{c} (\vy_{t}^{c} - \hat{\vy}_{t}),\\
\vSigma_{t}^{-1} &= \vSigma_{t\vert t-1}^{-1} + w_t \vH_{t}^{\intercal} \vR_{t}^{-1} \vH_{t}, & (\vSigma_{t}^{c})^{-1} &= \vSigma_{t\vert t-1}^{-1} + w_t^c \vH_{t}^{\intercal} \vR_{t}^{-1} \vH_{t}.
\end{align*}

for $\vmu_{t|t-1}$ and $\hat{\vy}_t$ as before. Then, following the same procedure as in the proof of \cref{lemma:robust-kf}, we show that this PIF is bounded and the method is robust. 
\end{proof}

\paragraph{Proof of Theorem \ref{prop:weighted-kf-bounded-pif}}
The proof of \cref{prop:weighted-kf-bounded-pif} follows directly from \cref{lemma:kf,lemma:robust-kf,lemma:ekf,lemma:robust-ekf}.

\subsection{Ensemble Kalman Filter}
\label{proof:ensemble-kf}
Now, we extend \cref{prop:weighted-kf-bounded-pif} to the ensemble Kalman filter case. There are a couple of caveats to proving robustness for this case. While we know the distribution of the particles in the case where the state-space model is linear and Gaussian, this is not always the case. Therefore, we propose studying the empirical measures defined by the particles. The problem with this is that the Kullback-Leibler divergence is not defined for empirical measures with different supports. This is why we propose using the 2-Wasserstein distance instead since it is unbounded and well-defined for empirical measures. 
The 2-Wasserstein distance is defined as follows: If $P$ is an empirical measure with samples $x_1,...,x_n$ and $Q$ is an empirical measure with samples $y_1,...,y_n$, 
\begin{align*}
	D_{W_2}(P, Q) =  \inf_\pi \left( \frac{1}{n} \sum_{i=1}^n \left\|x_i - y_{\pi(i)}\right\|_2^2 \right)^{1/2},
\end{align*}
where the infimum is over all permutations $\pi$ of $n$ elements. Therefore, the PIF for the ensemble Kalman filter case has the form
\begin{align}
\label{eq:empirical-PIF}
	\operatorname{PIF}(\vy_t^{c},\vy_{1:t}) =  D_{W_2}\left(\mathbb{P}_{N}, \mathbb{P}_{N}^{c}\right),
\end{align}
where $\mathbb{P}_{N}$ is the empirical measure of the (non-contaminated) particles
$\left\{\hat{\theta}^{(i)}_{t} = \hat{\vtheta}_{t\vert t-1}^{\left(i\right)}+  w_t\bar{\vK}_{t}\left(\vy_{t}-\hat{\vy}_{t\vert t-1}^{(i)}\right)\right\}_{i=1}^{N}$,
and $\mathbb{P}_{N}^{c}$ is the empirical measure of the contaminated particles
$\left\{\hat{\theta}^{(i)}_{t^c} = \hat{\vtheta}_{t\vert t-1}^{\left(i\right)} +  w_t^c\bar{\vK}_{t}\left( \vy_{t}^{c}-\hat{\vy}_{t\vert t-1}^{(i)}\right)\right\}_{i=1}^{N}$.
As in the previous definition, if $\sup_{\vy_t^c\in\mathbb{R}^{d}}|\operatorname{PIF}(\vy_t^{c},\vy_{1:t})| < \infty$,
then the posterior is called outlier-robust.
\begin{lemma}
 The generalised posterior presented in \cref{sec:weighted-ensemble-kalman-filter}
 has bounded PIF (as defined in \cref{eq:empirical-PIF}) and is,
 therefore, outlier robust for any weighting function $W:\real^{\nout\times\nout}\to\real$ such that
    $\sup_{\vy_t\in\mathbf{R}^d}W(\vy_t, \hat{\vy}_t)<\infty$ and
    $\sup_{\vy_t\in\mathbf{R}^d}W(\vy_t, \hat{\vy}_t)^2\,\|\vy_t\|_{2}<\infty$.
\end{lemma}
\begin{proof}
Let $W:\real^{\nout\times\nout} \to \real$ be a weighting function such that
$\sup_{\vy_t\in\mathbf{R}^d}W(\vy_t, \hat{\vy}_t)<\infty$ and
$\sup_{\vy_t\in\mathbf{R}^d}W(\vy_t, \hat{\vy}_t)^2\|\vy_t\|_{2}<\infty$.
Define $w_t := W^2(\vy_t, \hat{\vy}_t)$, and $w_t^c := W^2(\vy_{t}^{c}, \hat{\vy}_t)$.
Then, the PIF
\begin{align*}
	\operatorname{PIF}(\vy_t^{c},\vy_{1:t}) =
     D_{W_2}(\mathbb{P}_{N}, \mathbb{P}_{N}^{c}) =
     \inf_\pi \left( \frac{1}{N} \sum_{i=1}^N \left\|\hat{\theta}^{(i)}_{t} - \hat{\theta}^{(\pi(i))}_{t^c} \right\|_2^2 \right)^{1/2},
\end{align*}
where the infimum is over all permutations $\pi$ of $N$ elements, $\mathbb{P}_{N}$ is the empirical measure of the particles $\left\{\hat{\theta}^{(i)}_{t} = \hat{\vtheta}_{t\vert t-1}^{\left(i\right)}+  w_t\bar{\vK}_{t}\left(\vy_{t}-\hat{\vy}_{t\vert t-1}^{(i)}\right)\right\}_{i=1}^{N}$, and $\mathbb{P}_{N}^{c}$ is the empirical measure of the contaminated particles $\left\{\hat{\theta}^{(i)}_{t^c} = \hat{\vtheta}_{t\vert t-1}^{\left(i\right)} +  w_t^c\bar{\vK}_{t}\left( \vy_{t}^{c}-\hat{\vy}_{t\vert t-1}^{(i)}\right)\right\}_{i=1}^{N}$. Therefore, we know that the infimum will be smaller than considering only the identity as permutation,
\begin{align*}
	D_{W_2}(\mathbb{P}_{N}, \mathbb{P}_{N}^{c}) \leq \left( \frac{1}{N} \sum_{i=1}^N \left\|\hat{\theta}^{(i)}_{t} - \hat{\theta}^{(i)}_{t^c} \right\|_2^2 \right)^{1/2}.
\end{align*}
Now, using the equation of the particles:
\begin{align*}
	D_{W_2}(\mathbb{P}_{N}, \mathbb{P}_{N}^{c}) &\leq \left( \frac{1}{N} \sum_{i=1}^N \left\|\hat{\vtheta}_{t\vert t-1}^{\left(i\right)}+  w_t\bar{\vK}_{t}\left(\vy_{t}-\hat{\vy}_{t\vert t-1}^{(i)}\right) -\hat{\vtheta}_{t\vert t-1}^{\left(i\right)} -  w_t^c\bar{\vK}_{t}\left(\vy_{t}^{c}-\hat{\vy}_{t\vert t-1}^{(i)}\right) \right\|_2^2 \right)^{1/2}\\
   & = \left( \frac{1}{N} \sum_{i=1}^N \left\| w_t\bar{\vK}_{t}\left(\vy_{t}-\hat{\vy}_{t\vert t-1}^{(i)}\right) -  w_t^c\bar{\vK}_{t}\left(\vy_{t}^{c}-\hat{\vy}_{t\vert t-1}^{(i)}\right) \right\|_2^2 \right)^{1/2}.
\end{align*}
We use triangle inequality to obtain:
\begin{align*}
	D_{W_2}(\mathbb{P}_{N}, \mathbb{P}_{N}^{c}) &\leq
   \left( \frac{1}{N} \sum_{i=1}^N \left(\left\| w_t\bar{\vK}_{t}\left(\vy_{t}-\hat{\vy}_{t\vert t-1}^{(i)}\right)\right\|_2 + \left\|w_t^c\bar{\vK}_{t}\left(\vy_{t}^{c}-\hat{\vy}_{t\vert t-1}^{(i)}\right) \right\|_2\right)^2 \right)^{1/2}\\
   &\leq
   \left( \frac{2}{N} \sum_{i=1}^N \left\| w_t\bar{\vK}_{t}\left(\vy_{t}-\hat{\vy}_{t\vert t-1}^{(i)}\right)\right\|_2^2  + \left\|w_t^c\bar{\vK}_{t}\left(\vy_{t}^{c}-\hat{\vy}_{t\vert t-1}^{(i)}\right) \right\|_2^2 \right)^{1/2}.
\end{align*}
Lets define $C_{11} = \sum_{i=1}^N \| w_t\bar{\vK}_{t}\left(\vy_{t}-\hat{\vy}_{t\vert t-1}^{(i)}\right)\|_2^2$, a constant that does not depend on the contamination. Therefore:
\begin{align*}
	D_{W_2}(\mathbb{P}_{N}, \mathbb{P}_{N}^{c}) &\leq
   \left( \frac{2}{N} \left( C_{11} + \sum_{i=1}^N \left\|w_t^c\bar{\vK}_{t}\left(\vy_{t}^{c}-\hat{\vy}_{t\vert t-1}^{(i)}\right) \right\|_2^2 \right)\right)^{1/2}.
\end{align*}
First, we observe that $\left\|w_t^c\bar{\vK}_{t}\left(\vy_{t}^{c}-\hat{\vy}_{t\vert t-1}^{(i)}\right) \right\|_2^2 = \left(w_t^c\right)^2 \left(\vy_{t}^{c}-\hat{\vy}_{t\vert t-1}^{(i)}\right)^{\intercal}\bar{\vK}_{t}^{\intercal}\bar{\vK}_{t}\left(\vy_{t}^{c}-\hat{\vy}_{t\vert t-1}^{(i)}\right)$. Now we bound this expression using the fact that for every positive definite matrix $A$ and vector $\vz$,
we have $\lambda_{\min}(A)\|\vz\|^2\leq \vz^{\intercal}A\vz\leq\lambda_{\max}(A)\|\vz\|^2$,
where $\lambda_{\min}(A)$ and $\lambda_{\max}(A)$ are the minimum and maximum eigenvalues of $A$,
respectively \citep[Chapter 5.7][]{horn2012matrix}:
\begin{align*}
	\left\|w_t^c\bar{\vK}_{t}\left(\vy_{t}^{c}-\hat{\vy}_{t\vert t-1}^{(i)}\right)  \right\|_2^2 \leq (w_t^c)^2 \lambda_{\max}(\bar{\vK}_{t}^{\intercal}\bar{\vK}_{t})  \left\|\vy_{t}^{c}-\hat{\vy}_{t\vert t-1}^{(i)} \right\|_2^2,
\end{align*}
where $\lambda_{\max}\left(\bar{\vK}_{t}^{\intercal}\bar{\vK}_{t}\right)\leq C_{12}<\infty$,
since $\bar{\vK}_{t}^{\intercal}\bar{\vK}_{t}$ is a positive definite matrix, and thus, all eigenvalues are real-valued. Then,
\begin{align*}
	\left\|w_t^c\bar{\vK}_{t}\left(\vy_{t}^{c}-\hat{\vy}_{t\vert t-1}^{(i)}\right) \right\|_2^2 \leq C_{12}\left(w_t^c \left\|\vy_{t}^{c}-\hat{\vy}_{t\vert t-1}^{(i)} \right\|_2\right)^2.
\end{align*}
Now, since $\sup_{\vy_t^c\in\mathbb{R}^d} w_t^c\leq C_{13}<\infty$, and $\sup_{\vy_t^c\in\mathbf{R}^d}w_t^c\|\vy_t^c\|_{2}\leq C_{14}<\infty$, we have:
\begin{align*}
  C_{12}\left(w_t^c \left\|\vy_{t}^{c}-\hat{\vy}_{t\vert t-1}^{(i)} \right\|_2\right)^2 &\leq C_{12} \left(w_t^c\left\| \vy_{t}^c\right\|_{2}+w_t^c\left\|\hat{\vy}_{t\vert t-1}^{(i)}\right\|_{2}\right)^{2}\leq  C_{12} \left(C_{14}+C_{13}\left\|\hat{\vy}_{t\vert t-1}^{(i)}\right\|_{2}\right)^{2} = C_{15}.
\end{align*}
Putting it all together, we find that
\begin{align*}
	D_{W_2}(\mathbb{P}_{N}, \mathbb{P}_{N}^{c}) &\leq
   \left( \frac{2}{N} \left( C_{11} + \sum_{i=1}^N C_{15} \right)\right)^{1/2}<\infty.
\end{align*}
\end{proof}
\section{Extensions of the weighted likelihood filter}
\label{sec:wlf-extensions}

Below, we discuss a number of extension to the WoLF methodology including generalisations to (i) exponential-family members,  (ii) multi-output
weighting functions, and (iii) the EnKF.

\subsection{Exponential family likelihoods}
\label{sec:wlf-expfam-extension}
We extend WoLF for measurements modelled using an element of the  exponential family of distributions (as
first mentioned in Section \ref{subsec:wlf-nonlinear-extensions}).
Classical examples of exponential families, in addition to the Gaussian distribution, are
the Bernoulli distribution,
the Gamma distribution, 
the Beta distribution, and
the Poisson distribution.
These distributions can be considered to tackle filtering problems when
the measurements are generated respectively from
a binary process,
a process that only takes values in positive real line,
a process that takes values in the interval $[0,1]$, or
a counting process.

We take the mass function of a measurement $\vy_t \in B \subseteq \real^\nout$ to be of the form
\begin{align}
    p(\vy_t|\vtheta_t) &=
    \text{expfam}(\vy_t|\veta_t)
    = Z^{-1}(\veta_t)\exp\Big(\veta_t^\intercal T(\vy_t) + b(\vy_t) \Big).
    \label{eqn:obs}
\end{align}
with
$B$ the support of the measurement $\vy_t$,
$\veta_t \in\real^k$ the natural parameters,
$A: B\to\real^k$ the sufficient statistic function,
$b: B \to \real$ the base measure, and
$Z: \real^k\to\real$ the normalising function.
To simplify the notation, let $\veta_t = h_t(\vtheta_t)$ and $\vz_t = T(\vy_t)$.
We also define the dual (moment) parameters
\begin{equation}
\vlambda_t=\expect{\vz_t\vert\veta_t} = \nabla_{\veta_t}\log Z(\veta_t)
\end{equation}
and the conditional variance
\begin{equation}
\vR_t=\text{Cov}[\vz_t\vert\veta_t]=\nabla_{\veta_t}\vlambda_t=\nabla_{\veta_t}^{2}\log Z(\veta_t),
\end{equation}
and we take our predictive model to output the dual parameters:
\begin{align}
    \vlambda_t = h_t(\vtheta_t).
\end{align}
For example, for a Gaussian likelihood with fixed observation noise
and a linear observation model as in \eqref{eq:linear-ssm},
we have that $A(\vy_t)=\vy_t$,
$\vlambda_t = \vH_t \vtheta_t$, and
$\vR_t$ is constant.

The exponential family EKF algorithm of
\citet{Ollivier2018} approximates the likelihood in \eqref{eqn:obs} with a moment-matched Gaussian,
\begin{align}
    \label{eq:EF-EKF-likelihood}
    q(\vy_t|\vtheta_t) = \gauss(\vz_t|\bar{\vlambda}_t,\vR_t),
\end{align}
where $\bar{\vlambda}_t =\vH_t( \vtheta_t - \bar\vmu_{t | t-1}) + h_t(\bar\vmu_{t | t-1})$, and
$\vmu_{t|t-1}$ is given by \eqref{eq:linearised-state-mean}.
We modify this algorithm by including a weighting term on the log-likelihood to get
\begin{align}
\log q(\vy_{t}\vert\vtheta_{t}) &=
W_{t}(\vy_t, \hat{\vy}_t)^2
\left(-\frac{1}{2}\bar{\vlambda}_{t}^{\top}{\vR}_{t}^{-1}\bar{\vlambda}_{t}+\bar{\vlambda}_{t}^{\top}{\vR}_{t}^{-1}T(\vy_{t})\right) + {\rm cst},
\end{align}
with ${\rm cst}.$ a constant term that does not depend on $\vtheta$.
We leave the study of $W(\vy_t,\hat{\vy}_t)$ when modelling a non-Gaussian exponential family for future work.

\subsection{Dimension-specific weighting}
\label{sec:dimension-specific-weighting}
In this section, we provide details about the implementation
of a vector-valued weighting function $W_{t}: \real^\nout\times\real^\nout \to \real^\nout$
introduced in Section \ref{subsec:choice-weighting-function}.
See Section \ref{experiment:lorenz96} for an evaluation of this method.
We consider $\vy_t \in \real^\nout$ with $\nout > 1$.

If the likelihood factorises as
$p\left(\vy_{t}\vert\vtheta_{t}\right) = \prod_{j=1}^{\nout}p\left(y_{t,j}\vert\vtheta_{t}\right)$,
with $y_{t,j}$ the $j$-th element of the $t$-th measurement,
we define the weighted log-likelihood as
\begin{equation}\label{eq:weighted-independent-loglikelihood}
\log q\left(\vy_{t}\vert\vtheta_{t}\right) =
\sum_{j=1}^\nout w_{t,j}^{2}\log p\left(y_{t,j}\vert\vtheta_{t}\right),
\end{equation}
where $w_{t,j}^2 = W^2(\vy_t,\hat{\vy}_t)_j$ is the $j$-th entry of the vector-valued weight function.

For a Gaussian likelihood $p(\vy_{t}\vert\vtheta_{t})=\gauss\left(\vy_{t}\vert h_t\left(\vtheta_{t}\right),\vR_{t}\right)$
we define the weighted likelihood as
\begin{align}
    q(\vy_{t}\vert\vtheta_{t}) &=
    \gauss\left(\vy_{t}\vert h_t\left(\vtheta_{t}\right),\bar{\vR}_{t}\right), \label{eq:dim-weighted-gaussian-likelihood}\\
    \bar{\vR}_{t}^{-1} & =\Diag\left(\vw_{t}\right)\vR_{t}^{-1}\Diag\left(\vw_{t}\right),
    \label{eq:dim-weighted-Rbar}
\end{align}
which is a special case of \eqref{eq:weighted-independent-loglikelihood}
when $\vR_{t}$ is diagonal.%
\footnote{$\Diag(\vv)$ for $\vv\in\real^K$ is the $K\times K$ matrix with $\Diag(\vv)_{i,i} = \vv_{i}$ and $\Diag(\vv)_{i,j}=0$ for $i\ne j$.}
This expression scales the precision of each $y_{t,j}$ by $w_{t,j}^2$ while preserving correlations.
It can also be written in terms of dimension-specific weighting on the errors:
\begin{align}
\log q(\vy_{t}\vert\vtheta_{t})
 &=-\frac{1}{2}\tilde{E}_{t}\left(\vtheta_t\right)^{\intercal}\vR_{t}^{-1}\tilde{E}_{t}\left(\vtheta_t\right),\\
\tilde{E}_{t}\left(\vtheta_t\right) & =\Diag\left(\vw_{t}\right)\left(\vy_{t}-h_t\left(\vtheta_{t}\right)\right).
\end{align}
The EKF update from the weighted loglikelihood in \eqref{eq:dim-weighted-gaussian-likelihood} mirrors the standard EKF update
with the true observation precision $\vR_t^{-1}$ replaced by $\bar{\vR}_t^{-1}$:
\begin{align}
\vSigma_{t}^{-1} &=
    \vSigma_{t|t-1}^{-1} 
    + \vH_{t}^{\top} \bar{\vR}_t^{-1} \vH_t,
    \label{eq:WoLF-dim-cov-update}\\
\vmu_{t} &=
    \vmu_{t|t-1} 
    + \vK_{t} 
    (\vy_{t}-\hat{\vy}_t),
    \label{eq:WoLF-dim-mean-update}\\
\vK_{t} &=
    \vSigma_{t} \vH_{t}^{\top} \bar{\vR}_{t}^{-1}.
    \label{eq:WoLF-dim-gain}
\end{align}
This generalises the WoLF update with scalar weight given by Proposition \ref{prop:weighted-kf}.

Finally, when $w_{t,j}=0$ for one or more observation dimensions,
\eqref{eq:weighted-independent-loglikelihood} and \eqref{eq:dim-weighted-gaussian-likelihood}
define improper observation distributions
because the marginals on these dimensions are uniform over $\real$ and the precision $\bar{\vR}_t^{-1}$ is singular.
Nevertheless, the joint marginal for the dimensions with positive weights, $\vy_{t,J_t}$ where $J_t=\{j:w_{t,j}>0\}$, is a proper distribution. This is all that is needed because the observations $y_{t,j}$ with $w_{t,j}=0$ are ignored in the update.
This can be seen in \eqref{eq:WoLF-dim-cov-update},
\eqref{eq:WoLF-dim-mean-update}, and \eqref{eq:WoLF-dim-gain}, where
as a consequence of \eqref{eq:dim-weighted-Rbar}, the
error terms and the Jacobian are both zeroed out on the zero-weighted dimensions.

\subsection{Weighted ensembles}
\label{sec:weighted-ensemble-kalman-filter}
In this section, we introduce the weighted EnKF first mentioned in Section \ref{subsec:wlf-nonlinear-extensions}.
We begin with a discussion of the EnKF originally referenced in Section \ref{subsec:enkf}.

\paragraph{The ensemble Kalman filter:}
The ensemble Kalman filter (EnKF) \citep{roth2017ensemble} uses an ensemble of $N\in\mathbb{N}$ particles
$\{\hat{\vtheta}_{t\vert t-1}^{\left(i\right)}\}_{i=1}^N$.
For each $i=1, \ldots, N$, the update step samples predictions $\hat{\vy}_{t\vert t-1}^{\left(i\right)}$
according to
\begin{equation}\label{eq:EnKF-sampling}
   \hat{\vy}_{t\vert t-1}^{\left(i\right)} \sim \gauss\left(h_t\left(\hat{\vtheta}_{t|t-1}^{(i)}\right),\vR_t\right).
\end{equation}
and then updates each particle according to
\begin{equation}\label{eq:EnKF-update}
\hat{\vtheta}_{t}^{\left(i\right)}=
\hat{\vtheta}_{t\vert t-1}^{\left(i\right)}+\bar{\vK}_{t}\left(\vy_{t}-\hat{\vy}_{t\vert t-1}^{(i)}\right),
\end{equation}

The EnKF gain $\bar{\vK}_{t}$ is 
\begin{equation}\label{eq:enkf-gain-matrix}
\bar{\vK}_{t}={\rm cov}_{i}\left[\hat{\vtheta}_{t\vert t-1}^{\left(i\right)},\hat{\vy}_{t\vert t-1}^{\left(i\right)}\right]\left(\varQ{\hat{\vy}_{t\vert t-1}^{\left(i\right)}}{i}\right)^{-1}.
\end{equation}
We write $\expectQ{\cdot}{i},{\rm cov}_{i}\left[\cdot,\cdot\right],\varQ{\cdot}{i}$
to refer to the distribution over particles as indexed by $i$.

The EnKF converges to standard KF as $N\to\infty$ when the likelihood is linear-Gaussian and the prior is also Gaussian,
\begin{align}
\label{eq:EnKF-linear-Gaussian-assumption}
p\left(\vy_{t}\vert \vtheta_{t}\right) & =\gauss\left(\vy_{t}\vert\vH_{t}\vtheta_{t},\vR_{t}\right),\\
\hat{\vtheta}_{t\vert t-1}^{\left(i\right)} & \sim\gauss\left(\vmu_{t\vert t-1},\vSigma_{t\vert t-1}\right).
\end{align}
Under these conditions the EnKF gain $\bar{\vK}_{t}$ matches the KF gain $\vK_{t}$, because
\begin{align}
{\rm cov}_{i}\left[\hat{\vtheta}_{t\vert t-1}^{\left(i\right)},\hat{\vy}_{t\vert t-1}^{\left(i\right)}\right] & =\vSigma_{t\vert t-1}\vH_{t}^{\top}\label{eq:EnKF-cross_cov-lin_gauss_limit},\\
\varQ{\hat{\vy}_{t\vert t-1}^{\left(i\right)}}{i} & =\vH_{t}\vSigma_{t\vert t-1}\vH_{t}^{\top}+\vR_{t}
\label{eq:EnKF-var-lin_gauss_limit},\\
\bar{\vK}_{t} & =\vSigma_{t\vert t-1}\vH_{t}^{\top}\left(\vH_{t}\vSigma_{t\vert t-1}\vH_{t}^{\top}+\vR_{t}\right)^{-1}\label{eq:EnKF-gain-lin_Gauss_limit}\\
 & =\vK_{t},
 \label{eq:EnKF-gain-lin_Gauss_limit-b}
\end{align}
and therefore the statistics of the posterior ensemble also match
the KF update
\begin{align}
\vmu_{t} =\expectQ{\hat{\vtheta}_{t}^{\left(i\right)}}{i}
 & =\vmu_{t\vert t-1}+\vK_{t}\left(\vy_{t}-\vH_{t}\vmu_{t\vert t-1}\right),
 \label{eq:EnKF-mean-update-lin_gauss_limit}\\
\vSigma_{t} =\varQ{\hat{\vtheta}_{t}^{\left(i\right)}}{i}
 & =\vSigma_{t\vert t-1}-\vK_{t}\vH_{t}\vSigma_{t\vert t-1}.\label{eq:EnKF-cov-update-lin_gauss_limit}
\end{align}

\paragraph{The weighted-likelihood EnKF:}
We propose a weighted version of EnKF based on our WoLF
with dimension-specific weights and Gaussian likelihood,
where \eqref{eq:EnKF-sampling} is modified to sample particle predictions from the weighted likelihood in \eqref{eq:dim-weighted-gaussian-likelihood}:
\begin{align}
    \hat{\vy}_{t|t-1}^{(i)}
    \sim \gauss\left(h_t\left(\hat{\vtheta}_{t|t-1}^{(i)}\right),\bar{\vR}_t\right).
    \label{eq:wEnKF-sampling-Rtilde}
\end{align}

The weighted EnKF converges to the WoLF as $N\to\infty$, by the same argument for the vanilla EnKF and EKF in
\eqref{eq:EnKF-cross_cov-lin_gauss_limit} through \eqref{eq:EnKF-cov-update-lin_gauss_limit}.
Specifically, under this limit the sampling scheme in \eqref{eq:wEnKF-sampling-Rtilde} and the update in \eqref{eq:EnKF-update} yield
\begin{align}
    \bar{\vK}_t &= \vSigma_{t\vert t-1}\vH_{t}^{\top}\left(\vH_{t}\vSigma_{t\vert t-1}\vH_{t}^{\top}+\bar{\vR}_{t}\right)^{-1},\\
    \vmu_{t}  &= \vmu_{t\vert t-1}+ \bar{\vK}_{t}\left(\vy_{t}-\vH_{t}\vmu_{t\vert t-1}\right),\\
    \vSigma_{t} &= \vSigma_{t\vert t-1}-\bar{\vK}_{t}\vH_{t}\vSigma_{t\vert t-1},
    \label{eq:wEnKF-cov-update-lin_gauss_limit}
\end{align}
which matches the WoLF update in \eqref{eq:WoLF-dim-cov-update}, \eqref{eq:WoLF-dim-mean-update}, and \eqref{eq:WoLF-dim-gain}.

For the case of scalar weights with $w_t^2 = W^2(\vy_t,\hat{\vy}_t) \in \{0,1\}$ as in \eqref{eq:thresholded-mahalanobis-weight}, the weighted EnKF can be implemented by unweighted sampling as in \eqref{eq:EnKF-sampling} followed by a weighted update that replaces \eqref{eq:EnKF-update} with
\begin{equation}\label{eq:weighted-enkf-1d}
    \hat{\vtheta}^{(i)}_{t} = \hat{\vtheta}^{(i)}_{t|t-1} + w_t^2 \bar{\vK}_t \left(\vy_t - \hat{\vy}_{t|t-1}^{(i)}\right).
\end{equation}
That is, the update is simply skipped when $w_t^2=0$, in agreement with Proposition \ref{prop:weighted-kf} and \cref{algo:wlf-step}.

For the case of vector weights $\vw_t$ as in \cref{sec:dimension-specific-weighting},
when $\vw_t\in\{0,1\}^d$ the weighted EnKF can be implemented by
a generalisation of \eqref{eq:weighted-enkf-1d} that ignores only those observation components $\vy_{t,j}$
for which $\vw_{t,j}=0$.
Specifically, we can use the vanilla EnKF sampling in \eqref{eq:EnKF-sampling} but sample only the positively weighted dimensions, i.e.\ $\hat{\vy}^{(i)}_{t|t-1,J_t}$ where $J_t=\{j:\vw_{t,j}>1\}$. 
Thus $\hat{\vy}^{(i)}_{t|t-1}$ has size $|J_t|$ and \eqref{eq:enkf-gain-matrix} yields $\bar{\vK}_t$ with size $\nparams\times|J_t|$, and the particles can be updated according to
\begin{equation}
    \label{eq:weighted-enkf-dimensional-update}
    \hat{\vtheta}^{(i)}_{t} = \hat{\vtheta}^{(i)}_{t|t-1} + \bar{\vK}_t \left(\vy_{t,J_t} - \hat{\vy}_{t|t-1,J_t}^{(i)}\right).
\end{equation}
This agrees with the vanilla EnKF update except that zero-weighted observation dimensions are ignored. It can be shown to converge to the dimensionally weighted WoLF described in \cref{sec:dimension-specific-weighting} as $N\to\infty$.%
\footnote{The general case of $\vw_t\in[0,1]^d$ can be implemented by combining this scheme with weighted sampling of positively weighted dimensions, i.e. using \eqref{eq:wEnKF-sampling-Rtilde} to sample $\hat{\vy}^{(i)}_{t|t-1,J_t}$.}

For the experiments reported here (\cref{experiment:lorenz96}), we take the shortcut of sampling predictions for all dimensions and generalising \eqref{eq:weighted-enkf-1d} to vector weights:
\begin{equation}
\hat{\vtheta}_{t}^{\left(i\right)}=
\hat{\vtheta}_{t\vert t-1}^{\left(i\right)} +
\bar{\vK}_{t}\Diag(\vw_t) \left(\vy_{t}-\hat{\vy}_{t\vert t-1}^{(i)}\right).
\label{eq:AP-EnKF-update}
\end{equation}
This approximates the method of sampling only $\hat{\vy}^{(i)}_{t|t-1,J_t}$
and updating with \eqref{eq:weighted-enkf-dimensional-update}.
This is more efficient in our Jax \citep{jax2018github} implementation
because all arrays are of constant size.
We generalise \eqref{eq:thresholded-mahalanobis-weight} to define the weight vector as
\begin{equation}
    \vw_{t,j} =
    \begin{cases}
        1 & \frac{1}{N}\sum_{i=1}^N\left(\vy_{t, j}-\hat{\vy}_{t\vert t-1, j}^{(i)}\right)^2 \leq c,\\
        0 & \text{otherwise},
    \end{cases}
\end{equation}
We call this the {\bf average-particle EnKF (AP-EnKF)}.

Alternatively, we propose the update equation
\begin{equation}
\hat{\vtheta}_{t}^{\left(i\right)}=
\hat{\vtheta}_{t\vert t-1}^{\left(i\right)} +
\bar{\vK}_{t}\,\Diag\left(\vw_t^{(i)}\right) \left(\vy_{t}-\hat{\vy}_{t\vert t-1}^{(i)}\right),
\end{equation}
with
\begin{equation}
    \vw_{t,j}^{(i)} =
    \begin{cases}
        1 & \left(\vy_{t, j}-\hat{\vy}_{t\vert t-1, j}^{(i)}\right)^2 \leq c,\\
        0 & \text{otherwise},
    \end{cases}
\end{equation}
which we call the {\bf per-particle EnKF (PP-EnKF)}.

\paragraph{The Huberised EnKF:}
The update equations for the PP-EnKF and AP-EnKF are related to the \textit{Huberised} EnKF (H-EnKF)
algorithm of \citet{roh2013}
that update particles according to
\begin{equation}
\hat{\vtheta}_{t}^{\left(i\right)}=
\hat{\vtheta}_{t\vert t-1}^{\left(i\right)} +
\bar{\vK}_{t}\,\rho\left(\vy_{t}-\hat{\vy}_{t\vert t-1}^{(i)}\right),
\end{equation}
with 
\begin{equation}
    \rho(z) = 
\begin{cases}
    c & \text{if } z > c, \\
    -c & \text{if } z < -c, \\
    z & \text{otherwise},
\end{cases}
\end{equation}
$c > 0$ a given threshold, and the function $\rho$ is applied element-wise.

\section{Additional numerical experriments}
\label{sec:additional-experiments}

This section provides additional numerical experiments.
In particular, in Appendix \ref{subsec:pif-plot-explain}, we discuss Figure \ref{fig:pif-2d-tracking} in more detail;
in Appendix \ref{sec:ablation},
we conduct robustness checks for the numerical experiments presented in Section \ref{sec:experiments}; and
in Appendix \ref{experiment:training-neural-network}, we provide an additional experiment
for online learning in a non-stationary environment with outlier measurements.

\subsection{Description of the PIF}
\label{subsec:pif-plot-explain}

In this section, we discuss 
the 2d tracking problem of Section \ref{experiment:2d-tracking} in more detail.
We generate $t=20$ steps from \eqref{eq:noisy-2d-ssm} and the last measurement $\vy_t$
is replaced with $\vy_t^{c} = \vy_t + \vepsilon$,
where $\vepsilon \in [-5, 5]\times[-5, 5]$ is the outlier noise.
Figure \ref{fig:2d-tracking-grid-weighting-term} (left) shows the weighting term for the WoLF methods
as a function of the outlier noise $\vepsilon$.
Alternatively, Figure \ref{fig:2d-tracking-grid-weighting-term} (right) shows the Mahalanobis distance
between the prior predictive $\hat{\vy}_t$ and the outlier measurement $\vy_t^c$.
\begin{figure}[htb]
    \centering
    \includegraphics[width=0.45\linewidth]{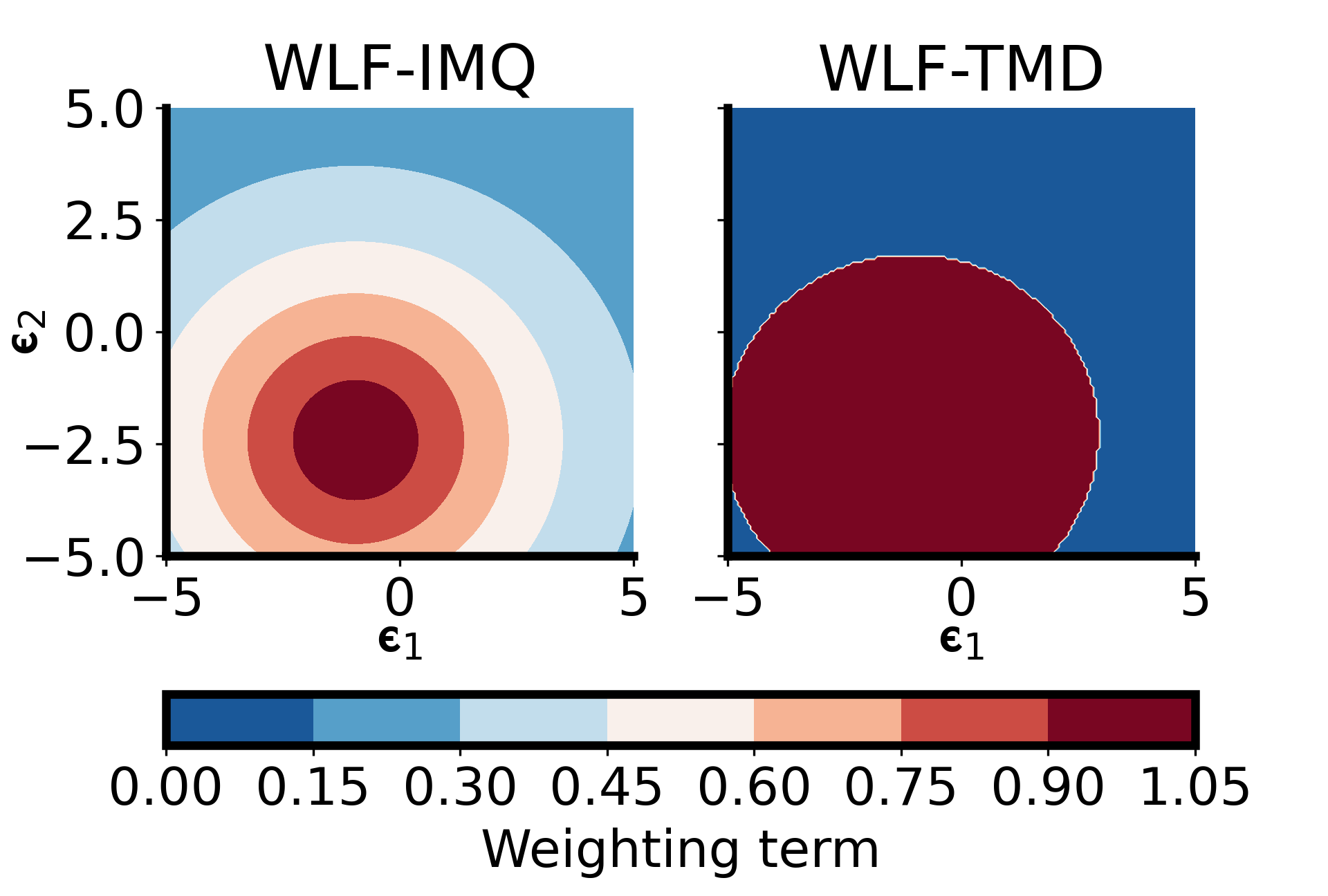}
    \removewhitespace[-4mm]
    \caption{
        Weighting term $W_t(\vy_{1:t})$ given an the outlier measurement $\vy_t^c$. The $x$-axis and $y$-axis describe the errors $\vepsilon \in [-5, 5]\times[-5, 5]$. 
    }
    \label{fig:2d-tracking-grid-weighting-term}
\end{figure}
We observe that the midpoint of the contours for all methods is not centred at $\vepsilon = (0, 0)$.
This is because the weighting term is a function of the prior predictive and the measurement
at time $t$. Hence, for the IMQ, the weighting is $1$ only if the prior predictive
equals the measurement, i.e., $\hat{\vy}_t = \vH_t\vmu_{t|t-1}$.
This result explains the distored PIF observed in Figure \ref{fig:pif-2d-tracking}.

Figure \ref{fig:2d-tracking-grid-pif-md-extended-domain} shows the PIF for the WoLF-IMQ and the WoLF-TMD
first shown in Figure \ref{fig:pif-2d-tracking}.
\begin{figure}[htb]
    \centering
    \includegraphics[width=0.5\linewidth]{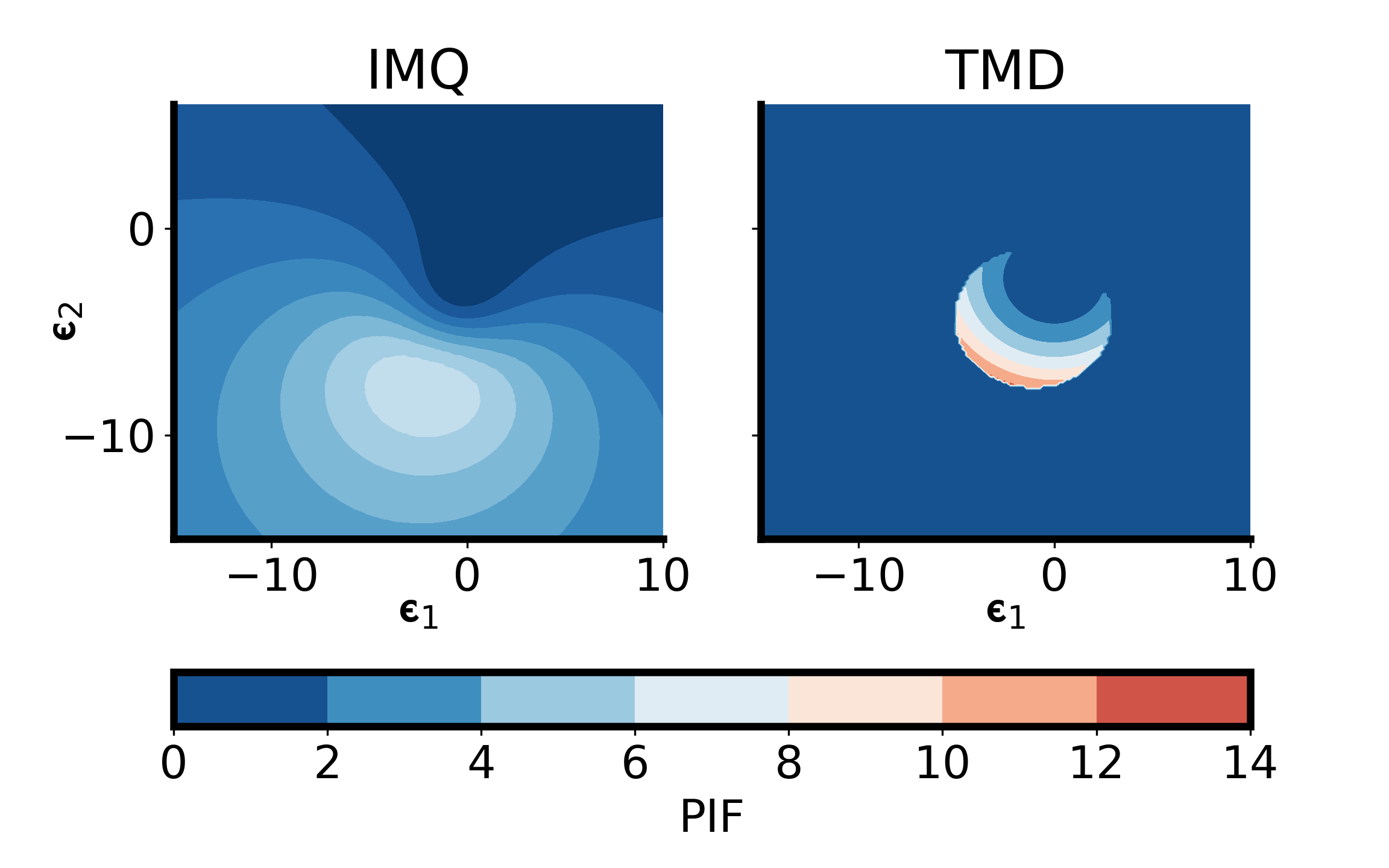}
    \vspace{-5mm}
    \caption{
        PIF of Figure \ref{fig:pif-2d-tracking} for extended domain of the errors $\vepsilon \in [-15, 10]\times[-15, 10]$. 
    }
    \label{fig:2d-tracking-grid-pif-md-extended-domain}
\end{figure}
We observe that the PIF for
the WoLF-TMD has a hard cutoff that strongly bounds the PIF values at the expense of higher PIF around a region near the 
cutoff boundary.
In contrast, the WoLF-IMQ bounds the PIF values more softly and has lower maximum PIF than the WoLF-TMD.

\subsection{Robustness}
\label{sec:ablation}

\subsubsection{2d tracking}
\label{subsec:2d-tracking-further-results}
In this section, we present further results from experiment \ref{experiment:2d-tracking}.
Figure \ref{fig:2d-tracking-all} shows the distribution over errors over all 4 state components for both outlier variants of the 2d tracking problem.
For the mixture variant, we see elongated tails in the error distribution
for the  \mWlfMd algorithm.
\begin{figure}[htb]
    \centering
    \includegraphics[width=0.45\linewidth]{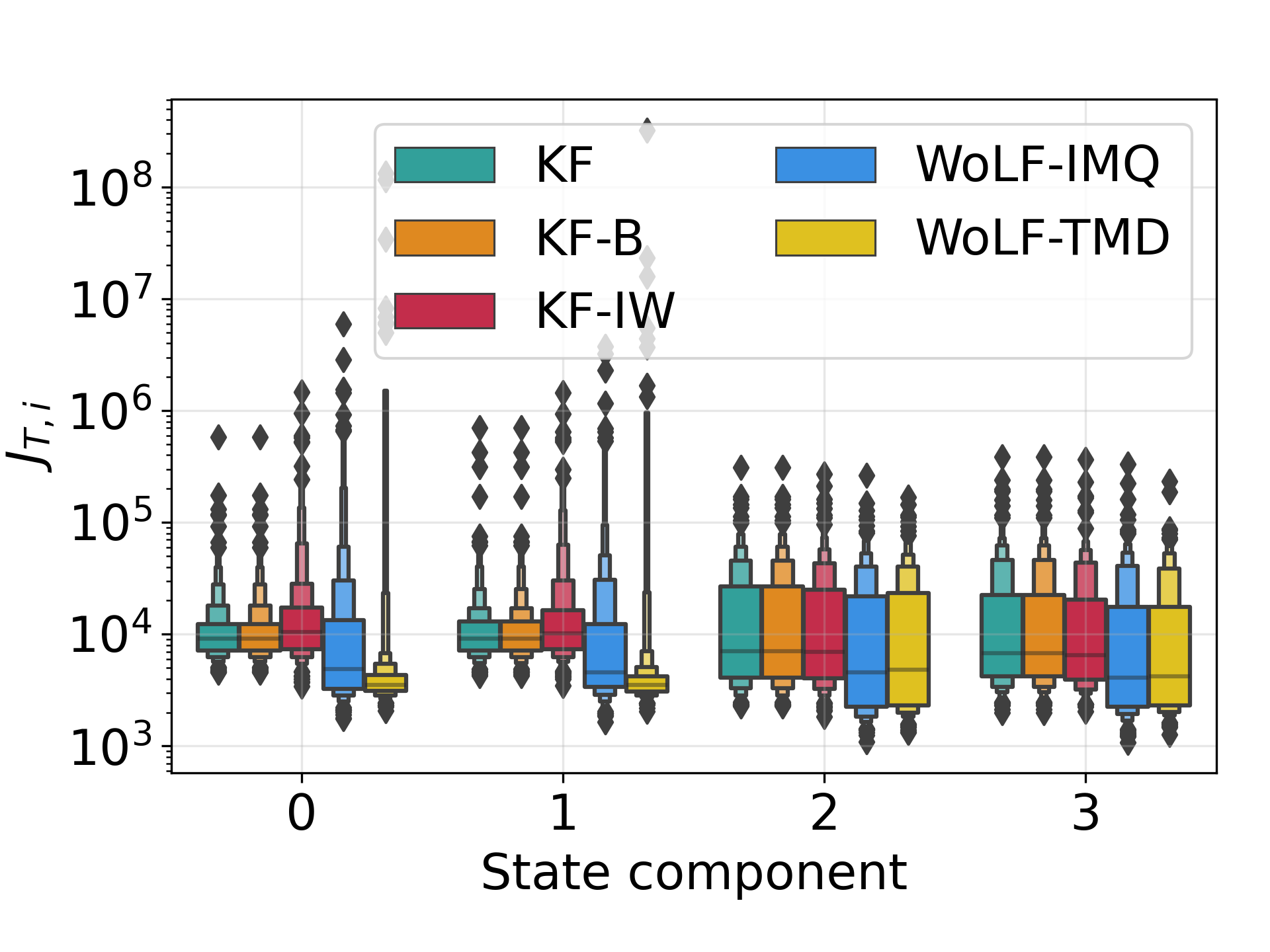}
    \includegraphics[width=0.45\linewidth]{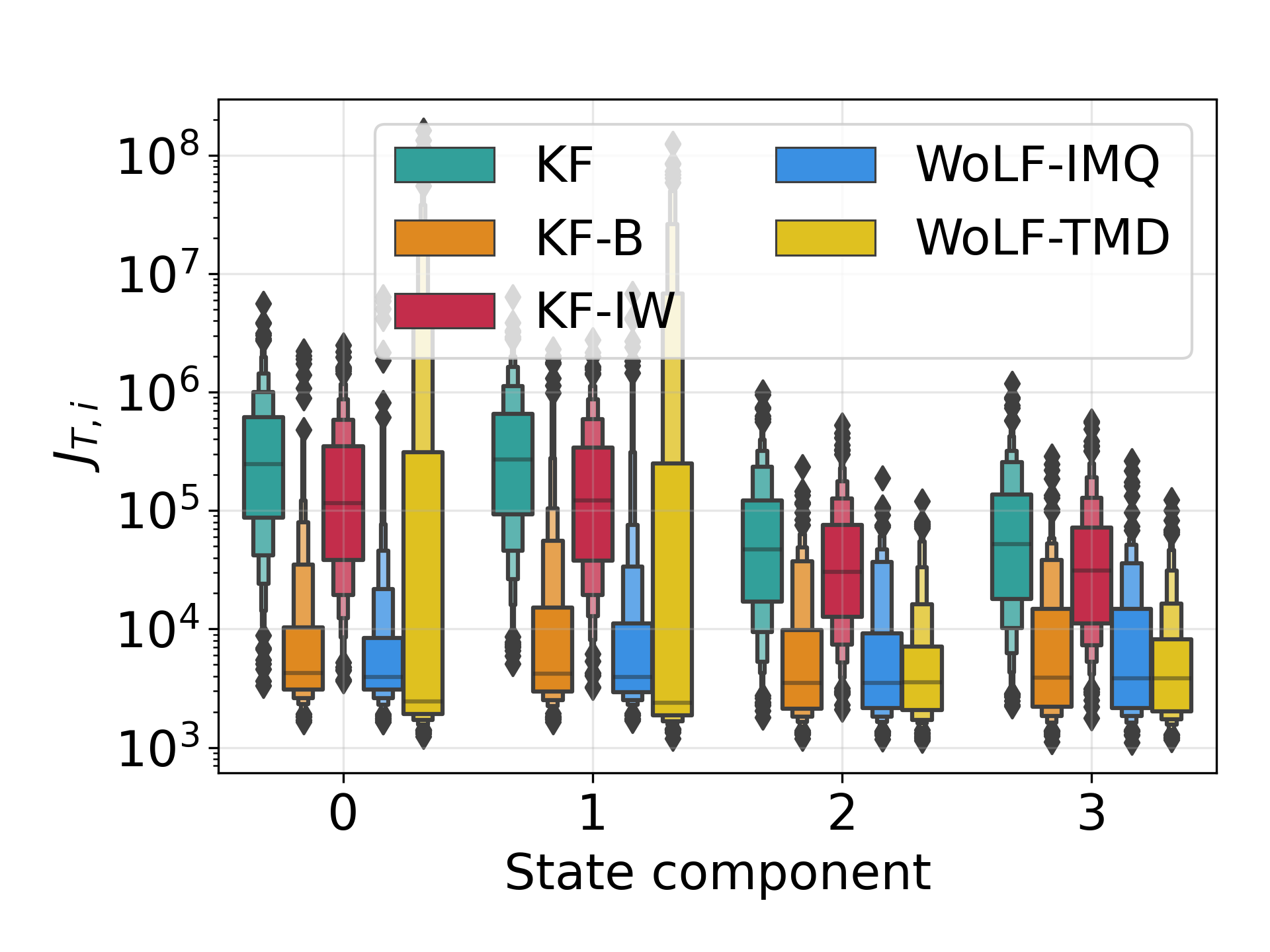}
    \removewhitespace[-6mm]
    \caption{
        Distribution of $J_{T,i}$ for all state components.
        The left panel is for the student variant \eqref{eq:noisy-2d-ssm-outlier-covariance} and
        the right panel is for the mixture variant \eqref{eq:noisy-2d-ssm-outlier-mean}.
    }
    \label{fig:2d-tracking-all}
\end{figure}
To see why, consider the trace of weights over time in
Figure \ref{fig:wlf-weight-both}.
We observe that both \mWlfImq and \mWlfMd set the weighting term close to (or equal) to zero in outlier events.
However, the \mWlfMd is more prone to have false positives 
(at a rate of about 7\% in this example),
in which it fails to update the posterior state.
This explains the elongated tails.
\begin{figure}[tbh]
    \centering
    \includegraphics[width=0.75\linewidth]{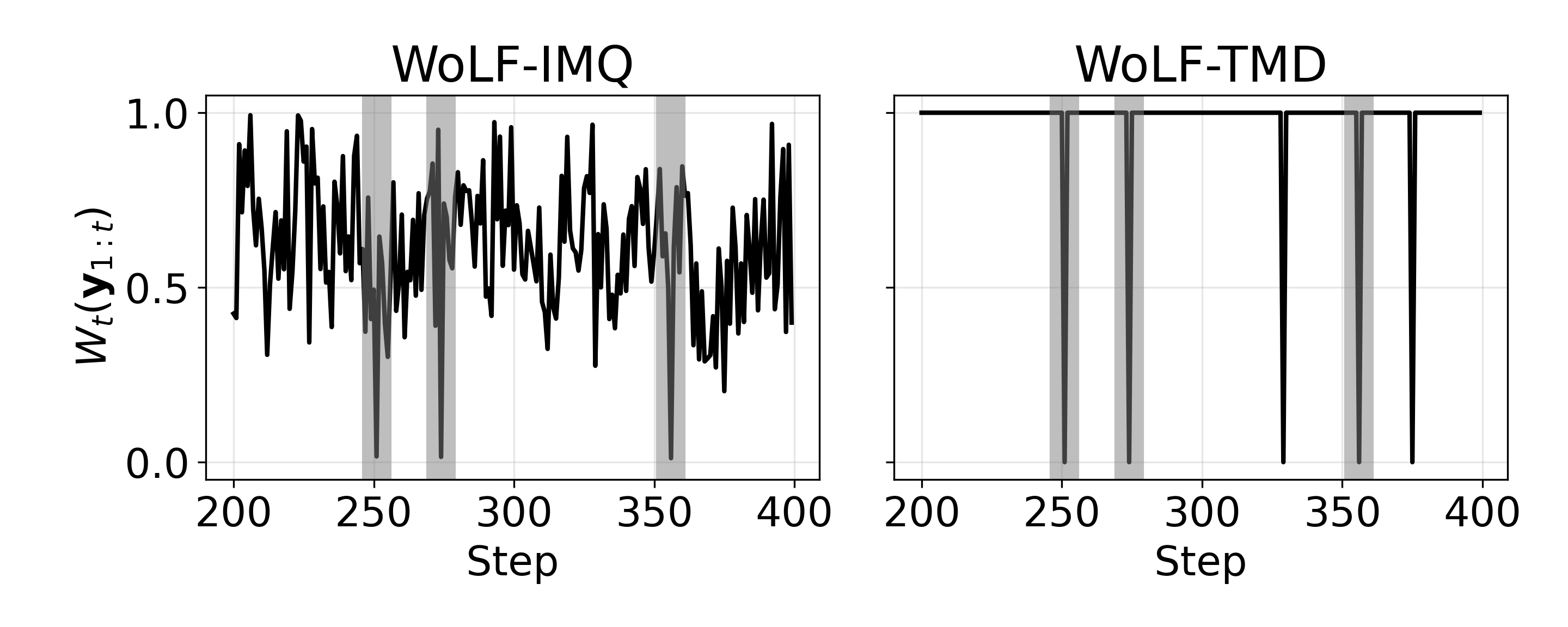}
    \removewhitespace[-6mm]
    \caption{Trace of the weighting term for WoLF-IMQ and WoLF-MD.}
    \label{fig:wlf-weight-both}
\end{figure}

\paragraph{Right panel of Figure \ref{fig:linear-ssm-sample-runs}}
\begin{figure}[htb]
    \centering
    \includegraphics[width=0.45\linewidth]{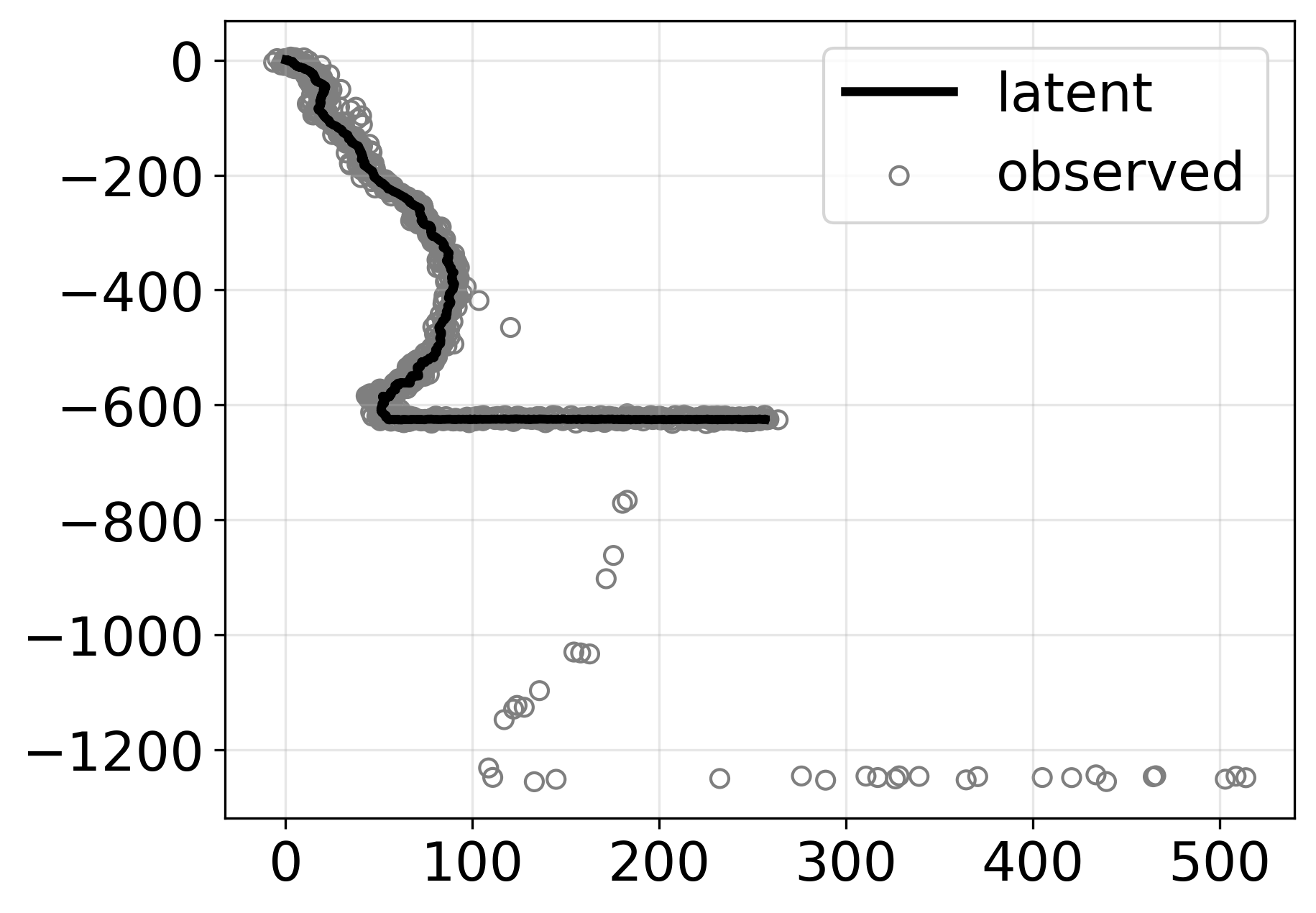}
    \caption{Sample of the top-left panel in Figure \ref{fig:linear-ssm-sample-runs}}
    \label{fig:mixture-model-uncropped}
\end{figure}

\paragraph{2d tracking: hyperparameter choice}
Here we show the performance of each method for the experiment in Section \ref{experiment:2d-tracking}
as we vary the hyperparameters. We do this for both, the Mixture case and the Student-t case.

Figure \ref{fig:hyperparam-stress}
shows the RMSE for estimating the first state component as we vary the $c$ hyperparameter for the IMQ and the TDM weighting functions;
for \mAgamenoni, we use two inner iterations and we vary the $h$ hyperparameter;
for \mWang, we fix $\beta = 1.0$, use four inner iterations and vary the value of $\alpha$.
The definition for the RMSE $J_{T,0}$ is as in Section \ref{experiment:2d-tracking}.

\begin{figure}[htb]
    \centering
    \includegraphics[width=0.45\linewidth]{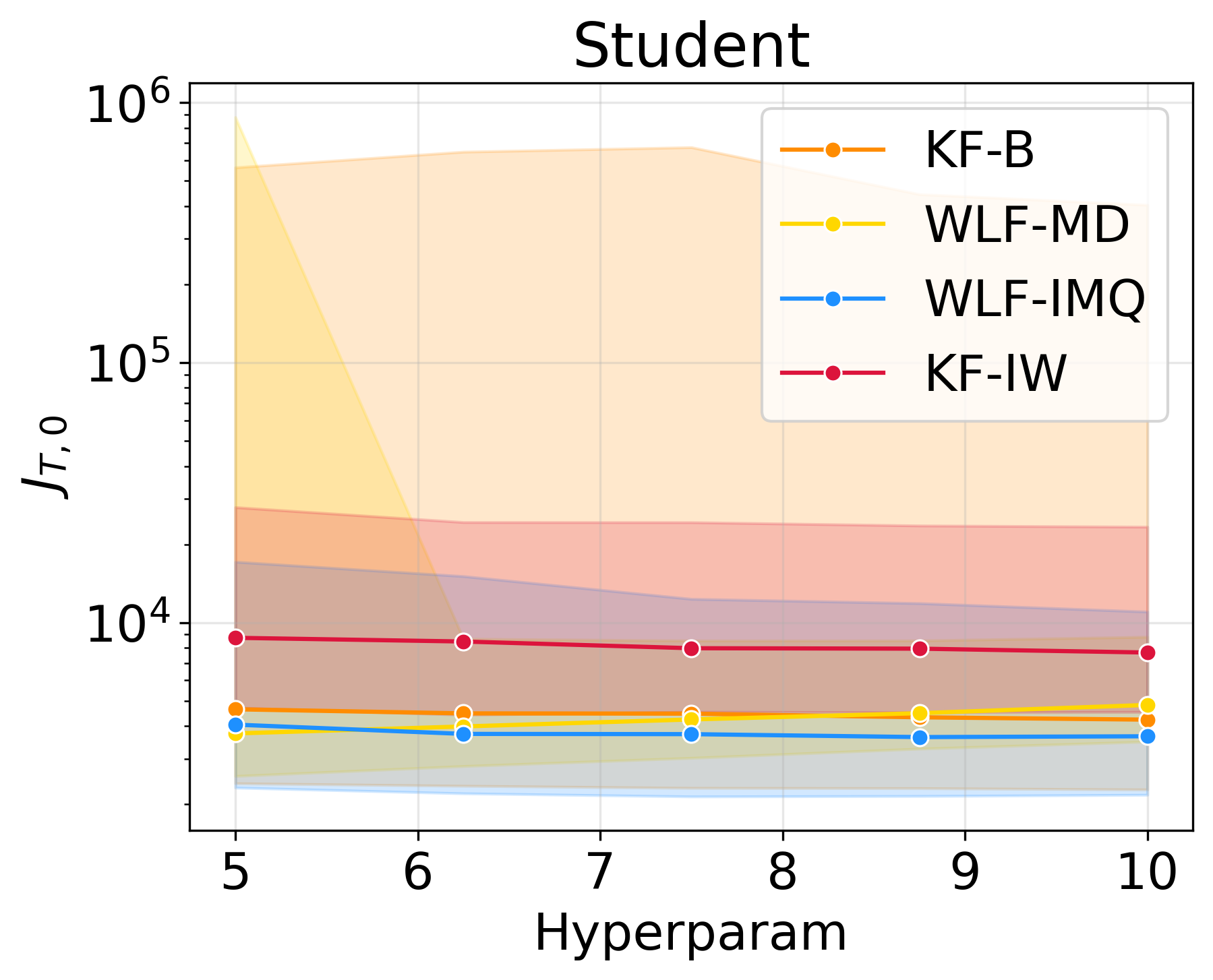}
    \includegraphics[width=0.45\linewidth]{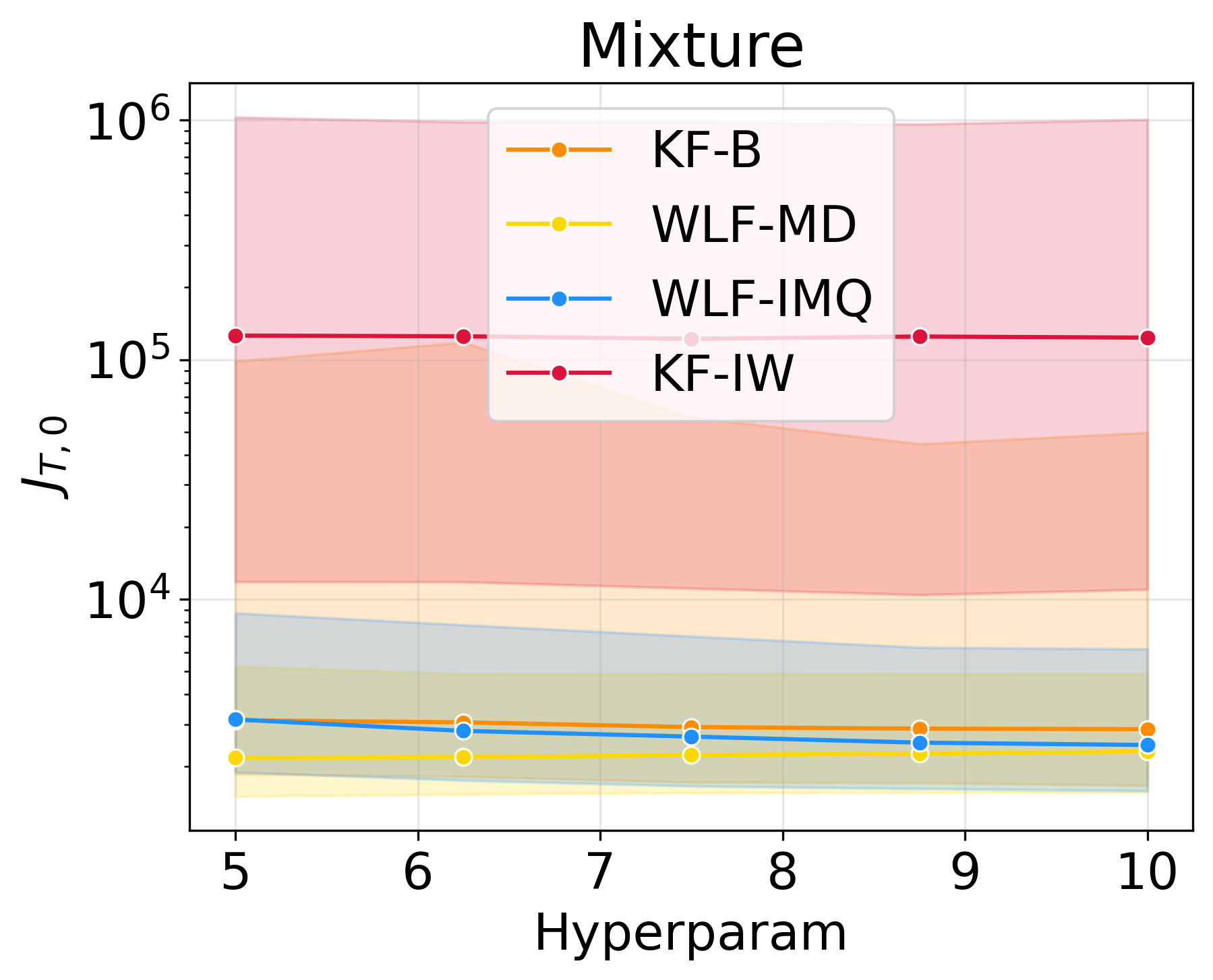}
    \caption{$J_{t,0}$ over 100 runs for the methods as a function of their hyperparameter.}
    \label{fig:hyperparam-stress}
\end{figure}

\paragraph{Empirical comparison to \cite{boustati2020generalised}}
In this section, we compare the particle-filter-based method of \cite{boustati2020generalised},
which we denote the \texttt{RBPF}.
Figure \ref{fig:2d-tracking-with-rbpf} extends Figure \ref{fig:2d-ssm-sample} to include the \texttt{RBPF}.
\begin{figure}
    \centering
    \includegraphics[width=0.45\linewidth]{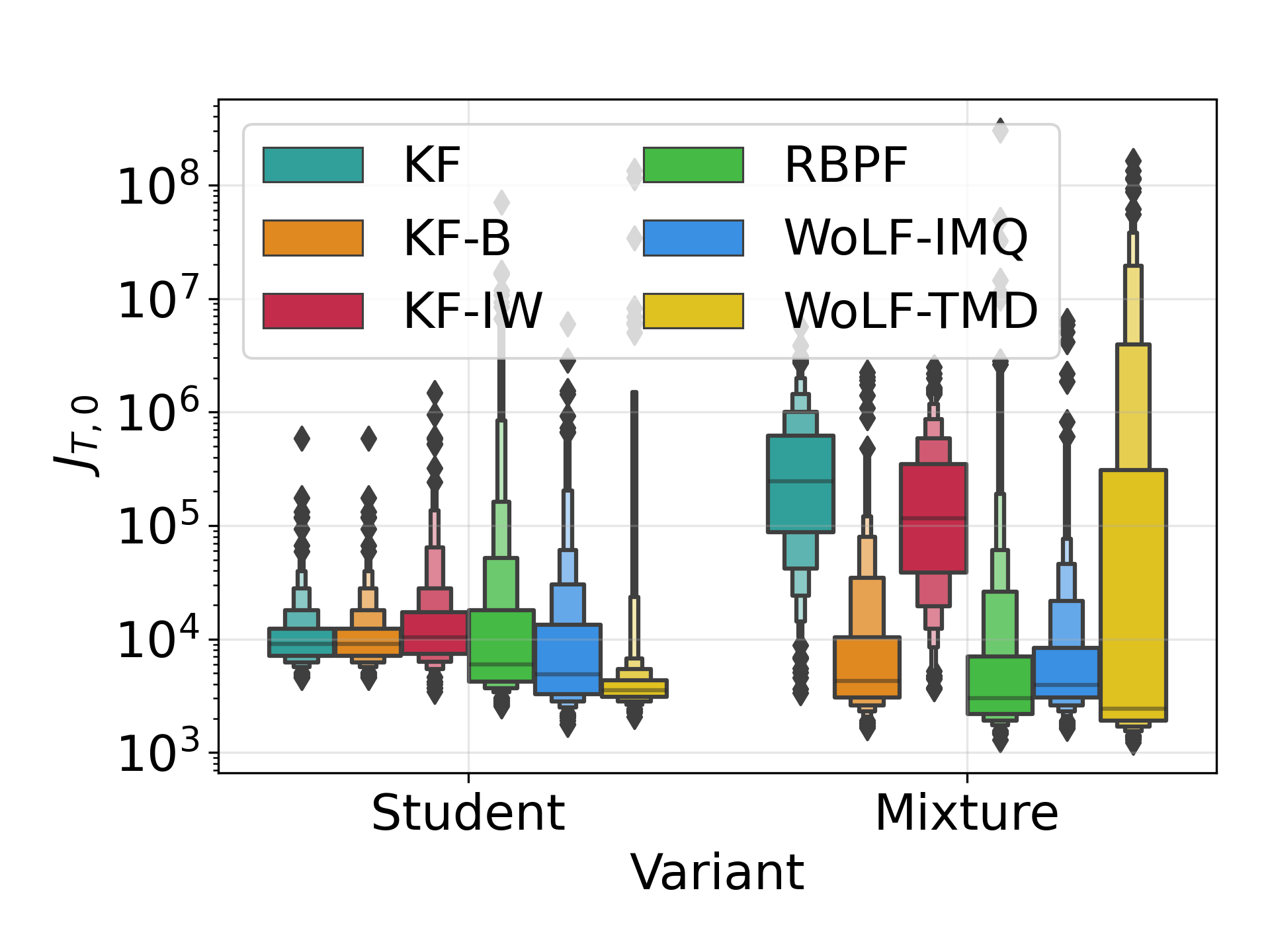}
    \includegraphics[width=0.45\linewidth]{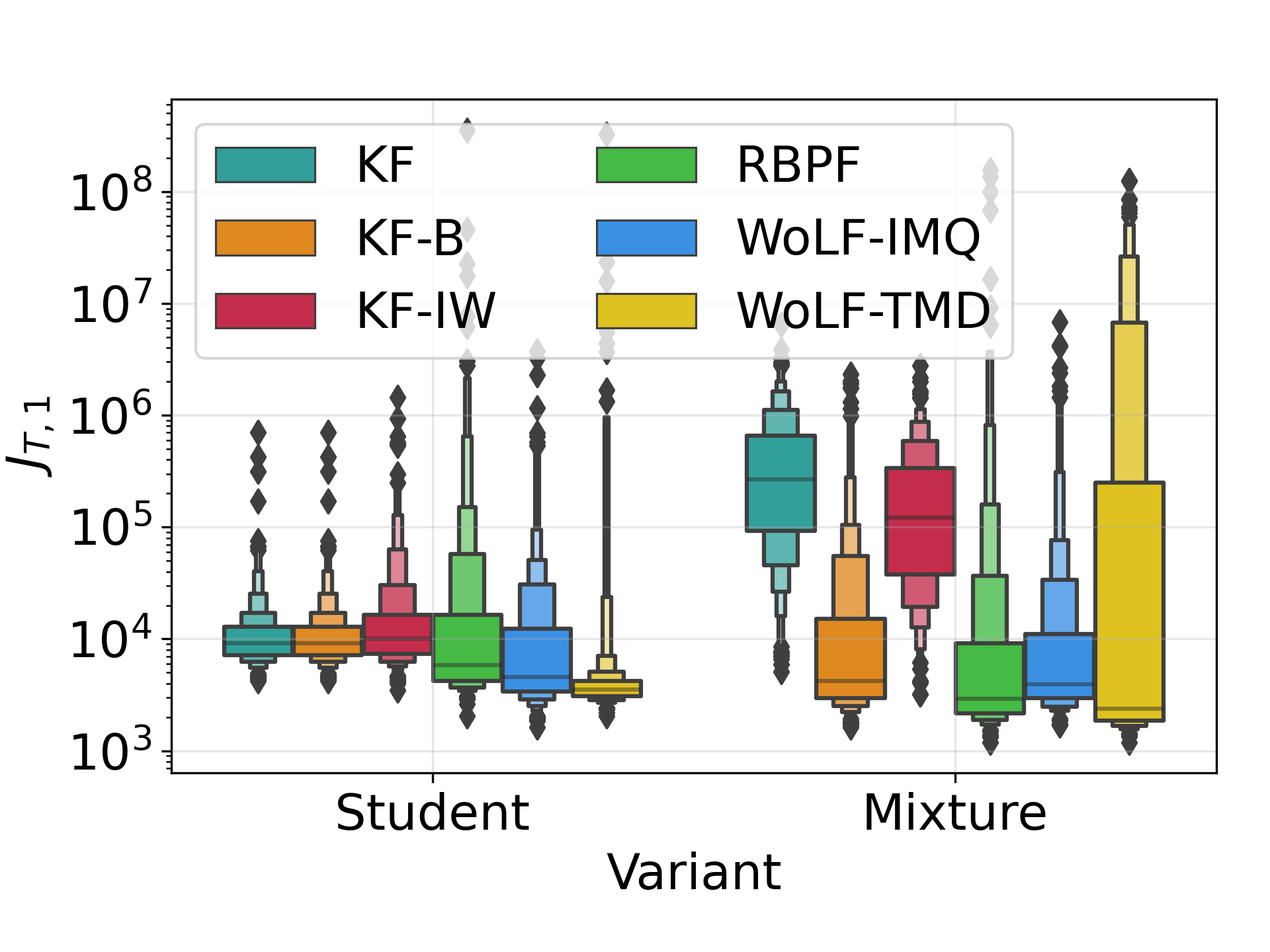}
    \caption{
        Distribution  (across 500 2d tracking trials)
        of RMSE for first component of the state vector, $J_{T,0}$.
        Left panel: Student observation model.
        Right panel: Mixture observation model.
        We extend Figure \ref{fig:2d-ssm-sample} to include the \texttt{RBPF}.
    }
    \label{fig:2d-tracking-with-rbpf}
\end{figure}
The \texttt{RBPF} performs comparably to our proposed method; however, it has much higher computational cost and does not have a closed-form solution.

\subsubsection{UCI datasets}
\label{subsec:uci-further-results}
In this section, we give more details on the MLP regression experiments in Section \ref{subsec:uci-corrupted}.
The size of the datasets and models which we use are summarised in  Table \ref{tab:uci-description}.

\begin{table}[htb]
\centering
\begin{tabular}{lrrr}
    \toprule
     & \#Examples $T$ & \#Features $\nout$ & \#Parameters $\nparams$ \\
    Dataset &  &  &  \\
    \midrule
    Boston & $ 506 $ & $ 14 $ & $ 321 $ \\
    Concrete & $ 1,030 $ & $ 9 $ & $ 221 $ \\
    Energy & $ 768 $ & $ 9 $ & $ 221 $ \\
    Kin8nm & $ 8,192 $ & $ 9 $ & $ 221 $ \\
    Naval & $ 11,934 $ & $ 18 $ & $ 401 $ \\
    Power & $ 9,568 $ & $ 5 $ & $ 141 $ \\
    Protein & $ 45,730 $ & $ 10 $ & $ 241 $ \\
    Wine & $ 1,599 $ & $ 12 $ & $ 281 $ \\
    Yacht & $ 308 $ & $ 7 $ & $ 181 $ \\
\bottomrule
\end{tabular}
\caption{Description of UCI datasets.
Number of parameters refers to the size of the one-layer
MLP.
}
\label{tab:uci-description}
\end{table}

\paragraph{Speed vs accuracy}
Figure \ref{fig:naval-time-error} (left) shows the time-step and RMedSe for the Kin8nm dataset over multiple trials.
Our  methods, the \mWlfImq and the \mWlfMd,
strike the best balance between RMedSE and running time among the competing methods.
The \mAgamenoniExtended and the \mWangExtended have comparable error rates to the WoLF methods,
but their running time is 6x and 12x slower than the \mWlfImq, respectively.
The \mkfExtended has comparable running time, but 3x higher average RMedSE than
the \mWlfImq.
Finally, the \ogd is 40\% faster, but has 2x higher average RMedSE than the \mWlfImq.

\paragraph{Sensitivity to outlier rate}
We evaluate the sensitivity of the methods to the choice of corrosion rate $p_\epsilon$ for the Kin8nm dataset.
Figure \ref{fig:naval-time-error} (right) shows the RMedSE after 100 trials of each method
as a function of $p_\epsilon$.
We set the hyperparameters using BO with $p_\epsilon = 0.1$
and evaluate the RMedSE with $p_\epsilon \in \{0.00, 0.05, \ldots, 0.45\}$.
We observe that all robust methods --- the \mWlfImq, the \mWlfMd, the \mWangExtended, and the \mAgamenoniExtended --- 
have similar RMEdSE for any choice of $p_\epsilon$ and have similar RMedSE rate of increase.
Conversely, the \mkfExtended and the \ogd are much less stable and their RMedSE error rate increases at a faster pace.

\begin{figure}[htb]
    \centering
    \includegraphics[width=0.45\linewidth]{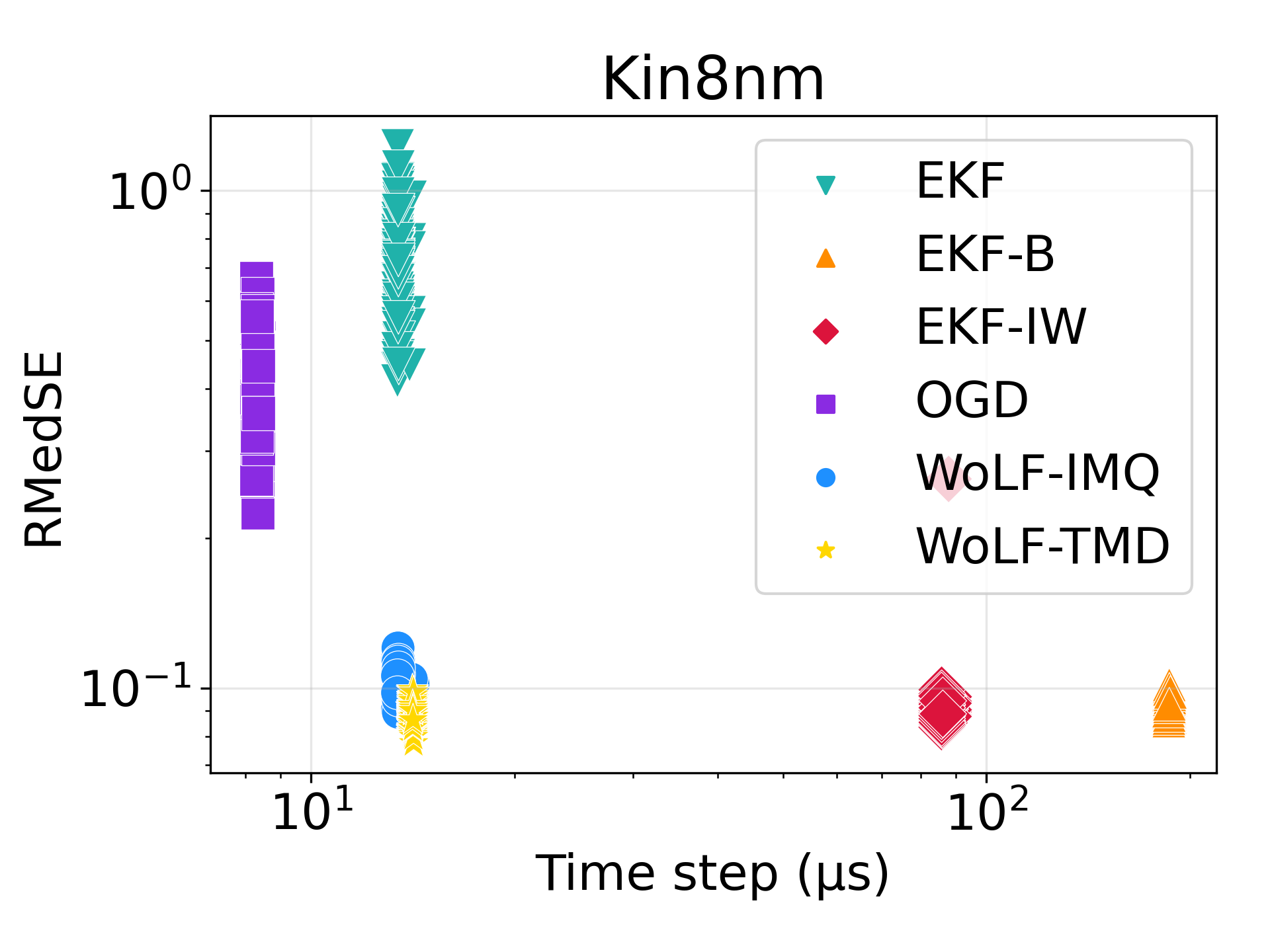}
    \includegraphics[width=0.45\linewidth]{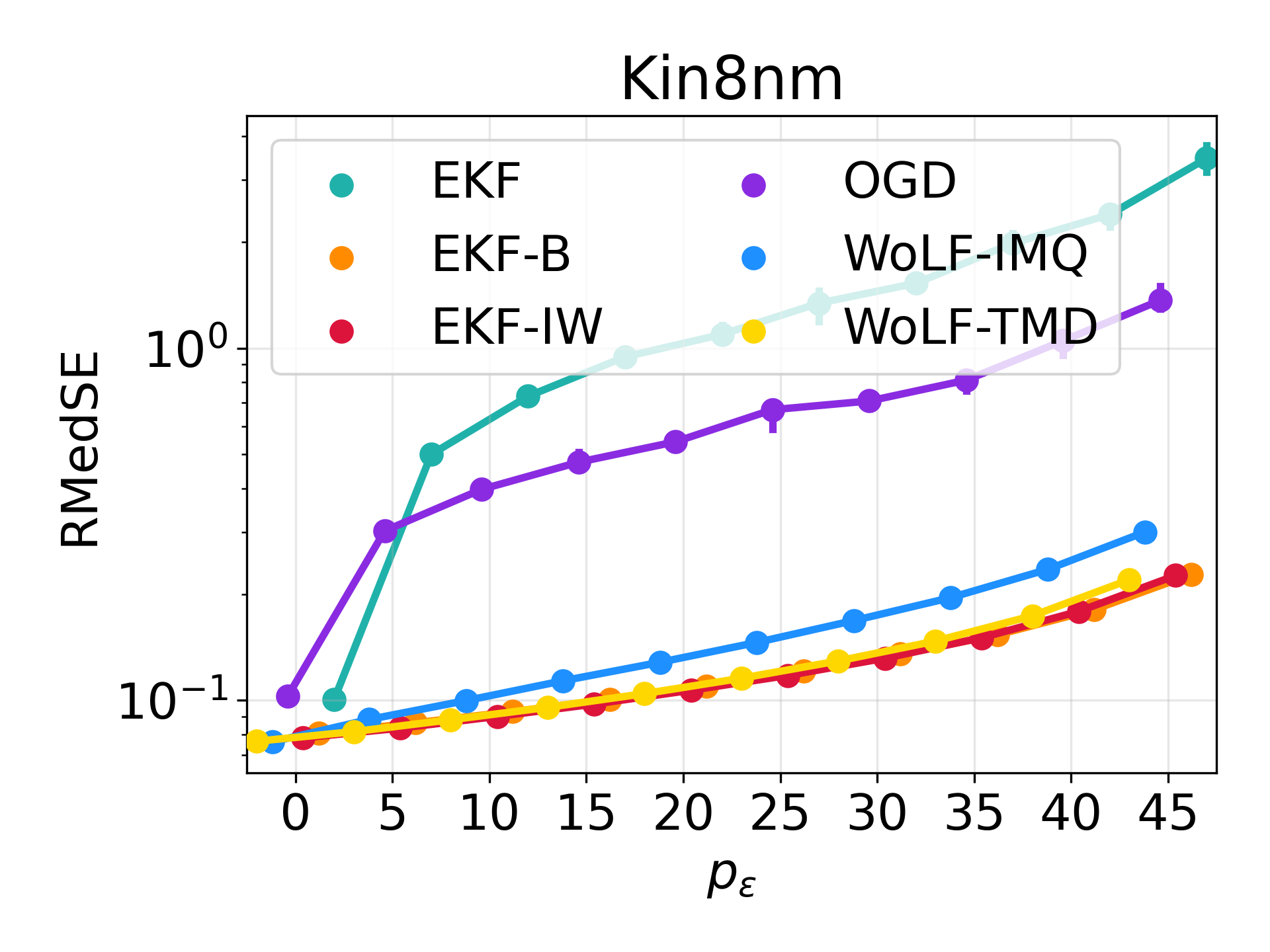}
    \removewhitespace[-6mm]
    \caption{
    \textbf{Left panel} shows the
    RMedSE (root median squared error) vs running time (per observation) to fit a neural network
    to the Kin8nm  UCI regression dataset.
    Each point corresponds to a different trial.
    \textbf{Right panel} shows  the 
     RMedSE (root median squared error) vs 
    percentage probability of outlier, $p_{\epsilon}$
    when fitting a neural network
    to the Kin8nm  UCI regression dataset.
    (We show the bootstrp average over 100 trials and 95\% confidence interval)
    Crucially, the same hyper-parameters are used
    for all experiments (and are estimated under 
    $p_\epsilon = 0.1$).
    }
    \label{fig:naval-time-error}
\end{figure}

\subsubsection{Lorenz96 model}
\label{subsec:lorenz96-further-results}

In this subsection, we evaluate the EnKF methods
in the Lorenz96 model presented in subsection \ref{experiment:lorenz96} 
when we only have $N=20$ particles but $d=100$ states.
In this scenario, the EnKF is usually modified to incorporate
the \textit{covariance inflation} proposed in \citet{anderson1999enkf-covariance-inflation}.
We present the results for a single run of the modified methods in the left panel of Figure
\ref{fig:lorenz96-methods-comparison-inflation-covariance}.
In the right panel of Figure \ref{fig:lorenz96-methods-comparison-inflation-covariance}, we perform a
sensitivity analysis to the choice of hyperparameter $c$.
\begin{figure}[htb]
    \centering
    \includegraphics[width=0.45\linewidth]{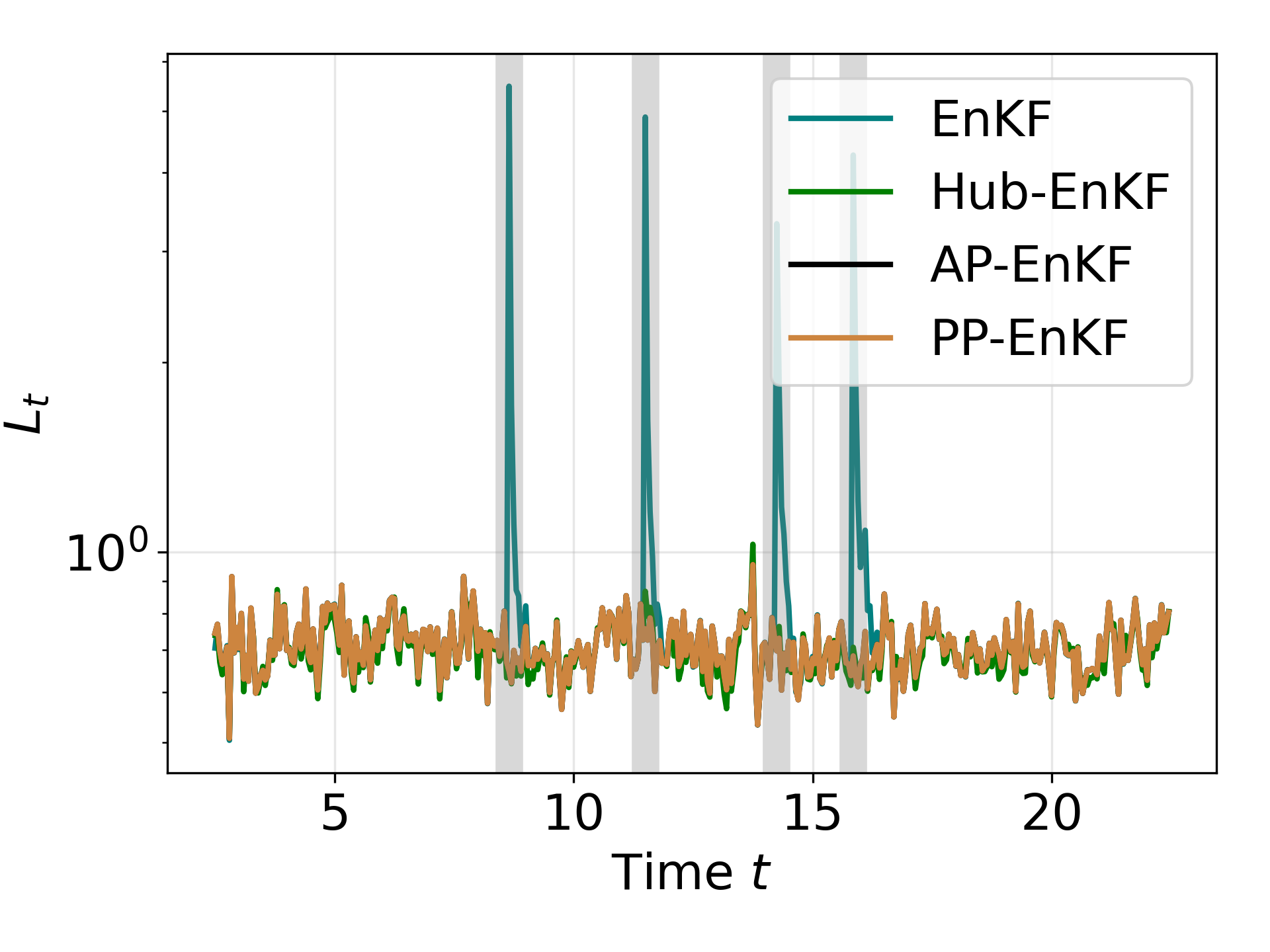}
    \includegraphics[width=0.45\linewidth]{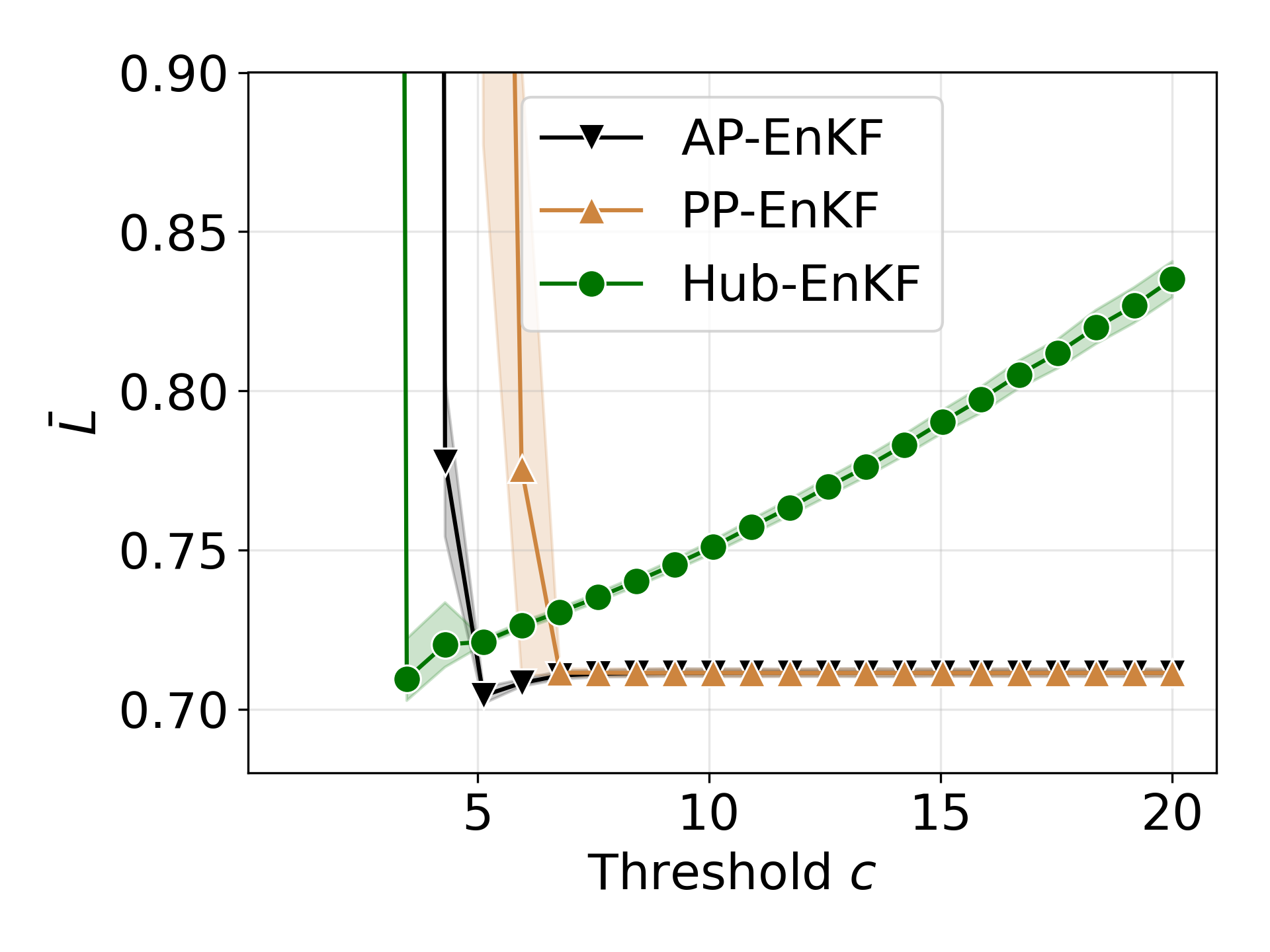}
    \removewhitespace[-6mm]
    \caption{
        The left panel shows a run of the \mEnkf, the \mWEnkfHard, the \mWEnkfSoft, and the \mHubEnkf;
        with covariance inflation;
        outlier events are shown in grey vertical bars.
        The right panel shows the bootstrap estimate of $L_T$ over 20 runs and 500 bootstrapped samples
        as a function of the $c$ hyperparameter.
    }
    \label{fig:lorenz96-methods-comparison-inflation-covariance}
\end{figure}
Similar to the results presented in the main body of text, we observe that, in the best scenario,
the \mHubEnkf and \mWEnkfHard behave similarly.
However, the \mHubEnkf is more sensitive to the choice of hyperparameters than the
\mWEnkfHard and \mWEnkfSoft.

\subsection{Robust EKF for online MLP regression (1d) --- suplementary experiment}
\label{experiment:training-neural-network}

In this section, we consider an online
nonlinear 1d regression, with the training
data coming
either from an i.i.d. source, or a correlated source.
The latter 
corresponds to a non-stationary problem
\citep[see e.g.][]{cartea2023toxic-flow, arroyo2024deep, Duran-Martin2022}.

We present a stream of observations
${\cal D}^\text{filter} = (y_1, x_1), \ldots (y_T, x_T)$ with
$y_t \in \real$ the measurements,
$x_t \in \real$ the exogenous variables, and
$T = 1500$.
The measurements and exogenous variables are sequentially sampled from the processes
\eat{
\begin{equation}\label{eq:noisy-sinusoidal}
    \vy_t = 
    \begin{cases}
        \vx_t / 5 - 10\cos(\vx_t\,\pi) + \vx_t ^ 3 + V_t & \text{w.p.}\, 1 - p_\epsilon,\\
        U_t & \text{w.p. }\, p_\epsilon,
    \end{cases}
\end{equation}
where $\vx_t\sim {\cal U}[-3, 3]$, $V_t\sim {\cal N}(0, 3)$, $U_t \sim {\cal U}[-40, 40]$, and $p_\epsilon = 0.05$.
}
\begin{equation}\label{eq:noisy-sinusoidal}
    y_t = 
    \begin{cases}
         \vtheta^*_{1} x_t  - \vtheta^*_{2}
         \cos(\vtheta^*_{3} x_t\,\pi) + \vtheta^*_{4} x_t^3 
        + V_t & \text{w.p.}\, 1 - p_\epsilon,\\
        U_t & \text{w.p. }\, p_\epsilon,
    \end{cases}
\end{equation}
where the parameters of the observation
model are $\vtheta^*=(0.2, -10, 1.0, 1.0)$,
the inputs are $x_t\sim {\cal U}[-3, 3]$, 
and the noise is
$V_t\sim {\cal N}(0, 3)$, 
$U_t \sim {\cal U}[-40, 40]$, and $p_\epsilon = 0.05$.

We consider four configurations of this experiment. In each experiment
the data is either sorted by $x_t$ value (i.e, 
the exogenous variable satisfies $x_i < x_j$ for all $i < j$,
representing a correlated source) or is
unsorted (representing
an i.i.d. source), and
the measurement function is either 
a clean version of the true data generating process
(i.e., 
\eqref{eq:noisy-sinusoidal} with $p_\epsilon = 0$ and unknown coefficients $\vtheta$), or 
a neural network with unknown parameters $\vtheta$.
Specifically, we use a  multi-layered perceptron (MLP) with two hidden layers and 10 units per layer:
\begin{equation}\label{eq:experiment-mlp}
    h(\vtheta_t, x_t)
    = \vw_t^{(3)}\phi\left(\vw_t^{(2)}\phi\left(\vw_t^{(1)}x_t + \vb_t^{(1)}\right) + \vb_t^{(2)}\right) + \vb_t^{(3)},
\end{equation}
with
activation function $\phi(u) = \max\{0, u\}$
applied elementwise.
Thus the state vector encodes the parameters:
\begin{equation*}
\begin{aligned}
    \vtheta_t = (
    \vw_t^{(1)} \in \real^{10\times 1}, \vw_t^{(2)} \in \real^{10\times 10}, \vw_t^{(3)} \in \real^{1 \times 10},
    \vb_t^{(1)}\in\real^{10}, \vb_t^{(2)} \in \real^{10},  \vb_t^{(3)} \in \real)
\end{aligned}
\end{equation*}
and has size so that $\vtheta\in\real^{141}$.
Note that in this experiment $h_t(\vtheta) = h(\vtheta, x_t)$.
We set $Q_t = 10^{-4}{\bf I}$,
which allows the parameters to slowly drift over time and provides some reguralisation.

For each method, we evaluate the
$\text{RMedSE} = \sqrt{\text{median}\{(y_t - h_t(\vmu_{t | t- 1})) ^ 2\}_{t=1}^T}$.
The \mAgamenoniExtended and the \mWangExtended methods
are taken with two inner iterations,
which implies that their computational complexity is twice that of the WoLF methods.

\paragraph{MLP measurement model}

Figure \ref{fig:online-mlp-example-run-sorted} shows results
when the data are presented in sorted order of $x_t$.
We show the  performance on 100 trials.
The left panel
shows the mean prior-predictive $h(\vmu_{t | t-1}, x_t)$ of 
each method, and the underlying true state process,
for a single trial.
The right panel shows the RMedSE after multiple trials.
\begin{figure}[htb]
    \centering
    \includegraphics[width=0.45\linewidth]{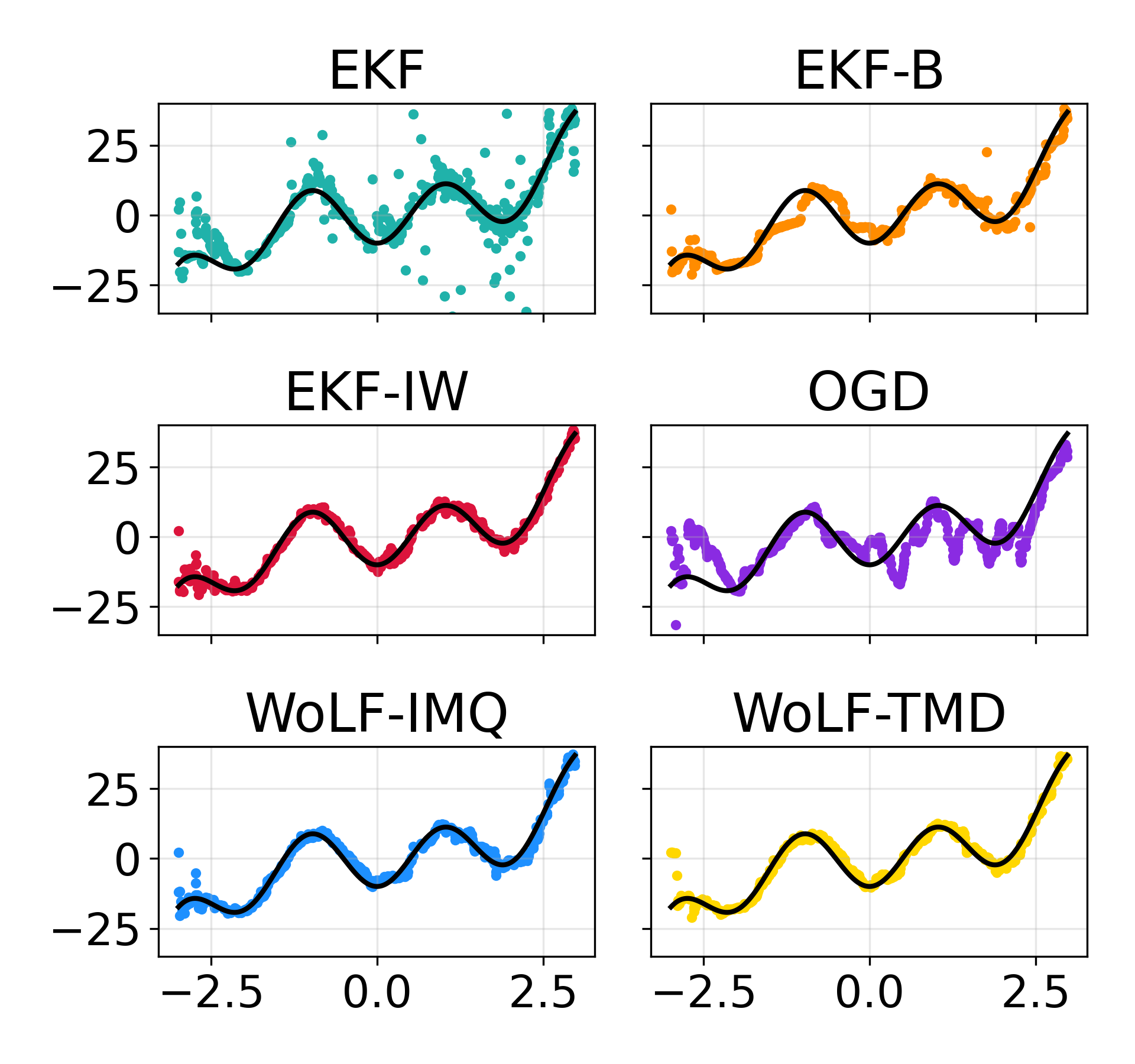}
    \includegraphics[width=0.45\linewidth]{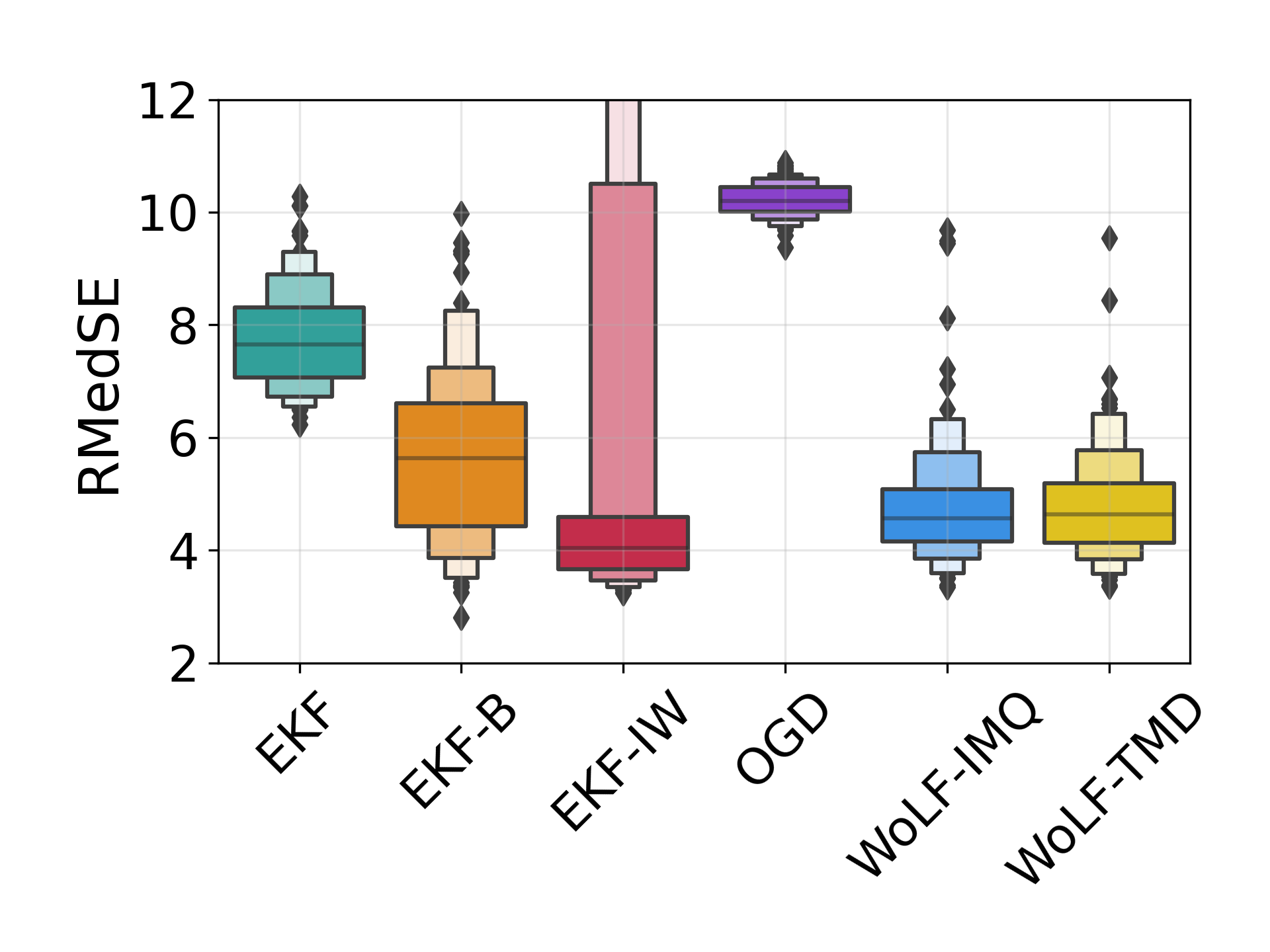}
    \removewhitespace[-6mm]
    \caption{
    Results with sorted data.
        Left panel shows a
        run of each filter on the 1d regression,
        with the true underlying
        data-generating function in solid black line
        and the next-step  predicted observation
        as dots.
        Right panel shows the
        RMedSE distribution over multiple trials.
    }
    \label{fig:online-mlp-example-run-sorted}
\end{figure}
We observe on the right panel that the \mWlfImq and the \mAgamenoniExtended have 
the lowest mean error and lowest standard deviation among the competing methods.
However, the \mAgamenoniExtended takes twice as long to run the experiment.
For all methods, the performance worse on the left-most side of the plot on the left panel,
which is a region with not enough data to determine whether a measurement is an inlier or an outlier.

Figure \ref{fig:online-mlp-example-run-unsorted} shows the results when data are presented in random order of $\vx_t$. 
\begin{figure}[htb]
    \centering
    \includegraphics[width=0.45\linewidth]{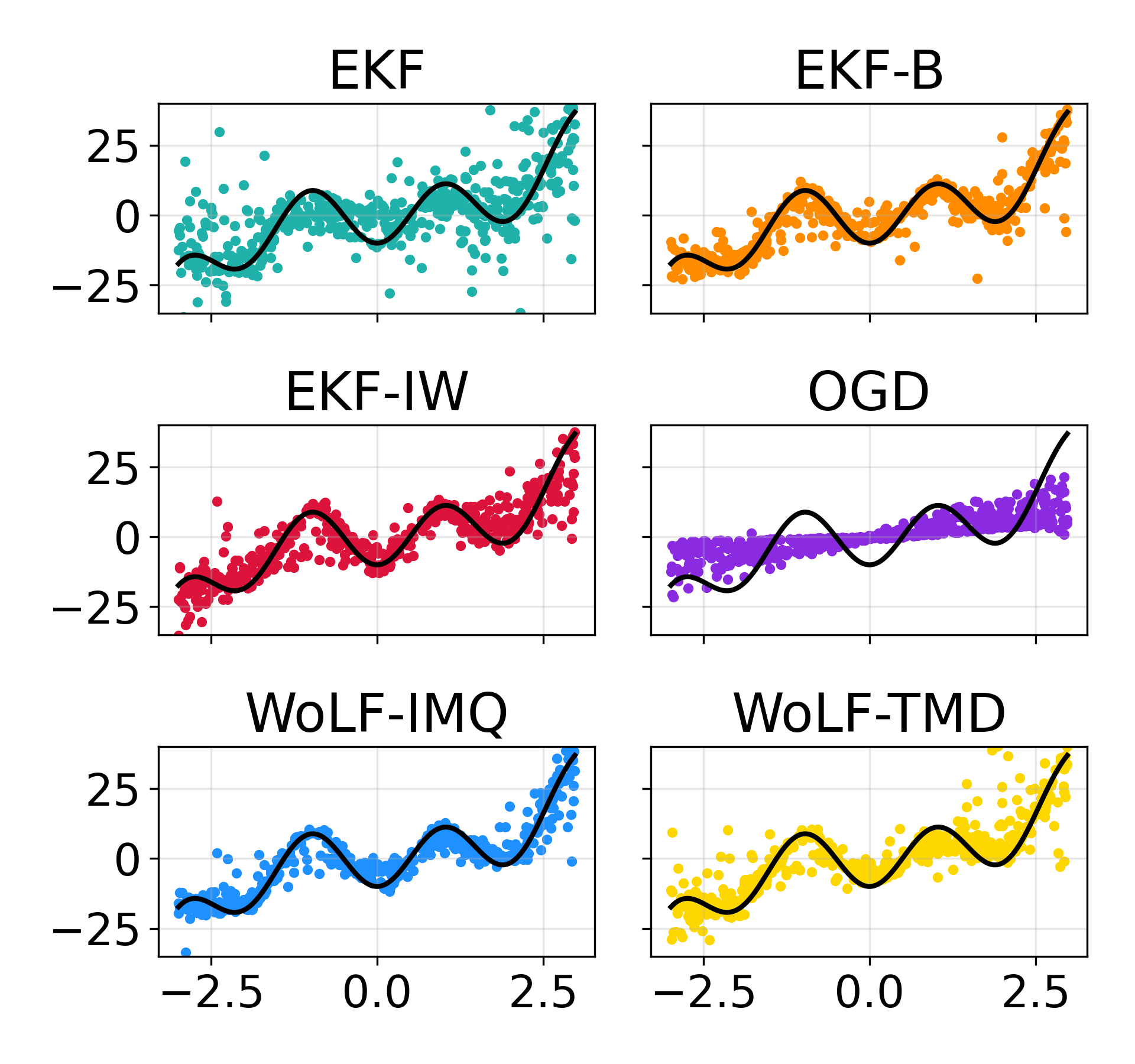}
    \includegraphics[width=0.45\linewidth]{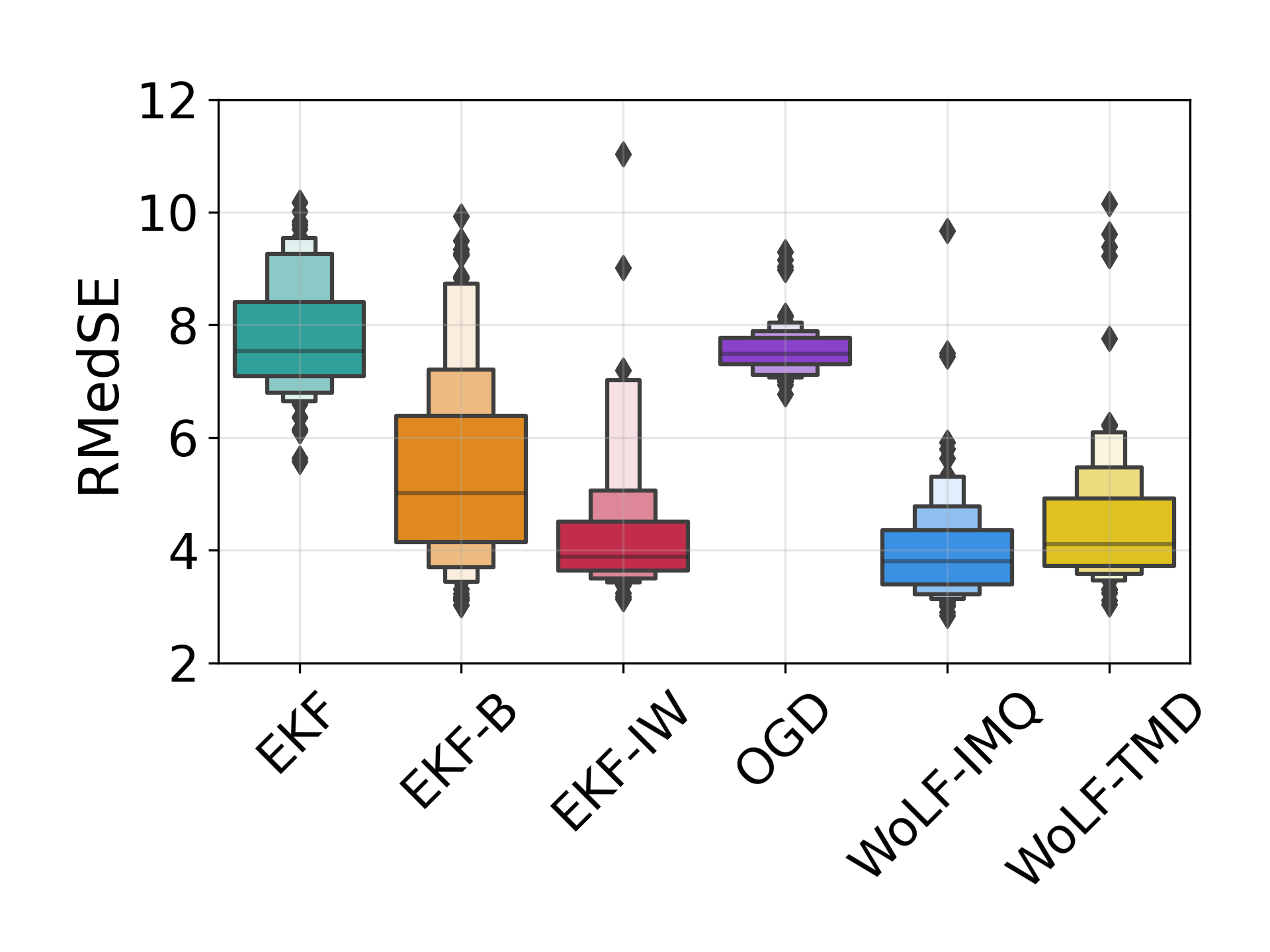}
    \removewhitespace[-6mm]
    \caption{
    Results with unsorted inputs.
        The left panel shows a run of each filter with the underlying
        data-generating function in solid black line
        and the next-step  predicted observation
        as dots.
        The right panel shows the distribution of $G_T$ for multiple runs.
        We remove all values of $G_T$ that have a value larger than 800.
    }
    \label{fig:online-mlp-example-run-unsorted}
\end{figure}
We show results for a single run on the left panel and the the RMedSE after multiple trials on the right panel.
Similar to the sorted configuration, we observe that the \mAgamenoniExtended and the \mWlfImq are the methods with lowest RMedSE.
However, the \mAgamenoniExtended has longer tails than the \mWlfImq.

\paragraph{True measurement model}

We modify the experiment above by taking the measurement function to be
$h_t(\vtheta_t) = h(\vtheta_t, x_t) = \vtheta_{t,1}x_t  - \vtheta_{t,2}\cos(\vtheta_{t,3} x_t\,\pi) + \vtheta_{t,4}x_t^3$,
with state $\vtheta_t \in \real^4$ and $\vtheta_{t,i}$ the $i$-th entry of the state vector $\vtheta_t$.
Figure \ref{fig:sorted-unsorted-clean-measurement} shows a single
run of the filtering process when the data is presented unsorted (left panel)
and sorted (right panel).
\begin{figure}[htb]
    \centering
    \includegraphics[width=0.45\linewidth]{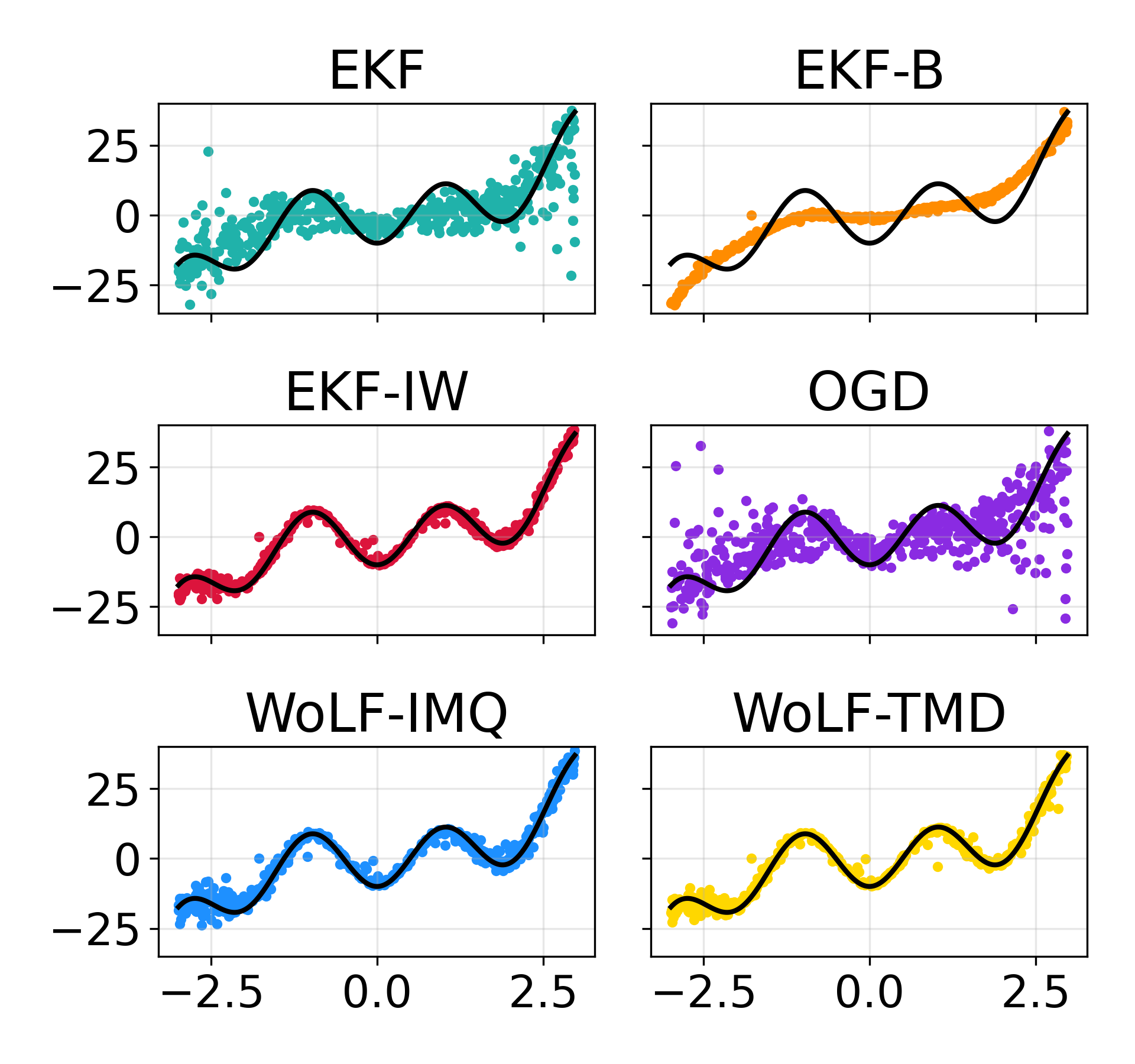}
    \includegraphics[width=0.45\linewidth]{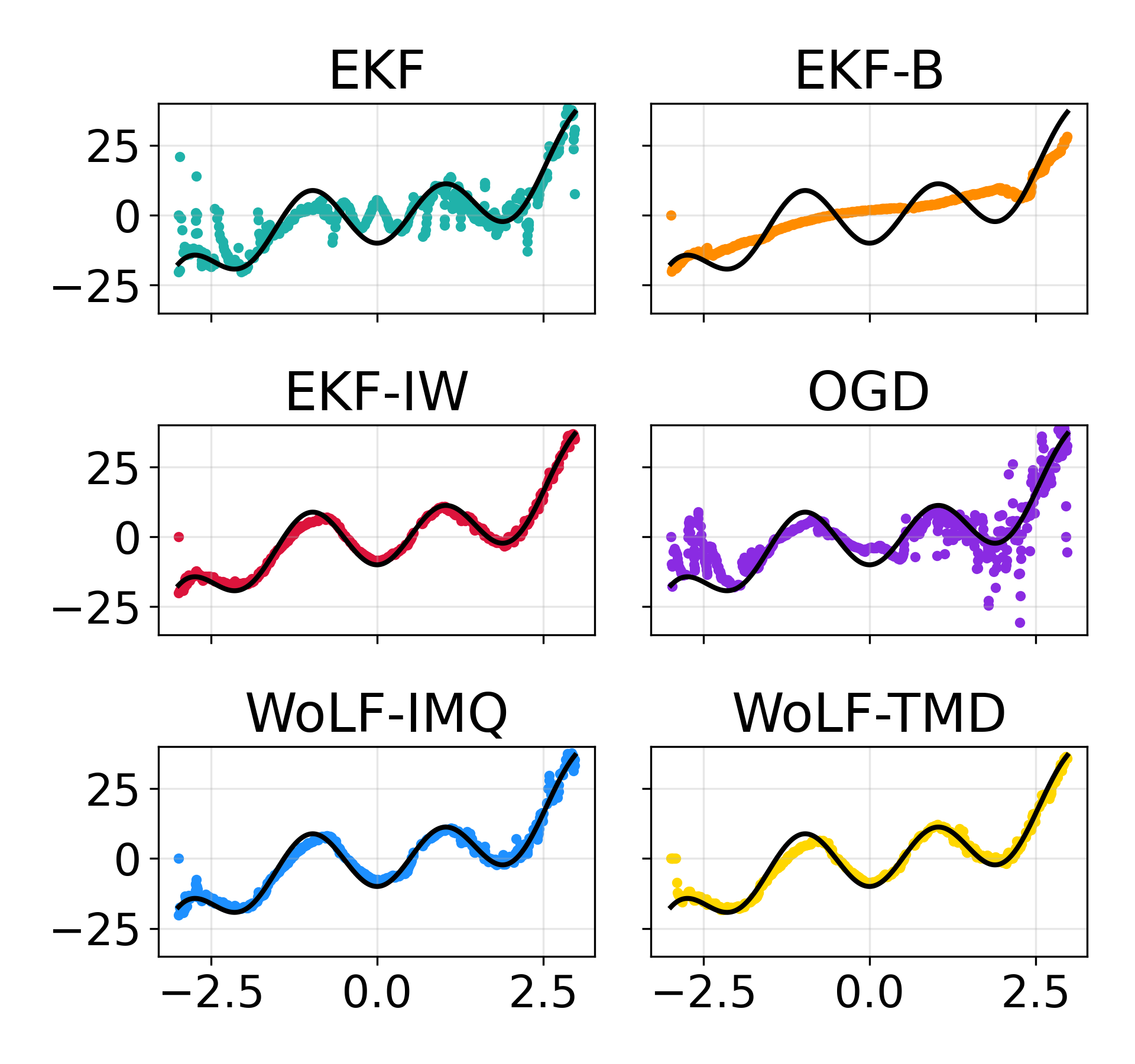}
    \removewhitespace[-6mm]
    \caption{
        The figure shows a run of each filter with the underlying
        data-generating function in solid black line
        and the evaluation of $h(\vmu_{t|t-1}, x_t)$ in points.
        The left panel shows the configuration with unsorted $x_t$ values and
        the right panel shows the configuration with sorted $x_t$ values.
    }
    \label{fig:sorted-unsorted-clean-measurement}
\end{figure}
We observe that the behaviour of the \mWlfImq, the \mWlfMd, and the \mAgamenoniExtended have similar performance.
However, the \mAgamenoniExtended takes twice the amount of time to run.
The \ogd and the \mkfExtended are not able to correctly filter out outlier measurement at the tails.
Finally, the \mWangExtended over-penalises inliers and does not capture the curvature of the measurement process.

\end{document}